\documentclass[final,onefignum,onetabnum]{siamart251216}

\usepackage{lipsum}
\usepackage{amsfonts}
\usepackage{graphicx}
\usepackage{epstopdf}
\usepackage{algorithmic}
\ifpdf
  \DeclareGraphicsExtensions{.eps,.pdf,.png,.jpg}
\else
  \DeclareGraphicsExtensions{.eps}
\fi

\newcommand{\changes}[1]{{#1}}

\def\R{\mathbb R}

\def\Var{\mathrm{Var}}
\def\Cov{\mathrm{Cov}}
\renewcommand\epsilon{\varepsilon}

\def\va{{\mathbf{a}}}
\def\vb{{\mathbf{b}}}

\def\ve{{\mathbf{e}}}

\def\vm{{\mathbf{m}}}

\def\vu{{\mathbf{u}}}
\def\vv{{\mathbf{v}}}

\def\vx{{\mathbf{x}}}

\def\vz{{\mathbf{z}}}

\def\mA{{\mathbf{A}}}
\def\mB{{\mathbf{B}}}
\def\mC{{\mathbf{C}}}
\def\mD{{\mathbf{D}}}
\def\mE{{\mathbf{E}}}
\def\mF{{\mathbf{F}}}

\def\mH{{\mathbf{H}}}
\def\mI{{\mathbf{I}}}

\def\mK{{\mathbf{K}}}

\def\mM{{\mathbf{M}}}
\def\mN{{\mathbf{N}}}
\def\mO{{\mathbf{O}}}
\def\m0{{\mathbf{0}}}
\def\mP{{\mathbf{P}}}
\def\mQ{{\mathbf{Q}}}
\def\mR{{\mathbf{R}}}

\def\mSXY{\mathbf{S}_{XY}}
\def\mSXX{\mathbf{S}_{XX}}
\def\mSYY{\mathbf{S}_{YY}}
\def\mT{{\mathbf{T}}}
\def\mU{{\mathbf{U}}}
\def\mV{{\mathbf{V}}}
\def\mW{{\mathbf{W}}}
\def\mX{{\mathbf{X}}}
\def\mY{{\mathbf{Y}}}
\def\mZ{{\mathbf{Z}}}
\def\mSigma{{\mathbf{\Sigma}}}
\def\mGamma{{\mathbf{\Gamma}}}
\def\mDelta{{\mathbf{\Delta}}}
\def\mLambda{{\mathbf{\Lambda}}}

\def\mPsi {{\mathbf{\Psi }}}

\def\R{\mathbb{R}}
\def\C{\mathbb{C}}

\def\trans{{\top}}

\DeclareMathOperator*{\argmax}{\arg\max}

\DeclareMathOperator{\Sp}{Sp}

\DeclareMathOperator{\Supp}{Supp}
\DeclareMathOperator{\Tr}{Tr}
\DeclareMathOperator{\trace}{trace}

\newcommand{\simiid}{\overset{\text{i.i.d.}}{\sim}}
\newcommand{\abs}[1]{\lvert #1 \rvert}
\newcommand{\Abs}[1]{\left\lvert #1 \right\rvert}
\newcommand{\deriv}[2]{\frac{\partial #1}{\partial #2}}
\newcommand{\esp}[1]{\mathbb{E}[#1]}
\newcommand{\Esp}[1]{\mathbb{E} \left[ #1 \right]}
\newcommand{\norm}[1]{\lVert #1 \rVert}
\newcommand{\Norm}[1]{\left\lVert #1 \right\rVert}
\newcommand{\ie}{\emph{i.e.,}~}
\newcommand{\asto}{\overset{\rm a.s.}{\longrightarrow}}

\newsiamremark{remark}{Remark}
\newsiamremark{hypothesis}{Hypothesis}
\crefname{hypothesis}{Hypothesis}{Hypotheses}
\newsiamthm{claim}{Claim}
\newsiamremark{fact}{Fact}
\crefname{fact}{Fact}{Facts}

\headers{High-Dimensional Partial Least Squares}{V. Léger and F. Chatelain}

\makeatletter
\def\thanks#1{\protected@xdef\@thanks{\@thanks
        \protect\footnotetext{#1}}}
\makeatother
\title{High-Dimensional Partial Least Squares: Spectral Analysis and Fundamental Limitations\thanks{}
\funding{The first author is funded from the MIC 2022 call (\emph{Interdisciplinary Approaches to Oncogenic Processes and Therapeutic Perspectives: Contributions of Mathematics and Computer Science to Oncology}), with financial support from ITMO Cancer of Aviesan within the framework of the 2021–2030 Cancer Control Strategy, on funds administered by Inserm.}}

\author{Victor Léger\thanks{Université Grenoble Alpes, CNRS, Grenoble INP, GIPSA-lab\\
 Grenoble, 38000, France
  (\email{victor.leger@grenoble-inp.fr},\email{florent.chatelain@grenoble-inp.fr}).} \and Florent Chatelain\footnotemark[2]}

\usepackage{amsopn}

\ifpdf
\hypersetup{
  pdftitle={High-Dimensional Partial Least Squares},
  pdfauthor={V. Léger and F. Chatelain}
}
\fi

\usepackage{blindtext}

\usepackage{xcolor}

\usepackage{epsfig, amssymb}

\newsiamthm{condition}{Conditions}

\newsiamthm{assumption}{Assumptions}

\usepackage{lastpage}

\begin{document}

\maketitle
\begin{abstract}
Partial Least Squares (PLS) is a widely used method for data integration, designed to extract latent components shared across paired high-dimensional datasets. Despite decades of practical success, a precise theoretical understanding of its behavior in high-dimensional regimes remains limited. In this paper, we study a data integration model in which two high-dimensional data matrices share a low-rank joint latent structure and also contain individual-specific components. We analyze the singular vectors of the associated cross-covariance matrix using tools from random matrix theory and derive asymptotic characterizations of the alignment between estimated and true latent directions. These results provide a quantitative explanation of the reconstruction performance of the PLS variant based on the Singular Value Decomposition (PLS-SVD) and identify regimes where the method exhibits counterintuitive or limiting behavior. Building on this analysis, we compare PLS-SVD with principal component analysis applied separately to each dataset and show its asymptotic superiority in detecting the joint latent subspace. Overall, our results offer a comprehensive theoretical understanding of high-dimensional PLS-SVD, clarifying both its advantages and fundamental limitations.
\end{abstract}



\section{Introduction}

Partial Least Squares (PLS) is a widely-used dimension reduction method originally developed by Hermann Wold in the 1960s \cite{Econometrics} to address econometric problems. Since then, it has been applied across many scientific domains, including chemometrics and spectrometric modeling \cite{Chemometrics}, as well as  bioinformatics and genomics \cite{Genomics}. PLS is particularly effective in high-dimensional settings with limited sample sizes, a typical scenario in genomics, where datasets may contain around 20{,}000 gene expression measurements for only a few hundred individuals. This strength has led to its broad adoption in genomic research for diverse tasks, such as analyzing gene correlations in yeast cells \cite{genes_levure}, imputing missing data \cite{imputation}, classifying cancer subtypes \cite{classification}, and predicting therapeutic outcomes \cite{therapy}.
Although numerous PLS algorithm variants exist \cite{PLS_variants}, they all rely on the same underlying principle: constructing latent components that maximize the covariance with the response variables. Let  $\mX \in \R^{n\times p}$ denote the matrix of the predictors and $\mY \in \R^{n\times q}$ the matrix of the responses, where $p$ and $q$ are the respective numbers of variables and $n$ is the common sample size. The PLS approach seeks to identify the projection directions that maximize the covariance between $\mX$ and $\mY$. 

More precisely, letting $r$ denote the dimension of the latent space shared by $\mX$ and $\mY$, the goal is to extract $r$ pairs of directions $(\vu_k,\vv_k)$ that maximize the sample covariance between the projected variables $\mX\vu_k$ and $\mY\vv_k$. 
For the first component, all PLS variants reduce to the same optimization problem:
\begin{align*}
\argmax_{\vu_1, \vv_1}~& \left(\mX\vu_1\right)^\trans \mY \vv_1 \equiv \argmax_{\vu_1, \vv_1}~\vu_1^\trans  \mSXY \vv_1 \\
\textrm{s.t.}~&\lVert \vu_1 \rVert=1 \text{ and } \lVert \vv_1 \rVert=1,
\end{align*}
where $\mSXY=\frac{1}{\sqrt{pq}}\mX^\trans\mY$ is the normalized sample cross-covariance matrix.
This problem admits a unique solution for $(\vu_1,\vv_1)$, which corresponds to the leading left and right singular vectors of $\mSXY$, as guaranteed by the Eckart--Young--Mirsky theorem.

After extracting the first component, the subsequent steps depend on the specific PLS variant. Some algorithms perform separate deflation of $\mX$ and $\mY$, removing the part explained by the previously obtained components; others deflate the cross-product $\mX^\trans\mY$ directly. In addition, each variant imposes orthogonality constraints to ensure identifiability of the successive directions. These constraints differ across PLS formulations and lead to distinct properties (typically geared toward prediction or modeling).
At iteration $k>1$, the next pair $(\vu_k,\vv_k)$ is obtained by solving
\begin{align*}
\argmax_{\vu_k, \vv_k}&~\vu_k^\trans \mSXY^{(k)}\vv_k\\
\textrm{s.t.}~\lVert \vu_k \rVert&=1 \text{ and } \lVert \vv_k \rVert=1, \ \text{and orthogonality constraints},
\end{align*}
where $\mSXY^{(k)}$ denotes the updated version of $\mSXY$ at step $k$.

Because numerous algorithmic variants exist, our goal is not to survey all implementations but rather to highlight the behavior and the fundamental challenges faced by PLS in high-dimensional settings. For this purpose, we focus on the PLS-SVD variant \cite{PLS_variants}, which requires  each new direction to satisfy
$$
\vu_k \perp \operatorname{span}\{\vu_1, \ldots, \vu_{k-1}\}, \qquad
\vv_k \perp \operatorname{span}\{\vv_1, \ldots, \vv_{k-1}\}.
$$
Under these constraints, the Eckart--Young--Mirsky theorem ensures that all components can be obtained simultaneously through the singular value decomposition (SVD) of $\mSXY$, thereby avoiding any iterative deflation step. Our analysis therefore focuses on the spectral properties of the square symmetric matrices
\begin{align}
\mK \equiv& \frac{1}{pq} \mY^\trans\mX\mX^\trans\mY = \mSXY^\trans\mSXY \in \R^{q\times q}, \label{eq:K} \\
\tilde{\mK} \equiv& \frac{1}{pq} \mX^\trans\mY\mY^\trans\mX = \mSXY\mSXY^\trans \in \R^{p\times p}, \label{eq:K_tilde}
\end{align}
whose nonzero eigenvalues coincide with the nonzero squared singular values of $\mSXY$, and whose eigenvectors correspond respectively to the right and left singular vectors of $\mSXY$.
In this work, we further consider the following asymptotic regime.
\begin{assumption}[High-dimensional asymptotic regime]
We assume that \\
$p,q,n \to \infty$ with the ratios of the sample size to the dimensions converging to finite positive constants:
\begin{align}
\frac{n}{p} \to \beta_p>0, \qquad
\frac{n}{q} \to \beta_q>0, \qquad
\beta \equiv \max(\beta_p,\beta_q), \label{eq:betas}
\end{align}
where $\beta$ can be interpreted as the limit of the ratio $\frac{n}{d}$ with $d=\min(p,q)$.
\label{ass:asympt}
\end{assumption}
This high-dimensional framework is particularly well suited for random matrix theory, which provides sharp characterizations of the limiting spectral distribution of the  large matrices involved, the behavior of their largest eigenvalues, and the alignment of the corresponding eigenvectors.

\subsection{Related Works}

The statistical properties of PLS have been studied from various perspectives over the past decades. Early theoretical work focused on establishing the algorithmic foundations and interpretations of PLS \cite{helland1988onthe, garthwaite1994interp}. Later theoretical developments analyzed the structure of PLS estimators and derived conditions for their consistency and asymptotic behavior in classical settings where the number of samples grows while dimensions remain fixed \cite{naik2000pls,helland2000theo}.
However, these results do not directly apply to modern high-dimensional scenarios,  and subsequent work has highlighted the limitations of PLS consistency when the number of variables is large relative to the sample size \cite{chun2010sparse}.

\paragraph{{Random Matrix Theory and Spiked Models}}
The application of random matrix theory (RMT) to understand high-dimensional statistical methods has gained considerable attention. 
Foundational results include the Mar\v{c}enko--Pastur law \cite{marchenko1967distribution}, which characterizes the limiting spectral distribution of an isotropic  sample covariance matrix, and the Baik--Ben~Arous--Péché (BBP) phase transition \cite{baik2005phase} for spiked covariance models, which specifies the threshold above which signals yield isolated eigenvalues. 
These tools have been successfully applied to principal component analysis \cite{paul2007asymptotics,loubaton2011almost, benaych2012singular}, as well as to 
efficient algorithms for computing low-rank matrix approximations \cite{couillet2021two},
providing precise characterizations of eigenvalue behavior and eigenvector consistency in high-dimensional regimes. 
Multimodal data have also been studied through the asymptotic analysis of low-rank tensor decompositions; see, for example, \cite{lebeau2025random}.
Our work extends these techniques to the PLS setting, where the presence of two coupled data matrices creates additional complexity.

\paragraph{{PLS in High Dimensions}} Despite its widespread use across numerous scientific domains, PLS remains comparatively less understood from a theoretical standpoint in high-dimensional settings. Recent efforts have begun addressing this gap. Sparse PLS methods were proposed in \cite{chun2010sparse} to regularize the solutions, while \cite{cook2019pls} analyzed the asymptotic prediction behavior of PLS regression across various high-dimensional regimes, establishing convergence rates. These contributions, however, primarily concern regression problems and stop short of providing a full characterization of the spectral properties of PLS matrices. 

The necessary foundations for understanding the bulk distribution of the singular values in noise-only (i.e., spike-free) settings are provided in a very recent work on cross-covariance matrices in these regimes \cite{swain2025distribution}. 
Concurrent and complementary to our work, \cite{mergny2025spectral} independently analyzed correlated spiked cross-covariance models motivated by multimodal learning, providing a characterization of BBP-type phase transitions for PLS using free probability techniques. 
The two approaches employ different RMT methodologies: the authors of \cite{mergny2025spectral} leverage free probability theory and subordination relations \cite{belinschi2007new}, while our analysis relies on deterministic equivalents and resolvent methods \cite{couillet_liao_2022}. These complementary perspectives provide different insights into the high-dimensional behavior of PLS and cross-covariance matrices as detailed in the following summary of contributions. 
%
\paragraph{Deflation Schemes and Algorithm Variants}%
Our analysis focuses on the PLS-SVD formulation and hereafter we refer to this as PLS. While  several alternative deflation schemes have been proposed in the literature \cite{PLS_variants}, they differ primarily in their deflation strategy and orthogonality constraints.
Crucially, the first iteration of all these variants  is identical to PLS-SVD, which serves as the essential core operator for two-block PLS methods.
By establishing the spectral properties and phase transitions of PLS-SVD, we provide a theoretical framework for these iterative deflation schemes. Because these methods rely on successive deflations, they are inherently constrained by the interlacing property of the spectrum. Consequently, phenomena we rigorously analyzed in PLS-SVD, such as spurious spikes, spike duplication and skewed limit vectors, are structurally inherited by all subsequent components in the iterative scheme. Finally, our Random Matrix Theory framework assumes that the dimensions $p$ and $q$ grow proportionally to $n$. Therefore, asymmetric PLS1 models where $q=1$ fall outside the scope of this specific asymptotic limit.

\subsection{Summary of Contributions}
\label{sec:contributions}

This paper provides a comprehensive theoretical analysis of Partial Least Squares (PLS) in high-dimensional settings through the lens of random matrix theory.
In the following summary, we preview the notations of our signal-plus-noise model: $\mP$ and $\mR$ represent the joint latent structures, $\mSigma$ quantifies their alignment, while $\mM$ and $\mN$ denote the individual components. These matrices, which will be formally defined in Section~\ref{sec:model}, 
are introduced here to provide a direct mapping between our theoretical claims and the model's architecture. 
We acknowledge that \cite{mergny2025spectral} concurrently provides spectral thresholds and spike positions, though their study is restricted to partially correlated spiked models. Our paper significantly extends this and other previous works through the following contributions:
\begin{itemize}
\item \textbf{Deterministic equivalents and limiting spectral distribution.} We establish deterministic equivalents for the resolvent matrices associated with the PLS cross-covariance kernels (Theorem~\ref{thm:eq_det}), which serve as the keystone for all subsequent results. Building on these equivalents, we characterize the limiting distribution of the squared singular values of the normalized cross-covariance matrix $\mSXY = \frac{1}{\sqrt{pq}}\mX^T\mY$ (Proposition~\ref{prop:lsd}). 
Together, these foundational results provide the technical framework for analyzing signal detection and recovery in PLS.
\item \textbf{Phase transitions and spike detection.} We identify the critical phase transition threshold $\tau$, defined by Equation~\eqref{eq:thresh}, above which signal components yield isolated singular values that emerge from the bulk distribution and become detectable. We provide closed-form expressions for the asymptotic locations of these spikes as functions of the underlying signal singular values across multiple settings: for joint shared structures (encoded by $\mP$, $\mR$, and $\mSigma$) with either partial (Proposition~\ref{prop:isolated_PR}) or perfect alignment ($\mSigma=\mI_r$, Corollary~\ref{cor:values_perfect_align}), as well as for individual-specific components ($\mM$ and $\mN$, Propositions~\ref{prop:isolated_ST}).
\item \textbf{Eigenvector alignment and fundamental limitations.} We precisely quantify the alignment between PLS singular vectors and the  deterministic signal directions (Propositions~\ref{prop:alignment_ST} and~\ref{prop:alignment_PR}). Our analysis reveals three fundamental limitations of PLS in high-dimensional regimes:
\begin{itemize}
    \item \textbf{Spurious individual components:} When individual-specific structures ($\mM$ and $\mN$) are present, PLS can spuriously align with these uninformative directions rather than with the joint signal. Proposition~\ref{prop:alignment_ST} shows that the corresponding singular vectors do not align with any deterministic signal direction beyond their generating component, effectively representing noise-driven artifacts.
    \item \textbf{Spurious duplication of joint components:} Even when a joint component is present in the true signal, partial alignment can give rise to two distinct pairs of dominant PLS singular vectors (Proposition~\ref{prop:isolated_PR}), thereby artificially inflating the apparent rank of the joint latent subspace. While this duplication phenomenon was also observed in \cite{mergny2025spectral}, we show that it disappears entirely for pairs of perfectly aligned components (Proposition~\ref{prop:isolated_PR} and Corollary~\ref{cor:values_perfect_align}).
    \item \textbf{Systematic skewing of joint components:} PLS singular vectors fail to perfectly recover the singular vectors of the shared true signal (encoded in $\mP\mSigma\mR^T$), even under perfect alignment ($\mSigma=\mI_r$). Instead, they align with skewed versions of these directions (Proposition~\ref{prop:alignment_PR}). This noise-induced distortion persists asymptotically and vanishes only in the limit of infinite signal strength or under specific special cases (Remarks~\ref{rem:signal_dominant} and \ref{rem:diagonal}).
\end{itemize}
\item \textbf{Comparison with PCA.} We establish that, under the same asymptotic regime, PLS exhibits strictly greater statistical power than separate PCA applied independently to each data matrix for detecting joint latent directions (Proposition~\ref{prop:PLSvsPCA}). Specifically, we prove that whenever separate PCA detects all $r$ joint spikes, PLS necessarily detects them as well, with strictly larger spectral separation. This theoretical advantage confirms PLS as a superior method for integrative analysis.
\item \textbf{Model framework.} Our analysis is conducted within a general signal-plus-noise model that decomposes the data matrices $\mX$ and $\mY$ into joint, individual, and noise components \eqref{eq:model}. This structure directly connects to integrative data analysis methods, such as JIVE (Joint and Individual Variation Explained) \cite{JIVE-genomics}, providing a natural setting for understanding multimodal learning. Notably, our framework encompasses the model analyzed in \cite{mergny2025spectral}, which is restricted solely to partially aligned joint components. Furthermore, while we focus on the PLS-SVD formulation for  simplicity and clarity, its initial step is identical across all two-block PLS variants, providing a theoretical framework that naturally extends to other iterative deflation schemes.
\end{itemize}

Taken together, these results offer an explicit and comprehensive theoretical characterization of PLS in high-dimensional regimes within a generic integrative signal-plus-noise framework, accounting for joint, individual, and noise components. This reveals both the advantages of PLS for signal detection and its fundamental limitations for signal recovery. Our findings point to promising directions for developing enhanced methods that filter out spurious components while preserving the joint latent structure.

\paragraph{Outline of the paper}
Section~\ref{sec:model} introduces the notations, the proposed data integration model, and the underlying assumptions in the high-dimensional regime.  
In Section~\ref{sec:high_dim_analysis}, we conduct a random matrix analysis of the singular vectors of the cross-covariance matrix associated with the proposed model. These results are then exploited to provide a quantitative analysis of the reconstruction performance of the signal components obtained with the PLS-SVD approach, and to shed light on some of its counterintuitive and limiting behaviors.  
Building on these findings, Section~\ref{sec:pca_comparison} is devoted to a comparison with principal component analysis applied separately to each data matrix. We demonstrate the asymptotic superiority of PLS-SVD for detecting the directions of the latent joint subspace.  
Finally, Section~\ref{sec:conclusion} concludes the paper and discusses the results. Most proofs are deferred to the appendices\footnote{The Python code used to reproduce all the figures in the paper is available at \url{https://gricad-gitlab.univ-grenoble-alpes.fr/paper-codes/pls-svd/figures}}.

\section{Model for PLS Analysis}
\label{sec:model}

\subsection{General Notations}
$a$, $\va$, $\mA$  denote a scalar, a vector, and a matrix, respectively. The imaginary part of $z \in \C$ is $\Im[z]$. The set $\{1, \ldots, n\}$ of positive integers less than or equal to $n$ is denoted $[n]$. The support of a probability measure $\mu$ is denoted $\Supp(\mu)$. 
%
Given a sequence of random variables $(X_n)_{n \geqslant 0}$, its almost sure convergence to $L$ is denoted $X_n \asto L$.
The normal distribution with mean $m$ and variance $\sigma^2$ is denoted $\mathcal{N}(m, \sigma^2)$. Similarly, the multivariate normal distribution with mean vector $\vm$ and covariance $\mC$ is denoted $\mathcal{N}(\vm, \mC)$. 
Depending on the context, $\lambda_1(\mA) \geq \lambda_2(\mA) \geq \dots \geq 0$ denote either the \emph{squared} singular values of a rectangular matrix $\mA$, or the eigenvalues of a square matrix $\mA$,  listed in \emph{non-increasing} order. The corresponding left singular vectors (or eigenvectors) are denoted $\vu_k(\mA)$, while the right singular vectors of a rectangular matrix are denoted $\vv_k(\mA)$.
Given an $n \times n$ matrix $\mB$, its trace is $\Tr \mB = \sum_{i = 1}^n \mB_{i, i}$ and its spectrum, $\Sp \mB$, is the set of all its eigenvalues. $\norm{\cdot}$ denotes the standard Euclidean norm for vectors and the corresponding operator norm (spectral norm) for matrices. The Frobenius norm is denoted as $\norm{\cdot}_F$.
Finally, we use standard asymptotic notation as the dimensions $p, q,$ and $n$ tend to infinity: $f(n)=\mathcal{O}\left(g(n)\right)$ and $f(n)=o\left(g(n)\right)$ mean that $|f|/|g|$ is asymptotically bounded and asymptotically negligible, respectively.

\subsection{Signal-Plus-Noise Model}

For the purpose of our analysis, we consider a model that explicitly accounts for joint and individual latent components. In contrast, classical PLS formulations typically rely on simpler representations, either by implicitly absorbing the structured individual terms into a generic residual that also includes noise \cite{naik2000pls, helland2000theo, chun2010sparse}, or through latent factor models of the form \cite{wold2001pls}:
\begin{align}
\begin{split}
\mX &= \mU\bar{\mP}^\trans+\mE \ \in \R^{n\times p} \\
\mY &= \mV\bar{\mR}^\trans+\mF \ \in \R^{n\times q},
\label{eq:model_gen}
\end{split}
\end{align}
where $\mU$ and $\mV$ are respectively $n\times r_U$ and $n\times r_V$ low-rank matrices of the latent vectors, $\bar{\mP}\in\R^{p\times r_U}$ and $\bar{\mR}\in\R^{q\times r_V}$ are loading matrices, and $\mE$ and $\mF$ are noise matrices.


We now show that this formulation admits a decomposition into joint and individual latent components.
Starting from the latent model \eqref{eq:model_gen} and 
using reduced QR decompositions, we write
\begin{equation*}
\mU=\mQ_U\mR_U,
\qquad
\mV=\mQ_V\mR_V,
\end{equation*}
with
$\mQ_U\in\R^{n\times r_U}$, $\mQ_V\in\R^{n\times r_V}$, $\mQ_U^\trans\mQ_U=\mI_{r_U}$, and $\mQ_V^\trans\mQ_V=\mI_{r_V}$.
Consider now the product
$\mQ_U^\trans \mQ_V \in \R^{r_U\times r_V}$
and its singular value decomposition
\begin{equation*}
\mQ_U^\trans \mQ_V
=
\mO_U \mLambda \mO_V^\trans,
\end{equation*}
where
\begin{equation*}
\mLambda=
\begin{bmatrix}
\mSigma & \m0_{r\times (r_V - r)}\\
\m0_{(r_U - r)\times r} & \m0_{(r_U - r)\times (r_V - r)}
\end{bmatrix},
\end{equation*}
with \(\mSigma=\operatorname{diag}(\rho_1,\ldots,\rho_r)\) the diagonal matrix of non-zero singular values,  and \(r\le \min(r_U,r_V)\) their number. 
By the Cauchy--Schwarz inequality, the coefficients $\rho_k$ satisfy
$\rho_k \in (0,1]$, for $k \in [r]$.
Equivalently, they can be interpreted as the cosines of the principal angles $\theta_k\in[0,\pi/2)$ between the subspaces spanned by the columns of $\mQ_U$ and $\mQ_V$, namely $\rho_k= \cos{\theta_k}$.
We now define the joint components
\begin{equation*}
\mT_X=\mQ_U \mO_U^{(r)},
\qquad
\mT_Y=\mQ_V \mO_V^{(r)},
\end{equation*}
where \(\mO_U^{(r)}\in\R^{r_U\times r}\) and \(\mO_V^{(r)}\in\R^{r_V\times r}\) denote the matrices formed by the first \(r\) columns of \(\mO_U\) and \(\mO_V\), respectively. Then
\begin{equation*}
\mT_X^\trans \mT_X=\mI_r,
\qquad
\mT_Y^\trans \mT_Y=\mI_r,
\qquad
\mT_X^\trans \mT_Y=\mT_Y^\trans \mT_X =\mSigma.
\end{equation*}
Let \(\mO_U^{\perp}\in\R^{r_U\times (r_U-r)} \) and \(\mO_V^{\perp} \in\R^{r_V\times (r_V-r)}\) denote the remaining singular vector matrices, and define
\begin{equation*}
\mQ_U^\perp=\mQ_U\mO_U^{\perp} (\mO_U^\perp)^\top \in\R^{n\times r_U},
\qquad
\mQ_V^\perp=\mQ_V\mO_V^{\perp}  (\mO_V^\perp)^\top \in\R^{n\times r_V},
\end{equation*}
hence
$(\mQ_U^\perp)^\top \mQ_V^\perp= \mO_U^{\perp} \m0_{(r_U - r)\times (r_V - r)} (\mO_V^\perp)^\top = \m0_{r_U\times r_V}.$
These span directions that are orthogonal to the joint subspace. The corresponding contributions define
\begin{equation*}
\mM=\mQ_U^\perp \mR_U \bar{\mP}^{\top},
\qquad
\mN=\mQ_V^\perp \mR_V \bar{\mR}^{\top}.
\end{equation*}
Substituting these decompositions yields
\begin{equation*}
\mU\bar{\mP}^\trans=\mT_X\mP^\trans+\mM,
\qquad
\mV\bar{\mR}^\trans=\mT_Y\mR^\trans+\mN,
\end{equation*}
with $\mP=\bar{\mP} \mR_U^\top \mO_U^{(r)}$ and $\mR=\bar{\mR} \mR_V^\top \mO_V^{(r)}$.
%
We therefore consider, without loss of generality, the following  latent model, decomposed into joint and individual components:
\begin{align}
\begin{split}
\mX &= \mT_X\mP^\trans+\mM+\mE \ \in \R^{n\times p}, \\ 
\mY &= \mT_Y\mR^\trans+\mN+\mF \ \in \R^{n\times q},
\label{eq:model}
\end{split}
\end{align}
where, for each data matrix ($\mX$ or $\mY$), the first term captures the latent structure shared by both datasets, while the second term represents the individual structure specific to that dataset. More precisely:
\begin{itemize}
    \item $\mT_X \in \R^{n \times r}$ and $\mT_Y \in \R^{n \times r}$ are the score matrices associated with the joint components. 
    
    \item $\mP \in \R^{p\times r}$ and $\mR \in \R^{q\times r}$  are the corresponding loading matrices.
    \item $\mM \in \R^{n\times p}$ and $\mN \in \R^{n\times q}$ are low-rank matrices capturing the individual structure of $\mX$ and $\mY$, respectively. 
    \item $\mE \in \R^{n\times p}$ and $\mF \in \R^{n\times q}$ are noise matrices with respective entries $E_{ij}\simiid \mathcal{N}(0, 1)$ and $F_{ij}\simiid \mathcal{N}(0, 1)$. 
\end{itemize}
The derivation of model \eqref{eq:model} naturally introduced the {\em coupling} (or {\em alignment}) matrix  between the joint latent components, 
\begin{equation}
\mSigma = \operatorname{diag}(\rho_1,\ldots,\rho_r) \in \R^{r \times r}, \quad   \rho_k\in(0,1], \quad \textrm{where } \mT_X^\top \mT_Y = \mSigma.
\label{eq:sigma}
\end{equation}
When \(\rho_k=1\), the \(k\)-th latent components are perfectly aligned, whereas \(0<\rho_k<1\) corresponds to partial alignment. Moreover, the construction of \eqref{eq:model} from \eqref{eq:model_gen} naturally yields, without loss of generality, the following orthogonality conditions.

\begin{condition}[Orthogonality conditions]
\[
\begin{aligned}
    \mT_X^\trans \mT_X  & = \mI_r, &
    \mT_X^\trans \mT_Y  & = \mSigma, &
    \mM^\trans \mT_X & = \m0_{p\times r}, &
    \mN^\trans \mT_X & = \m0_{q\times r},&
     \\
    \mT_Y^\trans \mT_Y  & = \mI_r, &
    \mM^\trans \mN &= \m0_{p\times q},&
    \mM^\trans \mT_Y & = \m0_{p\times r}, &
    \mN^\trans \mT_Y & = \m0_{q\times r}.
\end{aligned}
\]
\label{cond:ortho}
\end{condition}


\begin{remark}[On the noise model]
Assuming Gaussian noise is convenient for deriving our main results, as they rely on Gaussian-based calculations. However, asymptotic properties of classical spiked models are known to be universal with respect to the i.i.d. noise distribution, under mild assumptions such as the existence of a finite fourth moment. Similar universality is expected to hold in our setting as well, although establishing a rigorous proof lies beyond the scope of this work. 
For instance, one could employ the \emph{interpolation trick} from \cite{lytova_pastur_2009} to handle non-Gaussian variables, which typically requires controlling higher-order moments and bounding the resulting interpolation errors.
Moreover, for the purposes of our analysis, we may assume without loss of generality that the noise is standardized. Any variance factor can indeed be absorbed into a rescaling of the signal matrices, so that the effective signal-to-noise ratio is entirely captured by the magnitude of the signal components. 
Working with standardized noise thus simplifies the exposition while preserving the generality of the model.
\end{remark}

Several variants of this information-plus-noise model have already been introduced, notably for canonical correlation analysis within a Bayesian framework \cite{klami2013bayesian}, as well as for applications in genomics \cite{JIVE-genomics}, metabolomics \cite{JIVE-metabolomic}, and neuroscience \cite{JIVE-neuroscience}. These approaches are all built upon the {\it joint and individual variation explained} (JIVE) framework \cite{JIVE-genomics}, which employs the same decomposition: a low-rank component capturing the joint structure shared by $\mX$ and $\mY$, low-rank components specific to each data type, and a residual noise term.
Such models share similarities with \cite{trygg2003o2pls}, which explicitly accounts for structured components in the residuals. 
To ensure that the problem is non-trivial, we impose the following assumptions.

\begin{assumption}
In the asymptotic regime described in Assumption~\ref{ass:asympt}, we have
\[
\begin{aligned}
\Norm{\tfrac{1}{p} \mP\mP^\trans} &= \mathcal{O}(1), &
\Norm{\tfrac{1}{p} \mM^\trans\mM} &= \mathcal{O}(1), &
\Norm{\tfrac{1}{\sqrt{p}} \mM \mP} &= \mathcal{O}(1), &
r_M &= \mathcal{O}(1), \\
\Norm{\tfrac{1}{q} \mR\mR^\trans} &= \mathcal{O}(1), &
\Norm{\tfrac{1}{q} \mN^\trans\mN} &= \mathcal{O}(1), &
\Norm{\tfrac{1}{\sqrt{q}} \mN \mR} &= \mathcal{O}(1), &
r_N &= \mathcal{O}(1),
\end{aligned}
\]
where we recall that $\Norm{\cdot}$ is the spectral norm, \ie the matrix norm induced by the Euclidean norm, and $r_M$ (resp. $r_N$) is the rank of the matrix $\mM$ (resp. $\mN$).
\label{ass:nontriviality}
\end{assumption}

\noindent  The first two columns of assumptions in \ref{ass:nontriviality} are standard non-trivial conditions from random matrix theory. For example, suppose that the entries of $\bar{\mP}$ defined in \eqref{eq:model_gen} are random, i.i.d., and  $\mathcal{N}(0, 1)$ distributed. Then the conditions
$\Norm{\frac1p \mP\mP^\trans}= \mathcal{O}(1)$ and $\Norm{\frac1p \mM^\trans\mM}= \mathcal{O}(1)$  are satisfied almost surely.

The third column of assumptions is less intuitive. Although $\mM$ and $\mN$ are $n \times p$ and $n \times q$ matrices respectively, they actually have much lower ranks: $r_M = r_U-r$ and $r_N = r_V-r$, where $r_U$ and $r_V$ are the respective ranks of the matrices $\mU$ and $\mV$ defined in model \eqref{eq:model_gen}, and $r$ is the rank of the joint component. 
Using the construction of \eqref{eq:model} from \eqref{eq:model_gen}, we can write:
\begin{equation*}
    \frac{1}{\sqrt{p}} \mM \mP = \frac{1}{\sqrt{p}} \mQ_U\mO_U^\perp\left(\mO_U^\perp\right)^\top \mR_U {\bar{\mP}}^\trans \bar{\mP}\mR_U^\top\mO_U^{(r)} = \mQ_U\mO_U^\perp \bar{\mZ},
\end{equation*}
where $\bar{\mZ}=\frac{1}{\sqrt{p}}\left(\mW\mO_U^{\perp}\right)^\trans \mW\mO_U^{(r)}$ and $\mW=\bar{\mP}\mR_U^\top$.
Assuming that the entries of $\mW$ are random, i.i.d. and  $\mathcal{N}(0, 1)$ distributed, the matrices $\mW\mO_U^{(r)}$ and $\mW\mO_U^{\perp}$ have independent columns, which are furthermore mutually independent. Therefore, the matrix $\bar{\mZ}$ has entries with zero mean and unit variance. Since $\bar{\mZ} \in \R^{r_{M}\times r}$ is low-dimensional, we deduce that
\begin{equation*}
    \Norm{\frac{1}{\sqrt{p}} \mM \mP} = \Norm{\mQ_U\mO_U^\perp \bar{\mZ}} \stackrel{a.s.}{=} \mathcal{O}(1).
\end{equation*}
The corresponding conditions on $\mR$ and $\mN$ follow symmetrically.

Finally, the last column in Assumption~\ref{ass:nontriviality} is satisfied when the ranks of the individual components $\mM$ and $\mN$ remain fixed and finite. This standard low-rank assumption is essential for identifying potential spikes based on the noise-only model. The joint signal components already possess a finite-rank structure by construction in \eqref{eq:model}, as the factors $\mT_X$, $\mT_Y$, $\mP$ and $\mR$ have a fixed number $r$ of columns, which is assumed to correspond to their rank.

\section{Analysis of PLS-SVD in the High-Dimensional Regime}
\label{sec:high_dim_analysis}

This Section is dedicated to the analysis of the singular decomposition of $\mSXY=\frac{1}{\sqrt{pq}}\mX^\trans\mY$. More specifically, we will characterize the limiting singular distribution of $\mSXY$ (Proposition~\ref{prop:lsd}), as well as the spiked singular values due to the signal components of our model (Propositions~\ref{prop:isolated_ST} and \ref{prop:isolated_PR}). Finally, we characterize the alignments of the associated singular vectors (Proposition~\ref{prop:alignment_ST} and \ref{prop:alignment_PR}). This analysis is performed by means of the kernel matrices $\mK$ and $\tilde{\mK}$, defined respectively in \eqref{eq:K} and \eqref{eq:K_tilde}. As these objects are defined symmetrically, we often focus on the  kernel $\mK=\frac{1}{pq}\mY^\trans\mX\mX^\trans\mY$, and deduce the results for $\tilde{\mK}$. All the results of this section are obtained under Assumptions~\ref{ass:asympt} and \ref{ass:nontriviality} discussed in Section~\ref{sec:model}.

\subsection{Deterministic equivalents}

To analyze the spectral decomposition of the kernel $\mK$, we introduce a tool called the resolvent matrix, which is $\mQ(z)=(\frac{1}{pq}\mY^\trans\mX\mX^\trans\mY-z\mI_q)^{-1}$. Let us write the eigenvalue decomposition $\mK=\mU\mLambda\mU^\trans$, with $\mU=[\vu_1,\dots,\vu_q]\in\R^{q\times q}$ the eigenvectors of $\mK$ and $\mLambda$ the diagonal matrix with diagonal entries $\left(\lambda_1,\dots,\lambda_q\right)$ the eigenvalues of $\mK$. Then we have the interesting decomposition $\mQ(z) = \sum\limits_{i=1}^q \frac{\vu_i\vu_i^\trans}{\lambda_i-z}$,
which means that the eigendecomposition of $\mK$ can be deduced from the analysis of $\mQ(z)$, and more specifically the eigenvalues of $\mK$ can be read as the singular points of $\mQ(z)$. In order to make the statistical analysis feasible, we will replace the resolvent $\mQ(z)$ by an object $\bar{\mQ}(z)$, for which scalar observations will be asymptotically identical to the scalar observations of the resolvent itself. Such an object is called a deterministic equivalent, defined as follows.

\begin{definition}[Deterministic Equivalent] \label{def:matrix_equivalent}
Let $\mZ$ be a random matrix and $\bar{\mZ}$ be a deterministic matrix, both in $\R^{m \times m}$. 
We write $\mZ \leftrightarrow \bar{\mZ}$ if, for any deterministic matrix $\mA \in \R^{m \times m}$ and vectors $\va, \vb \in \R^m$ of bounded norms (spectral and Euclidean norms respectively)
$$
\frac{1}{m} \Tr \left(\mA \left( \mZ - \bar{\mZ} \right)\right) \xrightarrow[m \to \infty]{\text{a.s.}} 0, \quad \va^\trans \left(  \mZ - \bar{\mZ} \right) \vb \xrightarrow[m \to \infty]{\text{a.s.}} 0.
$$
The matrix  $\bar{\mZ}$ is called a  \emph{deterministic equivalent} of $\mZ$. For more details on this notion, see \cite[section 2.1.4]{couillet_liao_2022}.
\end{definition}

We are interested in the limiting spectral distribution $\mu$ of the squared singular values of $\mSXY$. We will characterize this distribution $\mu$ through a polynomial equation over its Stieltjes transform, which is defined in Definition~\ref{def:ST}. 

\begin{definition}[Stieltjes Transform]
For a real probability measure $\mu$ with support $\Supp(\mu)$, the Stieltjes transform $m_\mu(z)$ is defined, for all $z\in\C\backslash\Supp(\mu)$, as:
\begin{equation*}
    m_\mu(z) \equiv \int \frac{1}{t-z}\mu(dt)
\end{equation*}
\label{def:ST}
\end{definition}

We denote by $\mathbb{C}^+ = \{z \in \mathbb{C} : \Im(z) > 0\}$ the upper half-plane of the complex plane. For $z \in \mathbb{C}^+$, the Stieltjes transform $m(z)$ of a probability measure is uniquely determined as the solution to its defining equation satisfying $\Im(m(z)) > 0$ (or equivalently $m(z) \in \mathbb{C}^+$). For real values outside the spectral support, $z \in \mathbb{R} \setminus \text{Supp}(\mu)$, $m(z)$ is uniquely defined by analytic continuation from $\mathbb{C}^+$ to the real axis.


Given $\mX \in \R^{n \times p}$ and $\mY \in \R^{n \times q}$ from \eqref{eq:model},  
the resolvent matrix $\mQ(z)$ of $\mK$ defined in \eqref{eq:K}, and the coresolvent matrix $\tilde{\mQ}(z)$ of $\tilde{\mK}$ defined in \eqref{eq:K_tilde}, are given by
\[
\mQ(z) = \left( \frac{1}{pq}\,\mY^\trans \mX \mX^\trans \mY - z \mI_q \right)^{-1},
\qquad
\tilde{\mQ}(z) = \left( \frac{1}{pq}\,\mX^\trans \mY \mY^\trans \mX - z \mI_p \right)^{-1}.
\]
Theorem~\ref{thm:eq_det} then provides deterministic equivalents for these key quantities. A main advantage of the deterministic equivalent approach is that, starting from the resolvent and using concentration of traces, one can derive the limiting linear spectral statistics of the kernel matrices, as well as the limiting spectral distribution. Moreover, the isolated eigenvalues are given by the singular points of the deterministic equivalent. Lastly, the concentration of bilinear forms makes it possible to analyze the alignment of the associated eigenvectors. This latter property will be of primary interest in our spiked model analysis.
 
\begin{theorem}[Deterministic equivalent]
 Some deterministic equivalents $\bar{\mQ}$ of $\mQ$, and $\bar{\tilde{\mQ}}$ of $\tilde{\mQ}$, are given by 
 $$\bar{\mQ}(z) = -\frac{1}{z \tilde{m}(z)} \bar{\mQ}_{Y}\left( \frac{-1}{\tilde{m}(z)} \right) 
 \quad \text{ and } \quad 
 \bar{\tilde{\mQ}} = -\frac{1}{z m(z)} \bar{\mQ}_{X}\left( \frac{-1}{m(z)} \right),$$
 where $\tilde{m}(z) = \frac{q}{p}m(z) -\frac{1-\frac{q}{p}}{z}$.
 For $z \in \mathbb{C}^+$, $m(z)$ is the unique solution in $\mathbb{C}^+$ of:
    \begin{equation}
    -m^3(z)\frac{\beta_p}{\beta_q} z^2 + m^2(z)\left(1+\beta_p-2\frac{\beta_p}{\beta_q}\right)z + m(z)\left[z-\left(1-\beta_q\right)\left(\frac{\beta_p}{\beta_q}-1\right)\right] + 1 = 0.
    \label{eq:m(z)}
    \end{equation}
The matrices $\bar{\mQ}_{X}$ and $\bar{\mQ}_{Y}$ are expressed as:
 \begin{align}
\bar{\mQ}_{X}(z) &= \left( \frac1{m_{X}(z)} \mI_p + \frac{1}{1+m_{X}(z)} \frac1p \left(\mM^\trans \mM + \mP\mP^\trans \right)
+  \right. \nonumber\\
& \left. \frac{1}{(1+m_{X}(z)) m_{X}(z)} \frac1{pq} \mP\mSigma\mR^\trans \left( \frac{1+m_{X}(z)}{m_{X}(z)}\mI_q + \frac1q\mR\mR^\trans \right)^{-1} \mR\mSigma\mP^\trans \right)^{-1},
\label{eq:QX_P_R_S}
\end{align}
\begin{align}
\bar{\mQ}_{Y}(z) &= \left( \frac1{m_{Y}(z)} \mI_q + \frac{1}{1+m_{Y}(z)} \frac1q \left(\mN^\trans \mN + \mR\mR^\trans \right)
+  \right. \nonumber\\
& \left. \frac{1}{(1+m_{Y}(z)) m_{Y}(z)} \frac1{pq} \mR\mSigma\mP^\trans \left( \frac{1+m_{Y}(z)}{m_{Y}(z)}\mI_p + \frac1p\mP\mP^\trans \right)^{-1} \mP\mSigma\mR^\trans \right)^{-1}.
\label{eq:QY_P_R_T}
\end{align}
Finally, for $z \in \mathbb{C}^+$, $m_X(z)$ and $m_Y(z)$ are respectively the unique solutions in $\mathbb{C}^+$ of:
\begin{align}
   z m^2_{X}(z) &- \left(\beta_p - 1 -z\right) m_{X}(z) + 1 = 0 \label{eq:XMP}, \\
   z m^2_{Y}(z) &- \left(\beta_q - 1 -z\right) m_{Y}(z) + 1 = 0 \label{eq:YMP}.
\end{align}
\label{thm:eq_det}
\end{theorem}
\begin{proof}
The proof of Theorem~\ref{thm:eq_det} relies on Gaussian-based calculations.  We first use the Nash--Poincaré inequality to show that the resolvent matrix admits, in the sense of Definition~\ref{def:matrix_equivalent}, a deterministic equivalent given by its expectation.  
We then apply Stein's lemma to derive an asymptotic expression for this expected resolvent.  
The full derivation is presented in Appendix~\ref{app:proof_ED}. 
\end{proof}
\begin{remark}[On the Mar\v{c}enko--Pastur Stieltjes transform]
The Stieltjes transforms  
$m_X(z)$ and $m_Y(z)$ involved in Theorem~\ref{thm:eq_det} can be defined alternatively using the Stieltjes transform of the well-known Mar\v{c}enko--Pastur distribution. Indeed, $m_X(z)=\beta_p^{-1} m_{\mathrm{MP};\beta_p^{-1}}\left(z \beta_p^{-1}\right)$ and $m_Y(z)=\beta_q^{-1} m_{\mathrm{MP};\beta_q^{-1}}\left(z \beta_q^{-1}\right)$, where $m_{\mathrm{MP};c}(\tilde{z})$ is the Stieltjes transform of the Mar\v{c}enko--Pastur distribution with aspect ratio $c$.
\end{remark}

The expressions obtained in Theorem \ref{thm:eq_det} highlight the complementary nature of the two deterministic equivalents. We can already observe that $\bar{\mQ}(z)$, which contains information about the right singular vectors of $\mSXY$, depends, among the individual terms, only on  $\mN$, whereas $\bar{\tilde{\mQ}}(z)$ associated with the left singular vectors, depends only on the individual terms $\mM$.

\subsection{Limiting singular distribution}

\begin{proposition}[Limiting Singular Distribution of $\mSXY$]
    Let $\mX \in \R^{n \times p}$ and $\mY \in \R^{n \times q}$ as in \eqref{eq:model}. We define $\mu$ as the limiting  distribution of the {\em squared} singular values of $\mSXY$. Then the density $f$ of $\mu$ is given on $\Supp(\mu)\backslash\{0\}$ as $f(x)=\frac{\beta}{\pi} \Im(\bar{m}(x))$, where $\bar{m}(z)$ is defined as the complex solution with positive imaginary part of
        \begin{equation}
         - \bar{m}^{3}(x) \beta_p \beta_q x^2 + \bar{m}^{2}(x) \left(\beta_p + \beta_q -2 \beta_p\beta_q\right) x + \bar{m}(x) \left[x- (1 -\beta_p) (1-\beta_q)\right] + 1 = 0.
         \label{eq:m_bar}
        \end{equation}
    The support of this density is given by $\Supp(\mu)\backslash\{0\}=(x_-,x_+)$, where $x_-$ and $x_+$ are the two real nonnegative roots of the third-order polynomial $\Delta(x)$, with the expression given in Equation \eqref{app:discriminant}.
    If $\beta<1$, there is an additional mass at zero $\mu(\{0\})=1-\beta$.
    
    Moreover, when $\mX$ and $\mY$ are noise-only matrices (\ie in the specific case where $\mP$, $\mR$, $\mM$ and $\mN$ are null matrices), we have the confinement of the spectrum. More precisely:
    \begin{equation}
        \max \lambda(\mSXY) \xrightarrow[n,p,q \to \infty]{\text{a.s.}} x_+,
        \label{eq:noeig}
    \end{equation}
    where $x_+$ is the right edge of $\Supp(\mu)$.
    \label{prop:lsd}
\end{proposition}

This result was recently established in the same form for noise-only matrices \cite{swain2025distribution}  (up to a difference in the normalization of the cross-covariance matrix $\mX^\trans \mY$  which is $1/n$ in their setting rather than $1/\sqrt{pq}$ in ours). For completeness, we provide our proof and an illustrative figure in Appendix~\ref{app:lsd} (our derivation follows directly from the deterministic equivalent established in Theorem~\ref{thm:eq_det}, whereas \cite{swain2025distribution} relies on free probability techniques and $S$-transforms).

Note that the presence of spikes does not alter the limiting distribution, since their number is fixed under the low-rank assumptions and their contribution becomes negligible in the limit.  In addition, Eq.~\eqref{eq:noeig} shows a confinement result for the spectrum: in the noise-only model, no eigenvalue escapes beyond the upper edge. This confinement property will be crucial for identifying isolated singular values as spike-induced outliers, which in turn carry information about the underlying signal components in the data matrices.

\begin{remark}[On the squared singular values]
For simplicity, we derived the limiting distribution of the \emph{squared} singular values of $\mSXY$.  
Note that the nonzero squared singular values of $\mSXY$ coincide with the nonzero eigenvalues of the kernel matrices \eqref{eq:K} and \eqref{eq:K_tilde}.  
Thus, their limiting spectral distributions can be inferred directly from those of the kernels, up to a rescaling that accounts for the differing numbers of zero eigenvalues.  
Finally, the limiting density of the \emph{nonzero singular values} themselves follows the standard change of variables $f_{\sigma}(\sigma) = 2 \sigma f(\sigma^2)$.
\end{remark}

\subsection{Singular vectors due to the specific components}

The deterministic equivalents given by Theorem~\ref{thm:eq_det} indicate that the deterministic left (resp.~right) singular vectors of $\mSXY$ only depend on $\mP, \mR$ and $\mM$ (resp.~$\mP, \mR$ and $\mN$). In fact, we can go further by showing that the left (resp.~right) singular space of $\mSXY$ generated by $\mM$ (resp.~$\mN$) is asymptotically orthogonal to the left (resp.~right) singular space generated by the joint components $\mP$ and $\mR$.

\begin{lemma}[Asymptotic orthogonality of the eigenspaces]
    Let us define $\mK_M \equiv \frac1p \mM^\trans\mM$ (resp.~$\mK_N \equiv \frac1q \mN^\trans\mN$) the kernel associated with $\mM$ (resp.~$\mN$). Then we have:
    \begin{itemize}
        \item $\frac1p \mM\bar{\mQ}_X(z)\mM^\trans = \frac1p \mM\left( \frac1{m_{Y}(z)} \mI_p + \frac{1}{1+m_{Y}(z)}\mK_M\right)\mM^\trans + o(1)$ 
        \item $\frac1q \mN\bar{\mQ}_Y(z)\mN^\trans = \frac1q\mN\left( \frac1{m_{Y}(z)} \mI_q + \frac{1}{1+m_{Y}(z)}\mK_N\right)\mN^\trans + o(1)$ 
    \end{itemize}
\label{lem:ortho}
\end{lemma}

Therefore, the singular values due to $\mM$ and $\mN$ can be deduced separately from the singular values due to $\mP$ and $\mR$. This result, which is a consequence of the assumptions $\tfrac{1}{\sqrt{p}}\Norm{\mM \mP} = \mathcal{O}(1)$ and 
$\tfrac{1}{\sqrt{q}}\Norm{\mN \mR} = \mathcal{O}(1)$ from Assumption~\ref{ass:nontriviality}, will be illustrated later in Figure~\ref{fig:combined_spikes}. Accordingly, we will first describe the isolated singular values associated with the individual components $\mM$ and $\mN$ (Proposition~\ref{prop:isolated_ST}) as well as the alignment of the associated singular vectors (Proposition~\ref{prop:alignment_ST}). Then, we will do the same for the isolated singular values and singular vectors associated with the joint components $\mP$ and $\mR$ (Propositions~\ref{prop:isolated_PR} and \ref{prop:alignment_PR}). For each of these singular values, there exists a threshold above which signals yield isolated singular values. Below this threshold, the singular values live in the ``bulk'' distribution described in Proposition~\ref{prop:lsd}, and are indistinguishable from noise. We prove in Appendix~\ref{app:proof_ST} that this threshold can be characterized as the largest positive root, denoted as $\tau>0$, of the following polynomial equation:
\begin{align}
\lambda^3 - \lambda (\beta_p \beta_q + \beta_p + \beta_q) - 2 \beta_p \beta_q =0,
\label{eq:thresh_pol}
\end{align}
which can be expressed in closed-form as
\begin{align}
\tau & = 2 \sqrt{\frac{\beta_p \beta_q+\beta_p+\beta_q}{3}} \cos{\left(
\frac13\operatorname{acos}{ \left(\beta_p \beta_q  \left(\frac{\beta_p \beta_q+\beta_p+\beta_q}{3}\right)^{-3/2} \right) }  \right)}.
\label{eq:thresh}
\end{align}

\begin{proposition}[Isolated singular values associated with $\mM$ and $\mN$]
Let $\tau$ be the phase transition threshold defined in \eqref{eq:thresh}. For all $k \in [r_M]$ (resp.~$k \in [r_N]$), let us denote $\lambda_{M,k}$ (resp.~$\lambda_{N,k}$) the $k$th nonzero eigenvalue of $\mK_M \equiv \frac1p \mM^\trans\mM$ (resp.~$\mK_N \equiv \frac1q \mN^\trans\mN$). Then:

\begin{itemize}
    \item For all $k \in [r_M]$, if $\lambda_{M,k} > \tau$, there exists an isolated {\em squared} singular value of $\mSXY$, denoted $\hat{\lambda}_{M,k} \equiv \lambda_{M,k}(\mSXY)$, that is mapped to $\lambda_{M,k}$, and such that:
    \begin{align*}
        \hat{\lambda}_{M,k} \xrightarrow[n,p,q \to \infty]{\text{a.s.}} \xi_{M,k} &= \frac{(\lambda_{M,k}+1) (\lambda_{M,k}+\beta_p) (\lambda_{M,k}+\beta_q)}{\lambda_{M,k}^2}.
     \end{align*}
    \item For all $k \in [r_N]$, if $\lambda_{N,k} > \tau$, there exists an isolated {\em squared} singular value of $\mSXY$, denoted $\hat{\lambda}_{N,k} = \lambda_{N,k}(\mSXY)$, that is mapped to $\lambda_{N,k}$, and such that:
    \begin{align*}
        \hat{\lambda}_{N,k} \xrightarrow[n,p,q \to \infty]{\text{a.s.}} \xi_{N,k} &= \frac{(\lambda_{N,k}+1) (\lambda_{N,k}+\beta_p) (\lambda_{N,k}+\beta_q)}{\lambda_{N,k}^2}.
     \end{align*}
\end{itemize}
\label{prop:isolated_ST}
\end{proposition}

Proposition~\ref{prop:isolated_ST} shows that, above the detection threshold, the estimators $\hat{\lambda}_{M,k}$ (resp. $\hat{\lambda}_{N,k}$) are upwardly biased for the true signal strengths $\lambda_{M,k}$ (resp. $\lambda_{N,k}$), with a bias that vanishes only in the limit where the spike strength tends to infinity. The following proposition allows us to characterize now the limiting behavior of the associated singular vectors.

\begin{proposition}[Alignment of the top singular vectors associated with $\mM$ and $\mN$]
Let $\vu_{M,k}  \in \R^p$ (resp.~$\vv_{N,k}  \in \R^q$) be the left singular vector of $\mM$ (resp.~right singular vector of $\mN$) associated with the $k$th squared singular value $\lambda_{M,k}$ (resp.~$\lambda_{N,k}$). Let $\hat{\vu}_{M,k} \equiv \vu_{M,k}(\mSXY)$ (resp.~$\hat{\vv}_{N,k} \equiv \vv_{N,k}(\mSXY)$) denote the corresponding left (resp.~right) singular vector of $\mSXY$ 
associated with the squared singular value $\hat{\lambda}_{M,k}$ (resp.~ $\hat{\lambda}_{N,k}$) defined in Prop.~\ref{prop:isolated_ST}.
Then
\begin{align}
        \langle \hat{\vu}_{M,k}, \vu_{M,k} \rangle^2  \xrightarrow[n,p,q \to \infty]{\text{a.s.}} 
        \zeta_{M,k} = \begin{cases}
        &  \frac{\lambda_{M,k}^3 - \left(\beta_p \beta_q +\beta_p+\beta_q\right) \lambda_{M,k} - 2 \beta_p \beta_q}{\lambda_{M,k} \left(\lambda_{M,k} + 1\right)\left(\lambda_{M,k} + \beta_q\right)} \quad \textrm{ if  $\lambda_{M,k}>\tau$,}\\
        & 0  \quad \textrm{ otherwise},
        \end{cases}
        \label{eq:align_S}
\end{align}
\begin{align}
        \langle \hat{\vv}_{N,k}, \vv_{N,k} \rangle^2  \xrightarrow[n,p,q \to \infty]{\text{a.s.}} \zeta_{N,k}=
        \begin{cases}
        &\frac{\lambda_{N,k}^3 - \left(\beta_p \beta_q +\beta_p+\beta_q\right) \lambda_{N,k} - 2 \beta_p \beta_q}{\lambda_{N,k} \left(\lambda_{N,k} + 1\right)\left(\lambda_{N,k} + \beta_p\right)} \quad \textrm{ if  $\lambda_{N,k}>\tau$,}\\
        & 0  \quad \textrm{ otherwise},
        \end{cases}
        \label{eq:align_T}
\end{align}
where $\tau$ is the phase transition threshold defined in \eqref{eq:thresh}.

Furthermore, if $\hat{\vu}_{N,k} \equiv \vu_{N,k}(\mSXY)$ (resp.~$\hat{\vv}_{M,k} \equiv \vv_{M,k}(\mSXY)$) denotes the left (resp.~right) singular vector of $\mSXY$ associated with the $k$th squared singular value $\hat{\lambda}_{N,k}$ (resp.~$\hat{\lambda}_{M,k}$), 
then for any deterministic vector $\omega_u\in\R^p$ (resp.~$\omega_v\in\R^q$) with bounded norm,
\begin{align*}
        \langle \hat{\vu}_{N,k}, \omega_u \rangle^2  &\xrightarrow[n,p,q \to \infty]{\text{a.s.}} 0, \\
        \langle \hat{\vv}_{M,k}, \omega_v \rangle^2  &\xrightarrow[n,p,q \to \infty]{\text{a.s.}} 0.
\end{align*}

\label{prop:alignment_ST}
\end{proposition}
\begin{proof}
Proofs of Propositions~\ref{prop:isolated_ST} and \ref{prop:alignment_ST} are presented in Appendices~\ref{app:proof_ST} and \ref{app:proof_align_ST}, respectively.
\end{proof}

Proposition~\ref{prop:alignment_ST} highlights a fundamental limitation of the PLS method in the high-dimensional regime. Since PLS is designed to align only with the components shared by $\mX$ and $\mY$, the matrices $\mM$ and $\mN$ correspond to uninformative signal components. In the special case where no noise is present, \ie when $\mE = \m0$ and $\mF = \m0$, the orthogonality conditions \ref{cond:ortho} imply that 
$\mSXY = \tfrac{1}{\sqrt{pq}}\, \mP\mSigma\mR^\trans,$
so that the singular vectors depend only on the joint components, as expected. However, in the presence of noise, Proposition~\ref{prop:alignment_ST} shows that the leading singular vectors of $\mSXY$ may nonetheless align with components that are irrelevant and specific to each individual matrix. These individual components are artificially carried by the noise, giving rise to spurious singular directions in the PLS analysis.

Figure~\ref{fig:spikes_ST} illustrates the theoretical predictions of Propositions~\ref{prop:isolated_ST} and~\ref{prop:alignment_ST} for the isolated singular values and alignments induced by the individual components $\mM$ and $\mN$. The left panel displays the empirical distribution of the squared singular values of $\mSXY$, showing clear spikes emerging from the bulk distribution at locations that match precisely the theoretical predictions $\xi_{M,k}$ and $\xi_{N,k}$. The right panel demonstrates the quality of our alignment predictions: the theoretical curves $\zeta_{M,k}$ and $\zeta_{N,k}$ as functions of $\lambda_M$ and $\lambda_N$ are overlaid with the empirical alignments (shown as bars), confirming that our asymptotic characterization accurately captures the finite-sample behavior even for moderate dimensions ($p=800$ and $q=4000$).
\begin{figure}[htbp!]
    \centering
    \includegraphics[width=\linewidth]{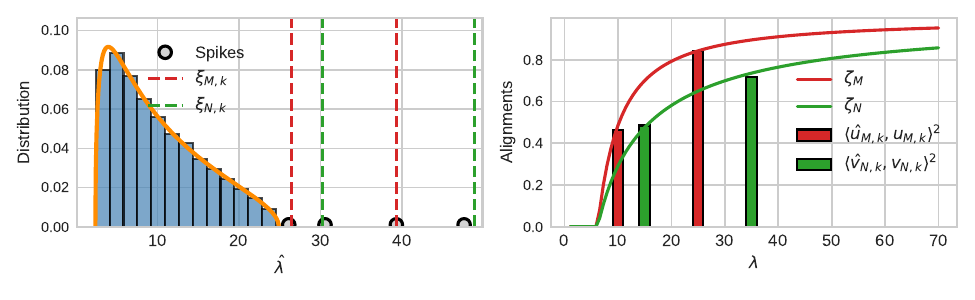}
    \caption{\textbf{(Left)} Empirical distribution of the squared singular values of $\mSXY$, including the spikes generated by $\mM$ and $\mN$, together with the limiting spike locations $\xi_{M,k}$ and $\xi_{N,k}$ predicted by Proposition~\ref{prop:isolated_ST}. 
    \textbf{(Right)} Limiting alignment $\zeta_M$ (resp.~ $\zeta_N$)  as functions of $\lambda_{M}$ (resp.~$\lambda_N$) predicted by Proposition~\ref{prop:alignment_ST}, shown alongside the empirical alignments of the corresponding singular vectors (bars). 
    \textbf{Experimental settings:} $\beta_p = 10$, $\beta_q = 2$ ($n = 8000$), $r_M=2$ with $\lambda_{M,1} = 25$ and $\lambda_{M,2} = 10$, $r_N=2$ with $\lambda_{N,1} = 35$ and $\lambda_{N,2} = 15$, $r=0$ ($\mT=\m0$, \ie no joint component).}
    \label{fig:spikes_ST}
\end{figure}

In addition, the second part of Proposition~\ref{prop:alignment_ST} establishes that the left (resp. right) singular vectors of $\mSXY$ associated with individual components of $\mN$ (resp.~$\mM$) do not align with any deterministic direction beyond their generating component. This illustrates once again a fundamental limitation of PLS: the noise can create spurious singular directions that carry no information about the underlying signal structure.

\subsection{Singular vectors due to the joint components}
\label{sec:common}

We now proceed in a similar manner to characterize the isolated singular values of $\mSXY$ associated with the joint components $\mP$ and $\mR$ (Proposition~\ref{prop:isolated_PR}), as well as the alignments of the corresponding singular vectors (Proposition~\ref{prop:alignment_PR}). 
To begin, we define the kernels associated with $\mP$ and $\mR$ as
\begin{equation}
\mK_P \equiv \frac{1}{p} \mP^\top \mP \quad \text{and} \quad \mK_R \equiv \frac{1}{q} \mR^\top \mR.    
\label{eq:kernel_PR}
\end{equation}
Using these quantities, we formalize the associated quadratic eigenvalue problem.

\begin{definition}[Joint component polynomials and eigenpairs]
\label{def:joint_polynomials}
Let $\mK_P$ and $\mK_R$ be the kernel matrices defined in \eqref{eq:kernel_PR}. We define the $r \times r$ quadratic matrix polynomials $\mGamma(\lambda)$ and $\tilde{\mGamma}(\lambda)$ as:
\begin{align*}
\mGamma(\lambda) &\equiv \lambda^2\mI_r - \lambda\left(\mSigma\mK_P\mSigma^{-1} + \mK_R + \mSigma\mK_P\mSigma\mK_R\right) - \mSigma\mK_P\left(\mSigma-\mSigma^{-1}\right)\mK_R,\\
\tilde{\mGamma}(\lambda) &\equiv \lambda^2\mI_r - \lambda\left(\mSigma\mK_R\mSigma^{-1} + \mK_P + \mSigma\mK_R\mSigma\mK_P\right) - \mSigma\mK_R\left(\mSigma-\mSigma^{-1}\right)\mK_P,
\end{align*}
which are constructed symmetrically by interchanging the roles of $\mK_P$ and $\mK_R$.
An eigenpair $(\lambda_k, \vx_k)$, with $\vx_k \neq \m0$, of the matrix polynomial $\mGamma(\lambda)$  is defined by the condition $\mGamma(\lambda_k)\vx_k = \m0$, which implicitly requires $\det(\mGamma(\lambda_k)) = 0$. Eigenpairs of $\tilde{\mGamma}(\lambda)$ are defined analogously.
\end{definition}

\
To analyze the spectral properties of these polynomials, we cast the quadratic problem into an equivalent linear form, as formalized in the following lemma.

\begin{lemma}[Linearization and spectral properties]
The quadratic matrix polynomials $\mGamma(\lambda)$ and $\tilde{\mGamma}(\lambda)$ introduced in Def.~\ref{def:joint_polynomials} share the exact same spectrum. Furthermore, the associated quadratic eigenvalue problem $\mGamma(\lambda)\vx = \m0$ can be solved by linearizing it into an equivalent $2r \times 2r$ standard linear eigenvalue problem $\mM \vv = \lambda \vv$, where the companion matrix is defined as:
$$\mM = \begin{bmatrix} \m0 & \mI_r \\ \mB & \mA \end{bmatrix},$$
with the block components given by:
\begin{align*}
\mA = \mSigma\mK_P\mSigma^{-1} + \mK_R + \mSigma\mK_P\mSigma\mK_R, \quad
\mB = \mSigma\mK_P\left(\mSigma-\mSigma^{-1}\right)\mK_R.
\end{align*}
The eigenvectors $\vx$ of the original QEP correspond to the first $r$ coordinates of the extended eigenvectors $\vv$. Moreover, the QEP is non-defective,  meaning it admits a full set of $2r$ linearly independent eigenvectors (by symmetry, the  eigenvectors of the QEP associated with $\tilde{\mGamma}(\lambda)$ follow analogously by interchanging the roles of $\mK_P$ and $\mK_R$.). 
Let 
\begin{align}
r_1 = \lvert \{ k \in [r] : \rho_k = 1 \} \rvert,
\label{eq:r1}
\end{align}
denote the number of perfectly aligned components, where the $\rho_k$ are the diagonal entries of the alignment matrix \eqref{eq:sigma}. Then, the spectrum of the QEP consists of exactly $r_1$ zero eigenvalues and $2r - r_1$ strictly positive real eigenvalues. 
\label{lem:existence}
\end{lemma}

\begin{proof}
The non-defectiveness of the QEP and the inertia of its eigenvalues are established by mapping the system to a generalized symmetric eigenvalue problem via a similarity transformation, with the detailed proof deferred to Appendix~\ref{app:proofQEP}.
\end{proof}

The following proposition characterizes the conditions under which isolated singular values induced by joint components separate from the bulk, and provides their almost sure limits.
\begin{proposition}[Isolated eigenvalues associated with $\mP$ and $\mR$]
Let $\tau$ be the phase transition threshold defined in \eqref{eq:thresh}.
For all $k \in [2r-r_1]$, where $r_1$ is the number of perfectly aligned components introduced in \eqref{eq:r1},  let us denote $\lambda_{T,k}$ the $k$th nonzero eigenvalue $\mGamma(\lambda)$ or, equivalently, of $\tilde{\mGamma}(\lambda)$
as introduced in Def.~\ref{def:joint_polynomials} and analyzed in Lemma~\ref{lem:existence}.
Then, if $\lambda_{T,k} > \tau$, there exists an isolated {\em squared} singular value of $\mSXY$, denoted $\hat{\lambda}_{T,k} \equiv \lambda_{T,k}(\mSXY)$, that is mapped to $\lambda_{T,k}$, and such that:
\begin{align*}
    \hat{\lambda}_{T,k} \xrightarrow[n,p,q \to \infty]{\text{a.s.}} \xi_{T,k} &= \frac{(\lambda_{T,k}+1) (\lambda_{T,k}+\beta_p) (\lambda_{T,k}+\beta_q)}{\lambda_{T,k}^2}.
\end{align*}
\label{prop:isolated_PR}
\end{proposition}
Proposition~\ref{prop:isolated_PR} reveals an important and somewhat counterintuitive phenomenon. Even though the true dimension of the latent joint subspace is $r$, partial alignment between two components may give rise to two distinct pairs of spiked singular vectors. This effect can be interpreted as a spurious duplication of spikes. As a consequence, the apparent dimension of the joint latent subspace may be artificially inflated to as much as $2r-r_1$.

Conversely, when all joint components are perfectly aligned, i.e., when $\mSigma=\mI_r$ (so that $r_1=r$),
the characterization of the isolated singular values becomes considerably simpler. In this case, the  nonzero eigenvalues
$\lambda_{T,k}$ are no longer obtained from a quadratic eigenvalue problem. Instead, they are simply the $r$ eigenvalues of an $r\times r$  matrix, as stated in the following corollary.

\begin{corollary}[Perfect alignments]
If all joint components are perfectly \\ aligned, \ie if $\mSigma=\mI_r$, then the matrix polynomial $\mGamma(\lambda)$ reduces to:
\begin{equation*}
    \mGamma(\lambda) = \lambda^2\mI_r - \lambda\left(\mK_P + \mK_R +\mK_P\mK_R\right) = \lambda\left(\lambda\mI_r - \left(\mK_P + \mK_R +\mK_P\mK_R\right)\right).
\end{equation*}
Consequently, $\mGamma(\lambda)$ admits exactly $r$ non-zero eigenvalues $(\lambda_{T,k})_{k=1}^r$, which coincide with the eigenvalues of
\begin{equation*}
\mK_T
\equiv
\mK_P+\mK_R+\mK_P\mK_R,
\end{equation*}
or, equivalently, of
\begin{equation*}
\tilde{\mK}_T \equiv \mK_R+\mK_P+\mK_R\mK_P.
\end{equation*}

\label{cor:values_perfect_align}
\end{corollary}

While understanding the emergence of these spikes is essential, a deeper question concerns the fidelity of their corresponding singular vectors. To this end, the following proposition quantifies the deterministic limit of the  empirical singular vectors of $\mSXY$.
\begin{proposition}[Alignment of the top eigenvectors associated with $\mP$ and $\mR$]
Let $\vu_k\in\R^p$ and $\vv_k\in\R^q$ be the normalized vectors
\begin{align*}
    \vu_k 
    \equiv \frac{\mP\tilde{\vx}_k}{\|\mP\tilde{\vx}_k\|} \quad \text{and} \quad \vv_k 
    \equiv \frac{\mR\vx_k}{\|\mR\vx_k\|},
\end{align*}
where $\vx_k$ and $\tilde{\vx}_k$ are the eigenvectors corresponding to the eigenvalue $\lambda_{T,k}$ of $\mGamma(\lambda)$ and $\tilde{\mGamma}(\lambda)$, respectively (see Def.~\ref{def:joint_polynomials}).

Let $\hat{\vu}_{k} \equiv \vu_{k}(\mSXY)$ (resp.~$\hat{\vv}_{k} \equiv \vv_{k}(\mSXY)$) be the corresponding 
left (resp.~right) singular vector of $\mSXY$ associated with the squared singular value $\hat{\lambda}_{T,k}$ defined in Prop.~\ref{prop:isolated_PR}. Then

\begin{align*}
\langle \hat{\vu}_{k}, \vu_{k} \rangle^2  \xrightarrow[n,p,q \to \infty]{\text{a.s.}} \zeta_{P,k}&=
\begin{cases}
&\frac{\left(\lambda_{T,k}^3 - \left(\beta_p \beta_q +\beta_p+\beta_q\right)\lambda_{T,k} -2\beta_p \beta_q\right)}
{\lambda_{T,k}\left(\lambda_{T,k} + 1\right)\left(\lambda_{T,k}+\beta_q\right)} \eta_{P,k}
\quad \textrm{ if  $\lambda_{T,k}>\tau$,}\\
& 0  \quad \textrm{ otherwise},
\end{cases}
\\
\langle \hat{\vv}_{k}, \vv_{k} \rangle^2  \xrightarrow[n,p,q \to \infty]{\text{a.s.}} 
\zeta_{R,k}&=
\begin{cases}
&\frac{\left(\lambda_{T,k}^3 - \left(\beta_p \beta_q +\beta_p+\beta_q\right)\lambda_{T,k} -2\beta_p \beta_q\right)}
{\lambda_{T,k}\left(\lambda_{T,k} + 1\right)\left(\lambda_{T,k}+\beta_p\right)}
\eta_{R,k}
\quad \textrm{ if  $\lambda_{T,k}>\tau$,}\\
& 0  \quad \textrm{ otherwise},
\end{cases}
\end{align*}
%
where $\tau$ is the phase transition threshold defined in \eqref{eq:thresh}, 
and where $\eta_{P,k}$ and $\eta_{R,k}$ are defined as:
\begin{align*}
    \eta_{P,k} &= \frac{\vx_{k}^\trans\left(\lambda_{T,k}\mI_r-\mK_R\right)\mSigma^{-1}\tilde{\vx}_{k}}{\vx_{k}^\trans\mSigma^{-1}\tilde{\mGamma}'(\lambda_{T,k})\tilde{\vx}_{k}},\\
    \eta_{R,k} &= \frac{\tilde{\vx}_k^\top\left(\lambda_{T,k}\mI_r-\mK_P\right)\mSigma^{-1} \vx_k}{\tilde{\vx}_k^\top\mSigma^{-1}\mGamma'(\lambda_{T,k})\vx_k}.
\end{align*}
with $\mGamma'$ (resp. $\tilde{\mGamma}'$) the derivative w.r.t. $\lambda$ of the polynomial $\mGamma$ (resp. $\tilde{\mGamma}$). 
\label{prop:alignment_PR}
\end{proposition}

\begin{proof}
Proofs of Propositions~\ref{prop:isolated_PR} and \ref{prop:alignment_PR} are displayed respectively in Appendices~\ref{app:proof_RP} and \ref{app:proof_align_RP}.
\end{proof}

As before, the results of this proposition can be simplified when $\mSigma=\mI_r$.
\begin{corollary}[Perfect alignment]
If all joint components are perfectly \\ aligned ($\mSigma=\mI_r$, so that $r_1=r$), the vectors 
$\vx_k$ and $\tilde{\vx}_k$ introduced in Proposition~\ref{prop:alignment_PR} are, respectively, the eigenvectors of $\mK_T$ and $\tilde{\mK}_T$ defined in Cor.~\ref{cor:values_perfect_align}.
Moreover, the quantities $\eta_{P,k}$ and $\eta_{R,k}$ defined in Proposition~\ref{prop:alignment_PR} simplify to
\begin{align*}
    \eta_{P,k}
    &= \frac{\lambda_{T,k}\vx_{k}^\trans\tilde{\vx}_{k} - \vx_{k}^\trans\mK_R\tilde{\vx}_{k}}{\lambda_{T,k}\vx_{k}^\trans\tilde{\vx}_{k}}
    = \frac{\lambda_{T,k} - \frac{\vx_{k}^\trans\mK_R\tilde{\vx}_{k}}{\vx_{k}^\trans\tilde{\vx}_{k}}}{\lambda_{T,k}},\\[6pt]
    \eta_{R,k}
    &= \frac{\lambda_{T,k}\tilde{\vx}_{k}^\trans\vx_{k} - \tilde{\vx}_{k}^\trans\mK_P\vx_{k}}{\lambda_{T,k}\tilde{\vx}_{k}^\trans\vx_{k}}
    = \frac{\lambda_{T,k} - \frac{\tilde{\vx}_{k}^\trans\mK_P\vx_{k}}{\tilde{\vx}_{k}^\trans\vx_{k}}}{\lambda_{T,k}}.
\end{align*}
\label{cor:vectors_perfect_align}
\end{corollary}

Unlike for the specific components,  both Proposition~\ref{prop:alignment_PR} and Corollary~\ref{cor:vectors_perfect_align} show that the singular vectors of $\mSXY$ do not align directly with those of $\mP\mSigma\mR^\trans$. This reveals another unexpected and counterintuitive phenomenon, highlighting a fundamental limitation of PLS algorithms in the high-dimensional regime. Indeed, referring once again to the noiseless setting where $\mSXY=\frac{1}{\sqrt{pq}}\mP\Sigma\mR^\trans$, one could expect the singular vectors of $\mSXY$ to align with those of $\mP\mSigma\mR^\trans$.
Instead, they align with \emph{skewed} versions of these singular vectors. Moreover, the directions of these deterministic limits are not necessarily mutually orthogonal (even though the singular vectors themselves are)! This skewing is once again induced by  noise, and this effect vanishes only in the limit where the signal becomes infinitely strong. The two remarks below provide a heuristic interpretation of the alignment directions of the spikes in the limiting cases.

\begin{remark}[When the joint signal becomes dominant]
In the specific case where both the eigenvalues of $\mK_R$ and $\mK_P$ go to infinity, then the polynomial $\mGamma$ becomes approximately, for large values of $\lambda$, $\mGamma(\lambda)\simeq \lambda^2\mI_r - \lambda\mSigma\mK_P\mSigma\mK_R$. Thus, $(\lambda_k)_k$ and $(\vx_k)_k$ become, respectively, the eigenvalues and eigenvectors of $\mSigma\mK_P\mSigma\mK_R$. Exploiting this last property, we can see that:

\begin{equation*}
    \vv_k \propto \mR\vx_k \propto \mR\mSigma\mK_P\mSigma\mK_R\vx_k \propto \mR\mSigma\mK_P\mSigma\mR^\top\vv_k.
\end{equation*}
Thus, $\vv_k$ is an eigenvector of $\mR\mSigma\mK_P\mSigma\mR^\top$, \ie a right singular vector of $\mP\mSigma\mR^\top$. Symmetrically, we also obtain that $\vu_k$ tends to align with the $k$th left singular vector of $\mP\mSigma\mR^\trans$. Therefore, in this setting, the singular vectors of $\mSXY$ due to joint components align with those of $\mP\mSigma\mR^\trans$.
\label{rem:signal_dominant}
\end{remark}

\begin{remark}[When one of the components becomes dominant]
If one of the components becomes dominant (say $\mK_R \gg \mK_P \simeq \m0$), then the polynomial $\mGamma$ becomes approximately, for large values of $\lambda$, $\mGamma(\lambda)\simeq \lambda^2\mI_r - \lambda\mK_R$. Thus, $(\lambda_k)_k$ and $(\vx_k)_k$ become, respectively, the eigenvalues and eigenvectors of $\mK_R$. Similarly, $(\tilde{\vx}_k)_k$ are the eigenvectors of $\mSigma\mK_R\mSigma^{-1}$, meaning that we have the relation $\tilde{\vx}_k\simeq\mSigma\vx_k$. It follows that: 
\begin{equation*}
    \eta_{P,k}\simeq \frac{\vx_k^\top\left(\lambda_{T,k}\mI_r-\mK_R\right)\vx_k}{\vx_k^\top\left(2\lambda_{T,k}\mI_r-\mK_R\right)\vx_k} \simeq 0 \quad \text{and} \quad \eta_{R,k}\simeq \frac{\vx_k^\top\left(\lambda_{T,k}\mI_r\right)\vx_k}{\vx_k^\top\left(2\lambda_{T,k}\mI_r-\mK_R\right)\vx_k} \simeq 1.
\end{equation*}

In this case, the formula of $\zeta_{R,k}$ is exactly the same as the formula of $\zeta_{N,k}$ given in Prop.~\ref{prop:alignment_ST}. Furthermore, as $\vx_k$ is an eigenvector of $\mK_R$, we have $\vv_k=\vx_k$, meaning that the right singular vectors of $\mSXY$ align with the eigenvectors of $\mK_R$. In fact, we recover the same setting as in Proposition~\ref{prop:alignment_ST} for the specific components (since when $\mK_P \simeq \m0$ there is no longer any joint component), and we recover the same result.
\label{rem:one_dominant}
\end{remark}

Apart from these limiting cases, there is another specific setting in which the singular vectors of $\mSXY$ actually align with those of $\mP\mSigma\mR^\trans$: when the kernels $\mK_P$ and $\mK_R$ are diagonal matrices, as explained in the following remark.

\begin{remark}[When the kernels are diagonal]
If $\mK_P$ and $\mK_R$ are diagonal matrices, then the matrices $\mGamma$ and $\tilde{\mGamma}$ are diagonal as well, which implies $\vx_k=\tilde{\vx}_k=\ve_k$. Therefore, $\vu_k$ and $\vv_k$ are  the $k$th column of $\mP$ and $\mR$ respectively. In this special setting, they coincide exactly with the $k$th left and right singular vectors of $\mP\mSigma\mR^\trans$.
\label{rem:diagonal}
\end{remark}

Figure~\ref{fig:spikes_PR} illustrates the behavior of isolated singular values and alignments due to the joint components $\mP$ and $\mR$, as characterized in Propositions~\ref{prop:isolated_PR} and~\ref{prop:alignment_PR}. The left panel shows the empirical distribution of squared singular values together with the predicted spike locations $\xi_{T,k}$, confirming the accuracy of our phase transition analysis. The right panel displays the theoretical alignment curves $\zeta_{P,k}$ and $\zeta_{R,k}$ as functions of $\lambda_{T,k}$, alongside the empirical alignments $\hat{\zeta}_{P,k}$ and $\hat{\zeta}_{R,k}$ (bars). Notably, the alignment strengths vary across singular values of different orders, as each singular vector aligns with a distinct skewed version of the corresponding singular vector of $\mP\mSigma\mR^\trans$, illustrating the noise-induced distortion effect in high-dimensional settings. Furthermore, because the second joint component is partially aligned ($\rho=0.2$), an additional spike emerges in this setting, illustrating the duplication phenomenon characterized in Proposition~\ref{prop:isolated_PR}. Thus, three distinct spikes appear, even though the joint components only have a rank of two.
\begin{figure}[htbp!]
    \centering
    \includegraphics[width=\linewidth]{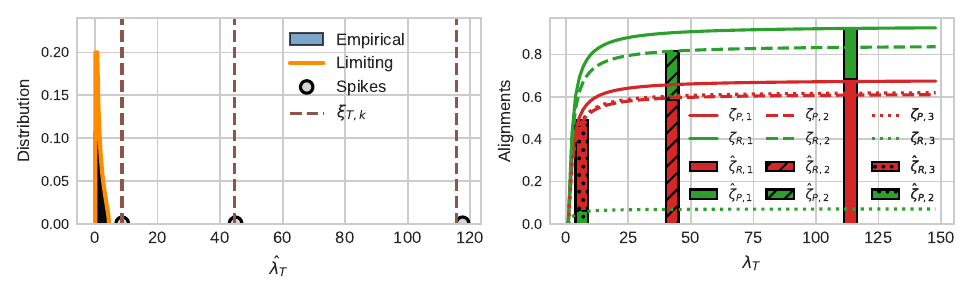}
    \caption{\textbf{(Left)} Empirical distribution of the squared singular values of $\mSXY$ together with the limiting spike locations $\xi_{T,k}$ predicted by Proposition~\ref{prop:isolated_PR}. 
    \textbf{(Right)} Limiting alignments $\zeta_{P,k}$ and $\zeta_{R,k}$ as functions of $\lambda_{T,k}$, as predicted by Proposition~\ref{prop:alignment_PR}, shown alongside the empirical alignments of the corresponding singular vectors (bars).
    \textbf{Experimental settings:} $\beta_p = 0.5$, $\beta_q = 0.5$ ($n = 500$), $r=2$ with $\lambda_{P,1} = \lambda_{P,2} = 10$, $\lambda_{R,1} = 50$, $\lambda_{R,2} = 5$, $\rho_1=0.2$ and $\rho_2=1$ ($\lambda_{T,1} = 113.7$, $\lambda_{T,2} = 42.5$, and $\lambda_{T,3} = 6.41$).}
    \label{fig:spikes_PR}
\end{figure}

Figure~\ref{fig:combined_spikes} demonstrates the simultaneous presence of spikes from both individual ($\mM$) and joint ($\mT$) components, providing empirical validation for Lemma~\ref{lem:ortho}. The left panel shows that the spike $\xi_{M,1}$ from the individual component $\mM$ ($r_M=1$) and the spike $\xi_{T,1}$ from the joint component ($r=1$, with $\rho_1=1$) both emerge clearly from the bulk distribution and are well separated. The right panel confirms that the alignment predictions $\zeta_{M,1}$, $\zeta_{P,1}$, and $\zeta_{R,1}$ remain accurate even in this combined setting. This figure illustrates how PLS can simultaneously detect signals from different structural sources, though as Proposition~\ref{prop:alignment_ST} warns, the individual spikes correspond to uninformative directions that should ideally be filtered out.
\begin{figure}[htbp!]
    \centering
    \includegraphics[width=\linewidth]{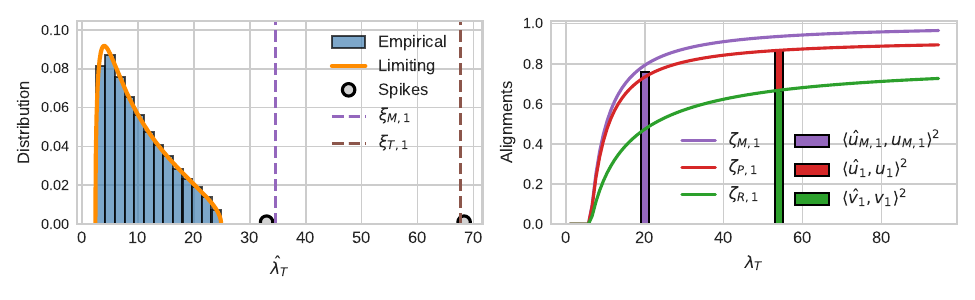}
    \caption{Spike locations and alignments for both specific and joint components. The matrices $\mM$ and $\mP$ are generated from independent Gaussian variables, which naturally ensures that Assumption~\ref{ass:nontriviality} is satisfied, even in the absence of strict orthogonality between the column spaces of $\mM^\trans$ and $\mP$.
    \textbf{(Left)} Empirical distribution of the squared singular values of $\mSXY$ together with the limiting spike locations $\xi_{M,1}$ predicted by Proposition~\ref{prop:isolated_ST} and $\xi_{T,1}$ predicted by Proposition~\ref{prop:isolated_PR}.
    \textbf{(Right)} Limiting alignment $\zeta_{M,1}$ as functions of $\lambda_{M,1}$, and $\zeta_{P,1}$ and $\zeta_{R,1}$ as functions of $\lambda_{T,1}$, as predicted respectively by Propositions~\ref{prop:alignment_ST} and \ref{prop:alignment_PR}, shown alongside the empirical alignments of the corresponding singular vectors (bars).
    \textbf{Experimental settings:} $\beta_p = 10$, $\beta_q = 2$ ($n = 8000$), $r_M=1$ with  $\lambda_{M,1} = 20$, $r=1$ with $\lambda_{P,1} =  10$, $\lambda_{R,1} = 4$  and $\rho_{1} = 1$ ($\lambda_{T,1} = 54$, $\tilde{\lambda}_{P,1} =  10$, $\tilde{\lambda}_{R,1} = 4$).
    \label{fig:combined_spikes}}
\end{figure}

Figure~\ref{fig:skewed_vectors} provides a visualization of the skewing phenomenon characterized in Proposition~\ref{prop:alignment_PR}. With the setting of Fig.~\ref{fig:spikes_PR}, the figure compares the second top singular vectors of $\mX^\trans\mY$ (\ie the top singular vectors corresponding to a partially aligned component in our model \eqref{eq:sigma}, with a cosine $\rho=0.2$) with the corresponding singular vectors of the signal matrix $\mP\mSigma\mR^\trans$.
Surprisingly, even though $\mX^\trans\mY$ contains the signal $\mP\mSigma\mR^\trans$ plus noise, the recovered singular vectors do not converge to the true signal directions. Instead, they align with skewed versions as predicted by our theory. This phenomenon is particularly pronounced for the left singular vector. This shows that the noise-induced distortion is not a finite-sample artifact but rather a fundamental asymptotic phenomenon that manifests clearly even in moderately-sized datasets. 

\begin{figure}[htbp!]
    \centering
    \includegraphics[width=\linewidth]{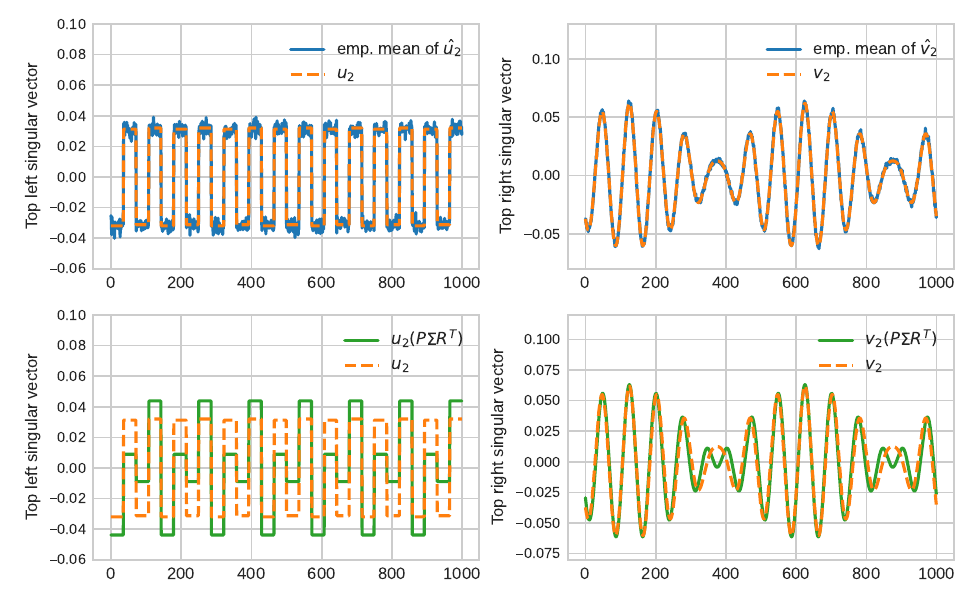}
    \caption{Comparison of the second top singular vectors of $\mSXY$ and of $\mP\mSigma\mR^\trans$. To make the interpretation easier, the singular vectors of $\mP\mSigma\mR^\trans$ are not generated randomly. The left singular vectors are displayed in the left column, while the right singular vectors are displayed in the right column. \textbf{(Top)} Comparison between the empirical mean of the second top singular vectors of $\mSXY$ ($\hat{\vu}_2$ and $\hat{\vv}_2$) and their predictions from Prop.~\ref{prop:alignment_PR} ($\vu_2$ and $\vv_2$).
    \textbf{(Bottom)} Comparison between the second top singular vectors of $\mP\mSigma\mR^\trans$ and the predicted second top singular vectors of $\mSXY$ ($\vu_2$ and $\vv_2$). 
    \textbf{Experimental settings:} $\beta_p = 0.5$, $\beta_q = 0.5$ ($n = 500$), $r=2$ with $\lambda_{P,1} = \lambda_{P,2} = 10$, $\lambda_{R,1} = 50$, $\lambda_{R,2} = 5$, $\rho_1=0.2$ and $\rho_2=1$ ($\lambda_{T,1} = 113.7$, $\lambda_{T,2} = 42.5$, and $\lambda_{T,3} = 6.41$), and $100$  Monte-Carlo runs used to compute the empirical means of $\hat{\vu}_2$ and $\hat{\vv}_2$.}
    \label{fig:skewed_vectors} 
\end{figure}

\section{PLS versus PCA}
\label{sec:pca_comparison}

We compare the ability of PLS and separate principal component analysis (PCA), applied independently to $\mX$ and $\mY$, to detect joint latent directions in the high-dimensional regime.  
To simplify the interpretation of PCA, we restrict attention to the case where only joint components are present in the integrative model~\eqref{eq:model}, \ie where $\mM$ and $\mN$ are null matrices.
While non-zero individual matrices would introduce additional spikes into the model, the orthogonality property established in Lemma~\ref{lem:ortho} ensures that these individual spikes do not interfere with the subsequent analysis of joint spikes. Under this setting, the joint latent components associated with $\mP$ and $\mR$ may be recovered by performing PCA separately on each data matrix. 
We also assume that the joint components are perfectly aligned, \ie $\mSigma=\mI_r$. In this case, there are only $r$ joint spikes, and the results of Cor.~\ref{cor:values_perfect_align} and Cor.~\ref{cor:vectors_perfect_align} apply.
We first recall classical results describing the behavior of PCA in our setting.

\begin{proposition}[PCA performance in the high-dimensional regime \cite{benaych2012singular}]
\label{prop:PCA}
Under the high-dimensional regime of Assumption~\ref{ass:asympt}, PCA applied separately to $\mX$ and $\mY$ detects all $r$ spikes if and only if the smallest spikes satisfy
\[
\lambda_{P,r} > \sqrt{\beta_p}.
\qquad \text{and} \qquad 
\lambda_{R,r} > \sqrt{\beta_q},
\]
where the $\lambda_{P,k}$ and  $\lambda_{R,k}$ for $k\in[r]$ are the eigenvalues of the size $r$ kernel matrices $\mK_P$ and $\mK_R$ defined in Equation~\eqref{eq:kernel_PR}.
In this case, the PCA loading vectors, given by the $r$ leading eigenvectors of
$\mSXX=\frac1p\mX^\trans\mX$ and $\mSYY=\frac1q\mY^\trans\mY$, denoted $\vu_k(\mSXX)$ and $\vu_k(\mSYY)$ respectively, align with the dominant left singular vectors of $\mP$ and $\mR$, denoted $\vu_k(\mP)$ and $\vu_k(\mR)$ respectively: 
\begin{align*}
\langle \vu_k(\mSXX), \vu_k(\mP) \rangle^2  
&\xrightarrow[n,p,q \to \infty]{\text{a.s.}} 
\zeta^{\mathrm{PCA}}_{P,k}
= 1 - \frac{\lambda_{P,k} + \beta_p}{\lambda_{P,k}(\lambda_{P,k}+1)},\\
\langle \vu_k(\mSYY), \vu_k(\mR) \rangle^2  
&\xrightarrow[n,p,q \to \infty]{\text{a.s.}} 
\zeta^{\mathrm{PCA}}_{R,k}
= 1 - \frac{\lambda_{R,k} + \beta_q}{\lambda_{R,k}(\lambda_{R,k}+1)}.
\end{align*}
\end{proposition}

To compare these results with the statistical performance of PLS, we now consider the eigenvalues $\lambda_{T,k}$ of the $r\times r$ matrix $\mK_T$ introduced in Corollary~\ref{cor:values_perfect_align}.  
The matrices $\mK_P$ and $\mK_R$ involved in the construction of $\mK_T$ are symmetric positive definite and full rank.  
Using standard spectral inequalities, we obtain the following lower bound.

\begin{lemma}[Lower bound on the eigenvalues of $\mK_T$]
\label{lem:lcbound}
For all $k\in[r]$,
\[
\lambda_{T,k} 
\ge 
\lambda_{P,r} 
+ \lambda_{R,r}
+ \lambda_{P,r}\,\lambda_{R,r}.
\]
\end{lemma}

\begin{proof}
By similarity, $\mK_T$ shares its eigenvalues with the symmetric matrix
\[
(\mI_r+\mK_P)^{-\frac12} \mK_T (\mI_r+\mK_P)^{\frac12}= 
\mK_P+(\mI_r+\mK_P)^{\frac12} \mK_R (\mI_r+\mK_P)^{\frac12}.
\]
By Weyl's inequality for symmetric matrices,
\[
\lambda_k(\mA+\mB) \ge \lambda_r(\mA) + \lambda_r(\mB),
\qquad \text{for all } k\in[r].
\]
Applying this inequality with $\mA=\mK_P$ and $\mB=(\mI_r+\mK_P)^{\frac12} \mK_R (\mI_r+\mK_P)^{\frac12}$ yields
\[
\lambda_{T,k}
\ge
\lambda_{P,r}
+
\lambda_r\!\left(
(\mI_r+\mK_P)^{\frac12} \mK_R (\mI_r+\mK_P)^{\frac12}
\right).
\]
By similarity, the eigenvalues of the latter matrix coincide with those of
\[
\mK_R^{\frac12} (\mI_r+\mK_P) \mK_R^{\frac12} = \mK_R + \mK_R^{\frac12} \mK_P \mK_R^{\frac12}
\]
Applying Weyl's inequality a second time on this sum gives:
\[
\lambda_{T,k}
\ge
\lambda_{P,r}
+
\lambda_{R,r}
+
\lambda_r(\mK_R^{\frac12}\mK_P\mK_R^{\frac12}),
\]
and by similarity $\lambda_r\!\left(\mK_R^{\frac12} \mK_P\mK_R^{\frac12} \right) = 
\lambda_r\!\left(\mK_P\mK_R \right).$
Since $\mK_P$ and $\mK_R$ are positive definite,
\[
\lambda_r(\mK_P\mK_R)
\ge
\lambda_{P,r}\,\lambda_{R,r},
\]
which concludes the proof.
\end{proof}

We now show that, under the same asymptotic regime, PLS exhibits strictly greater statistical power than separate PCA for detecting joint spikes, provided the joint components are sufficiently aligned.

\begin{proposition}[PLS dominates separate PCA for joint spike detection]
\label{prop:PLSvsPCA}
Assume that, under the high-dimensional regime of Assumption~\ref{ass:asympt}, PCA applied separately to $\mX$ and $\mY$ detects all $r$ spikes, \ie
\[
\lambda_{R,r} > \sqrt{\beta_q},
\qquad \text{and} \qquad 
\lambda_{P,r} > \sqrt{\beta_p}.
\]
Assume also that the joint components are perfectly aligned: $\mSigma_r=\mI_r$.
Then, for all $k\in[r]$,
\[
\lambda_{T,k} > \sqrt{\beta_p} + \sqrt{\beta_q} + \sqrt{\beta_p\beta_q},
\quad \text{and in particular} \quad
\lambda_{T,k} > \tau,
\]
where $\tau$ denotes the PLS spectral phase-transition threshold defined in~\eqref{eq:thresh}, above which a joint spike emerges.
Consequently, all joint spikes detectable by separate PCA are also detectable by PLS, and this implication is generally strict.
\end{proposition}

\begin{proof}
Lemma~\ref{lem:lcbound} implies
\[
\lambda_{T,k}
>
m
=
\sqrt{\beta_p} + \sqrt{\beta_q} + \sqrt{\beta_p\beta_q}.
\]
Evaluating the threshold polynomial~\eqref{eq:thresh_pol} at $\lambda=m$ reduces to $2 \beta_{p}^{3/2} \sqrt{\beta_{q}} + 2 \sqrt{\beta_{p}} \beta_{q}^{3/2}$\\ $+2 \beta_{p}^{3/2} \beta_{q} + 2 \beta_{p} \beta_{q}^{3/2} + 2 \beta_{p} \sqrt{\beta_{q}} + 2 \sqrt{\beta_{p}} \beta_{q} + 4 \beta_{p} \beta_{q}$, which is strictly positive.  
Moreover, the derivative of the polynomial is increasing and, for all $\lambda \ge m$, is bounded from below by its value at $\lambda = m$, namely
$6 \sqrt{\beta_{p}} \sqrt{\beta_{q}} + 6 \sqrt{\beta_{p}} \beta_{q} + 6 \beta_{p} \sqrt{\beta_{q}} + 2 \beta_{p} \beta_{q} + 2 \beta_{p} + 2 \beta_{q}$, which is strictly positive.  As a consequence, the polynomial admits no sign change above $m$, implying that the detection threshold $\tau$ satisfies $\tau < m$.  
Therefore, $\lambda_{T,k}>\tau$ for all $k\in[r]$, proving that any spike detectable by separate PCA is also detected by PLS, with strictly larger spectral separation.
This establishes the claimed superiority of PLS for spike detection in the considered regime.
\end{proof}

Figure~\ref{fig:align_2D} illustrates the phase transition threshold and the joint alignment with $\mP$ and $\mR$ components for both PLS-SVD and PCA. It highlights that PLS exhibits a lower overall threshold, as shown in Prop.~\ref{prop:PLSvsPCA}. Moreover, while the PCA threshold forms rectangular boundaries that prevent joint recovery when one component is too weak, the PLS threshold enables signal recovery in unbalanced settings, yielding stronger overlaps, particularly in these cases. 
Note, however, that extending this comparison of overlaps to the general case of arbitrary rank $r>1$ is not straightforward. While we provide analytical formulas for the alignments of both PCA and PLS, a direct quantitative comparison is complex because the methods align with different target vectors. Specifically, PCA aligns with the singular vectors of $\mP$ and $\mR$, whereas PLS aligns with skewed versions of the singular vectors of $\mP\mR^\trans$ as given in Cor.~\ref{cor:vectors_perfect_align}, which differ in the general case where $r>1$.

\begin{figure}[htbp!]
    \centering
    \includegraphics[width=1\linewidth]{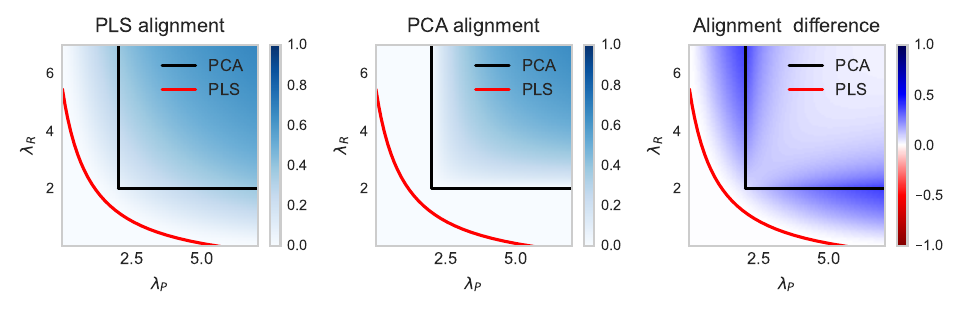}
    \caption{Colormaps of vector alignment as a function of $\lambda_P$ (horizontal axis) and $\lambda_R$ (vertical axis) in the $r=1$ setting. The phase-transition thresholds of PCA (black) and PLS (red) are also shown.
\textbf{(Left)} Colormap of the product of PLS alignments $\zeta_P \times \zeta_R$.
\textbf{(Middle)} Colormap of the product of PCA alignments $\zeta^{\mathrm{PCA}}_{P} \times \zeta^{\mathrm{PCA}}_{R}$.
\textbf{(Right)} Colormap of the difference between these products, $\zeta_P \times \zeta_R - \zeta^{\mathrm{PCA}}_{P} \times \zeta^{\mathrm{PCA}}_{R}$.
\textbf{Experimental settings:} $\beta_p = \beta_q = 4$, $r_M=r_N=0$ (no individual component).}
    \label{fig:align_2D}
\end{figure}

\section{Discussion and Concluding Remarks}
\label{sec:conclusion}

This work provides a comprehensive theoretical ana\-lysis of Partial Least Squares (PLS) in high-dimensional settings using random matrix theory. We established deterministic equivalents for the key resolvent matrices and characterized the limiting spectral distribution of the cross-covariance singular values. Our main contributions include the precise characterization of phase transitions determining when signal components yield isolated singular values, and the quantification of eigenvector alignment with the underlying signal structure. Importantly, we explicitly demonstrate that PLS exhibits greater statistical power than individual PCA applied to each data matrix separately for detecting the joint latent directions in the high-dimensional regime, confirming its theoretical advantage for integrative analysis.

A striking finding of our analysis is the identification of fundamental limitations in the high-dimensional regime. Even in the absence of individual-specific components ($\mM$ and $\mN$), noise induces a skewing effect that prevents PLS singular vectors from aligning with the true signal directions, except in some special cases or in the limit of infinite signal strength. More problematically, when individual components are present, PLS can spuriously align with these uninformative structures rather than with the joint signal, potentially degrading performance in prediction and interpretation tasks.

These findings have immediate practical implications for real-world data analysis, notably in multi-omics integrative studies such as Gene Set Enrichment Analysis (GSEA). In GSEA, PLS-SVD can serve as a powerful tool to estimate the cross-covariance matrix and identify shared latent structures, such as regulatory interactions and biological pathways between paired microRNA and mRNA datasets. However, when strong dataset-specific (individual) components exist, they severely bias this estimation. This bias can cause strong individual signals to be misidentified as as joint latent directions, potentially resulting in spurious biological associations and misleading interpretations.

Driven by these theoretical insights and practical challenges, several promising research directions emerge. First, our precise asymptotic characterization of spurious spikes strongly motivates the development of filtering procedures to eliminate these individual artifacts while preserving the true joint latent structure, in the spirit of \cite{trygg2003o2pls}. Such methods would reduce model complexity and mitigate spurious associations by focusing exclusively on the shared components PLS aims to extract.
Second, while we focused on the PLS-SVD formulation, our framework naturally extends to other common deflation schemes described in \cite{PLS_variants}. Even though the deflated components depend on the data matrices in a non-trivial way, all PLS variants share the same first step, and subsequent iterations involve only low-rank updates. As a result, similar deterministic equivalent techniques should allow us to characterize their high-dimensional behavior, providing a comprehensive theoretical understanding of the broader PLS family of algorithms, including regression tasks.

\medskip

\appendix

\section{Proof of Theorem~\ref{thm:eq_det}}
\label{app:proof_ED}

Theorem~\ref{thm:eq_det} gives jointly the deterministic equivalents of $\mQ(z)=(\frac{1}{pq}\mY^\trans\mX\mX^\trans\mY-z\mI_q)^{-1}$ and $\tilde{\mQ}(z)=(\frac{1}{pq}\mX^\trans\mY\mY^\trans\mX-z\mI_p)^{-1}$. The two quantities are symmetric, and we can deduce one from the other by making the following inversions: $\mX \leftrightarrow \mY$, $\mE \leftrightarrow \mF$, $\mP \leftrightarrow \mR$, $\mM \leftrightarrow \mN$, $p \leftrightarrow q$, $m(z) \leftrightarrow \tilde{m}(z)$, $m_X(z) \leftrightarrow m_Y(z)$. Therefore, we only display here the proof for the deterministic equivalent $\bar{\mQ}$ of $\mQ$.

The expectation $\Esp{\mQ}$ is a good candidate to be a deterministic equivalent of $\mQ$, but we need to prove it. Subsection \ref{app:convergence_ED} is dedicated to this proof, and Subsection \ref{app:calcul_ED} is dedicated to the computation of the mean $\Esp{\mQ}$.

\subsection{Preliminary tools}

The following lemmas will be extensively used to derive a deterministic equivalent.

The first lemma, derived from Markov’s inequality and the first Borel–Cantelli lemma, enables us to establish almost sure convergence.
\begin{lemma}
 \label{lem:almost_sure_0}
 For any sequence $X_n$ of random variables, if there exists $\kappa\in\mathbb{N}^\star$ and $\ell>1$ such that $\Esp{|X_n|^\kappa} = \mathcal{O}(n^{-\ell})$, then $X_n \overset{a.s.}{\longrightarrow} 0$.
\end{lemma}

The next two lemmas give key results for Gaussian calculations.
\begin{lemma}[From \cite{stein_estimation_1981}] \label{lem:stein}
    Let $Z \sim \mathcal{N}(0, 1)$ and $f : \R \to \R$ be a continuously
    differentiable function. When the following expectations exist, $\esp{Z f(Z)} = \esp{f'(Z)}$. In particular for $\vz \sim \mathcal{N}(\m0,\mC)$,
    and $f : \R^p \to \R$ a differentiable function with polynomially bounded partial derivatives,
    \[
        \Esp{z_i f(\vz)} = \sum_{j=1}^p \mC_{ij} \Esp{ \frac{\partial f(\vz)}{\partial z_j}}
    \]
\end{lemma}

\begin{lemma}[Nash--Poincaré inequality, see Proposition 2.1.6 of \cite{pastur_eigenvalue_2011}] \label{lem:nash-poincare}
    Let $\vz \sim \mathcal{N}(\m0, \mI_p)$ and $f : \R^p \to \R$ be a differentiable function with polynomially bounded partial derivatives $\partial_1 f, \ldots, \partial_p f$. Then,
    \[
        \Var(f(\vz)) \leqslant \Esp{\Norm{\nabla f(\vz)}^2}
    \]
    where $\nabla = \begin{bmatrix} \partial_1 & \ldots & \partial_p \end{bmatrix}^\trans$.
\end{lemma}

\subsection{Proof that the expected resolvent matrix is a deterministic equivalent}
\label{app:convergence_ED}

According to Lemma \ref{lem:almost_sure_0}, by bounding the quantities $\Esp{\Abs{\va^\trans(\mQ-\Esp{\mQ})\vb}^\kappa}$ and $\Esp{\Abs{\frac{1}{q}\Tr\left( \mA(\mQ-\Esp{\mQ})\right)}}^\kappa$, for a given $\kappa\in\mathbb{N}^\star$, we can prove the almost sure convergence of $\va^\trans(\mQ-\Esp{\mQ})\vb$ and $\frac{1}{q}\Tr\left( \mA(\mQ-\Esp{\mQ})\right)$, which ultimately proves that $\bar{\mQ}=\Esp{\mQ}$ is a valid deterministic equivalent for $\mQ$.

\subsubsection{Convergence of $\va^\trans (\mQ-\Esp{\mQ})\vb$}

To apply Lemma~\ref{lem:almost_sure_0}, we will find an integer $\kappa$ such that $\Esp{\Abs{\va^\trans (\mQ-\Esp{\mQ})\vb}^\kappa}=\mathcal{O}\left(\frac{1}{n^2}\right)$. We start with $\kappa=2$.
By applying the law of total variance:
\begin{equation}
    \Esp{\Abs{\va^\trans (\mQ-\Esp{\mQ})\vb}^2} = \Var{\left(\va^\trans \mQ\vb\right)} = \Esp{\Var{\left(\va^\trans \mQ\vb|\mY\right)}} + \Var{\left(\Esp{\va^\trans \mQ\vb|\mY}\right)}.
    \label{eq:total_var_0}
\end{equation}

We now successively bound each of the last two terms. By applying the Nash--Poincaré inequality (Lemma \ref{lem:nash-poincare}), we obtain that
\begin{equation*}
   \Var{\left(\va^\trans \mQ\vb|\mY\right)} \leq \sum\limits_{i=1}^n\sum\limits_{j=1}^p \Esp{\Abs{\frac{\partial \va^\trans \mQ\vb}{\partial \mE_{i,j}}}^2|\mY}
\end{equation*}

A useful result to compute the resolvent derivatives,  which comes from the equality $\deriv{\mA^{-1}}{\vx}=-\mA^{-1}\deriv{\mA}{\vx}\mA^{-1}$, 
is given by the following lemma.

\begin{lemma}
Let $\mX \in \R^{n\times p}, \mY \in \R^{n\times q}$ and $\mQ=(\frac{1}{pq}\mY^\trans\mX\mX^\trans\mY-z\mI_{q})^{-1}$ the associated resolvent matrix. For all, $a \in [n]$, $b \in [p]$ :
\begin{equation*}
\deriv{\mQ}{\mE_{ab}} = \deriv{\mQ}{\mX_{ab}} = -\frac{1}{pq} \mQ(\mY^\trans\ve_{a}\ve_{b}^\trans \mX^\trans\mY+\mY^\trans\mX \ve_{b}\ve_{a}^\trans \mY)\mQ
\end{equation*}
where $\ve_{k}$ is a vector of appropriate dimension with a single nonzero entry, equal to $1$, located in position $k$.
\label{lem:partial_derivative_X}
\end{lemma}

It follows that
\begin{equation}
    \frac{\partial \va^\trans \mQ\vb}{\partial \mE_{i,j}} = \sum\limits_{k=1}^q\sum\limits_{l=1}^q \va_k \frac{\partial \mQ_{k,l}}{\partial \mE_{i,j}}b_l = -\frac{1}{pq} \left[\mY\mQ\left(\va\vb^\trans+\vb\va^\trans \right)\mQ\mY^\trans\mX\right]_{i,j}.
\label{eq:partial_deriv_develop}
\end{equation}

Using the fact that $\abs{a+b}^2\leq2(\abs{a}^2+\abs{b}^2)$ and the definition of the Frobenius norm:
\begin{align}
\begin{split}
    &\sum\limits_{i=1}^n\sum\limits_{j=1}^p \Esp{\Abs{\frac{\partial \va^\trans \mQ\vb}{\partial \mE_{i,j}}}^2|\mY} \\
    &\leq \frac{2}{p^2q^2}  \left(\Esp{\Norm{\mY\mQ\va\vb^\trans\mQ\mY^\trans\mX}_F^2|\mY} +\Esp{\Norm{\mY\mQ\vb\va^\trans \mQ\mY^\trans\mX}_F^2|\mY} \right) \\
    &\leq \frac{4}{p^2q^2}\Norm{\mY}^4\Esp{\Norm{\mQ}^4\Norm{\mX}^2|\mY},
\label{eq:frobenius}
\end{split}
\end{align}
%
the last inequality being obtained using the properties $\Norm{\mA\mB}_F\leq\Norm{\mA}_F\Norm{\mB}$, $\Norm{\mA\mB}\leq\Norm{\mA}\Norm{\mB}$ and $\Norm{\va\vb^\trans }_F=\Norm{\vb\va^\trans }_F=\norm{\va}\norm{\vb}=1$. According to the Mar\v{c}enko--Pastur theorem and \cite{no_eig},  $\Norm{\mX}\overset{a.s.}{=}\mathcal{O}(\sqrt{n})$ and $\Norm{\mY}\overset{a.s.}{=}\mathcal{O}(\sqrt{n})$. As $\Norm{\mQ} \leq \abs{\Im(z)}^{-1}$ for any $z\in\C\backslash\R$, we have $\Norm{\mQ}=\mathcal{O}(1)$. Thus, we can prove that:
\begin{equation}
    \Esp{\Var{\left(\va^\trans \mQ\vb|\mY\right)}} \overset{a.s.}{=} \mathcal{O}\left(\frac{1}{n}\right)
    \label{eq:esp_var_0}
\end{equation}

Let us tackle the second term of \eqref{eq:total_var_0}. By applying the Nash--Poincaré inequality (Lemma \ref{lem:nash-poincare}), we prove that
\begin{equation*}
   \Var{\left(\Esp{\va^\trans \mQ\vb|\mY}\right)} \leq \sum\limits_{i=1}^n\sum\limits_{j=1}^q \Esp{\Abs{\frac{\partial \Esp{\va^\trans \mQ\vb|\mY}}{\partial \mF_{i,j}}}^2}
\end{equation*}

We state another useful Lemma, similar to Lemma \ref{lem:partial_derivative_X}, to deal with the derivative with respect to $\mY$ this time.

\begin{lemma} Let $\mX \in \R^{n\times p}, \mY \in \R^{n\times q}$ and $\mQ=(\frac{1}{pq}\mY^\trans\mX\mX^\trans\mY-z\mI_{q})^{-1}$ the associated resolvent matrix. For all $a \in  [\![1,n]\!]$, $b \in  [\![1,q]\!]$ :
\begin{equation*}
\deriv{\mQ}{\mF_{ab}} =
\frac{\partial \mQ}{\partial \mY_{ab}} = -\frac{1}{pq} \mQ \left( \ve_b\ve_a^\trans \mX\mX^\trans \mY + \mY^\trans \mX\mX^\trans \ve_a\ve_b^\trans \right) \mQ
\end{equation*}
\label{lem:partial_derivative_Y}
\end{lemma}

 By applying the Nash–Poincaré inequality (Lemma~\ref{lem:nash-poincare}), interchanging the expectation and derivative operators (justified by the boundedness of the moments) and using the conditional Jensen's inequality
 (which allows us to bound the square of the conditional expectation by the conditional expectation of the square), we have:
 \begin{align*}\Var(\Esp{\va^\trans\mQ \vb | \mY}) &\leq \sum_{i=1}^n \sum_{j=1}^q \Esp{\left( \frac{\partial \Esp{\va^\trans\mQ \vb | \mY}}{\partial \mY_{ij}} \right)^2} \\ 
 &\leq \sum_{i=1}^n \sum_{j=1}^q \Esp{ \Esp{ \left. \left( \frac{\partial (\va^\trans\mQ \vb)}{\partial \mY_{ij}} \right)^2 \right| \mY } } 
 = \sum_{i=1}^n \sum_{j=1}^q \Esp{ \left( \frac{\partial \va^\trans\mQ \vb}{\partial \mY_{ij}} \right)^2}
 \end{align*}
 Using the exact derivative from Lemma \ref{lem:partial_derivative_Y} and applying the same steps as in \eqref{eq:partial_deriv_develop} and \eqref{eq:frobenius}, 
 we obtain:
 \begin{align*}
 \sum_{i=1}^n \sum_{j=1}^q \Esp{ \left( \frac{\partial (\va^\trans\mQ \vb)}{\partial \mY_{ij}} \right)^2 } &\leq \frac{2}{p^2q^2} \Esp{ \Norm{\mX\mX^\trans \mY \mQ \va \vb^\trans \mQ}_F^2 + \Norm{\mX\mX^\trans \mY \mQ \vb \va^\trans \mQ}_F^2 } \\
 &\leq \frac{4}{p^2q^2} \Esp{ \Norm{\mX}^4 \Norm{\mY}^2 \Norm{\mQ}^4 }
 \end{align*}
 Using the same properties as before, we prove that:
\begin{equation}
   \Var{\left( \Esp{\va^\trans \mQ\vb|\mY}\right)} = \mathcal{O}\left( \frac{n^2 \cdot n \cdot 1}{n^4} \right)= \mathcal{O}\left(\frac{1}{n}\right)
   \label{eq:var_esp_0}
\end{equation}
To summarize, combining \eqref{eq:total_var_0}, \eqref{eq:esp_var_0} and \eqref{eq:var_esp_0}, so far we have:
\begin{equation*}
    \Esp{\Abs{\va^\trans (\mQ-\Esp{\mQ})\vb}^2} = \mathcal{O}\left(\frac{1}{n}\right)
\end{equation*}

Unfortunately, this is not enough to apply Lemma \ref{lem:almost_sure_0}, so we will compute a moment of higher order $\kappa=4$:
\begin{equation}
    \Esp{\Abs{\va^\trans (\mQ-\Esp{\mQ})\vb}^4} = \Var{\left(\Abs{\va^\trans (\mQ-\Esp{\mQ})\vb}^2\right)} + \Esp{\Abs{\va^\trans (\mQ-\Esp{\mQ})\vb}^2}^2
    \label{eq:moment4}
\end{equation}
From the previous calculus, the rightmost term is of order $\mathcal{O}\left(\frac{1}{n^2}\right)$, so we just need to bound the variance $\Var{\left(\Abs{\va^\trans (\mQ-\Esp{\mQ})\vb}^2\right)}$. Once again, to apply the Nash--Poincaré inequality, we need first to use the law of total variance to decompose the conditioning on $\mX$ and $\mY$ of the random variable $T=\va^\trans (\mQ-\Esp{\mQ})\vb$:
\begin{equation}
    \Var{\left(\Abs{\va^\trans (\mQ-\Esp{\mQ})\vb}^2\right)} = \Var{\left(\Abs{T}^2\right)} = \Esp{\Var{\left(\Abs{T}^2|\mY\right)}} + \Var{\left(\Esp{\Abs{T}^2|\mY}\right)}.
    \label{eq:total_var_1}
\end{equation}

Then, according to Lemma \ref{lem:nash-poincare}, and applying successively the same steps as in \eqref{eq:partial_deriv_develop} and \eqref{eq:frobenius},
\begin{align*}
    &\Var{\left(\Abs{\va^\trans (\mQ-\Esp{\mQ})\vb}^2|\mY\right)} \leq \sum\limits_{i=1}^n\sum\limits_{j=1}^p \Esp{\Abs{2\va^\trans (\mQ-\Esp{\mQ})\vb\frac{\partial \va^\trans \mQ\vb}{\partial \mE_{i,j}}}^2|\mY} \\
    & \leq \frac{8}{p^2q^2} \Esp{\Abs{\va^\trans (\mQ-\Esp{\mQ})\vb}^2\left(\Norm{\mY\mQ\va\vb^\trans \mQ\mY^\trans\mX}_F^2+\Norm{\mY\mQ\vb\va^\trans \mQ\mY^\trans\mX}_F^2\right)|\mY} \\
    & \leq \frac{16}{p^2q^2}\Norm{\mY}^4 \Esp{\Abs{\va^\trans (\mQ-\Esp{\mQ})\vb}^2\Norm{\mQ}^4\Norm{\mX}^2|\mY}.
\end{align*}

Then, by Cauchy-Schwarz inequality,
\begin{align*}
    &\Var{\left(\Abs{\va^\trans (\mQ-\Esp{\mQ})\vb}^2|\mY\right)}^2 \\
    &\leq \left(\frac{16}{p^2q^2}\right)^2\Norm{\mY}^8 \Esp{\Abs{\va^\trans (\mQ-\Esp{\mQ})\vb}^4|\mY}\Esp{\Norm{\mQ}^8\Norm{\mX}^4|\mY} \\
    &\leq \underbrace{\left(\frac{16}{p^2q^2}\right)^2 \Norm{\mY}^8\Esp{\Norm{\mQ}^8\Norm{\mX}^4|\mY}}_{\overset{a.s.}{=}\mathcal{O}\left(\frac{1}{n^2}\right)} \left(\Var{\left(\Abs{\va^\trans (\mQ-\Esp{\mQ})\vb}^2|\mY\right)} + \mathcal{O}\left(\frac{1}{n^2}\right)\right),
\end{align*}
which proves that:
\begin{equation}
    \Esp{\Var{\left(\Abs{T}^2|\mY\right)}} = \Esp{\Var{\left(\Abs{\va^\trans (\mQ-\Esp{\mQ})\vb}^2|\mY\right)}} = \mathcal{O}\left(\frac{1}{n^2}\right).
    \label{eq:var_esp_1}
\end{equation}

We now need to bound the term $\Var{\left(\Esp{\Abs{\va^\trans (\mQ-\Esp{\mQ})\vb}^2|\mY}\right)}$. Let $Z = \va^\trans \mQ\vb$. By expanding the conditional expectation, we have $\Esp{\Abs{Z - \Esp{Z}}^2|\mY} = \Var{(Z|\mY)} + \left(\Esp{Z|\mY} - \Esp{Z}\right)^2$.
Using the property $\Var{(A+B)} \leq 2\Var{(A)} + 2\Var{(B)}$, we get:
\begin{equation*}
    \Var{\left(\Esp{\Abs{Z-\Esp{Z}}^2|\mY}\right)} \leq 2\Var{\left(\Var{(Z|\mY)}\right)} + 2\Var{\left(\left(\Esp{Z|\mY} - \Esp{Z}\right)^2\right)}.
\end{equation*}
For the first term, using our previous work on $\Var{\left(\va^\trans \mQ\vb|\mY\right)}$, we have:
\begin{equation*}
    \Var{\left(\Var{\left(\va^\trans \mQ\vb|\mY\right)}\right)} \leq \left(\frac{4}{p^2q^2}\right)^2\Var{\left(\Esp{\Norm{\mY}^4\Norm{\mQ}^4\Norm{\mX}^2|\mY}\right)}
\end{equation*}
By bounding the variance by the second moment ($\Var(W) \leq \Esp{W^2}$) and applying the conditional Jensen's inequality, we obtain:
\begin{equation*}
    \Var{\left(\Esp{\Norm{\mY}^4\Norm{\mQ}^4\Norm{\mX}^2|\mY}\right)} \leq \Esp{\Norm{\mY}^8\Norm{\mQ}^8\Norm{\mX}^4} = \mathcal{O}\left(n^6\right),
\end{equation*}
which shows that $\Var{\left(\Var{(Z|\mY)}\right)}=\mathcal{O}(1/n^2)$.

For the second term, $\Var{\left(\left(\Esp{Z|\mY} - \Esp{Z}\right)^2\right)} \leq \Esp{\left(\Esp{Z|\mY} - \Esp{Z}\right)^4}$. By applying the Nash-Poincaré inequality on the $4^\text{th}$ central moment of this bounded function of $\mY$, this term is also bounded by $\mathcal{O}(1/n^2)$.

Finally:
\begin{equation}
    \Var{\left(\Esp{\Abs{T}^2|\mY}\right)} = \mathcal{O}\left(\frac{1}{n^2}\right).
    \label{eq:esp_var_1}
\end{equation}


Ultimately, combining equations \eqref{eq:moment4}, \eqref{eq:total_var_1}, \eqref{eq:var_esp_1} and \eqref{eq:esp_var_1} we have:
\begin{equation*}
    \Esp{\Abs{\va^\trans (\mQ-\Esp{\mQ})\vb}^4} = \mathcal{O}\left(\frac{1}{n^2}\right).
\end{equation*}
We now apply Lemma \ref{lem:almost_sure_0}, which proves the a.s.~convergence to $0$ of $\va^\trans (\mQ-\Esp{\mQ})\vb$.

\subsubsection{Convergence of $\frac{1}{q}\Tr\left( \mA(\mQ-\Esp{\mQ})\right)$}

To apply Lemma \ref{lem:almost_sure_0}, we will prove that $\Esp{\Abs{\frac{1}{q}\Tr\left( \mA(\mQ-\Esp{\mQ})\right)}}^2=\mathcal{O}(\frac{1}{n^2})$. By applying the law of total variance:
\begin{equation}
    \Esp{\Abs{\frac1q \Tr\left( \mA(\mQ-\Esp{\mQ})\right)}^2} = \Esp{\Var{\left(\frac1q \Tr (\mA\mQ)|\mY\right)}} + \Var{\left(\Esp{\frac1q \Tr (\mA\mQ)|\mY}\right)}.
    \label{eq:total_var_2}
\end{equation}

By applying successively Nash--Poincaré inequality (Lemma \ref{lem:nash-poincare}), Lemma \ref{lem:partial_derivative_X} and the same steps as before (see \eqref{eq:partial_deriv_develop} and \eqref{eq:frobenius}), it comes that
\begin{equation*}
    \Var{\left(\frac1q \Tr (\mA\mQ)|\mY\right)} \leq \frac{4}{p^2q^4}  \Norm{\mA}_F^2\Esp{\Norm{\mY}^4\Norm{\mQ}^4\Norm{\mX}^2}.
\end{equation*}

Finally, as $\Norm{\mA}_F^2\leq q \Norm{\mA}^2=\mathcal{O}(n)$, $\Norm{\mQ}\overset{a.s.}{=}\mathcal{O}(1)$, $\Norm{\mX}\overset{a.s.}{=}\mathcal{O}(\sqrt{n})$ and $\Norm{\mY}\overset{a.s.}{=}\mathcal{O}(\sqrt{n})$, we can prove that:
\begin{equation}
    \Esp{\Var{\left(\frac1q \Tr (\mA\mQ)|\mY\right)}} = \mathcal{O}\left(\frac{1}{n^2}\right).
    \label{eq:esp_var_2}
\end{equation}

Let us now tackle the second term $\Var{\left(\Esp{\frac1q \Tr( \mA\mQ)|\mY}\right)}$. By applying Nash--Poincaré inequality (Lemma \ref{lem:nash-poincare}), Lemma \ref{lem:partial_derivative_Y} and the usual steps, it comes that
\begin{align*}
    \sum_{i=1}^n\sum_{j=1}^q \Esp{\left( \frac{\partial}{\partial \mY_{ij}} \frac{1}{q} \Tr(\mA\mQ) \right)^2} 
    &\leq \frac{2}{p^2 q^4} \Esp{\Norm{\mX \mX^\trans \mY \mQ \mA \mQ}_F^2} \\
    &\leq \frac{2}{p^2 q^4} \Esp{\Norm{\mX}^4 \Norm{\mY}^2 \Norm{\mQ}^4 \Norm{\mA}_F^2}.
\end{align*}
%
which finally proves that:
\begin{equation}
    \Var{\left(\Esp{\frac1q \Tr (\mA\mQ)|\mY}\right)} =
    \mathcal{O}\left(\frac{n^2 \cdot n \cdot 1 \cdot n}{n^6}\right)=
    \mathcal{O}\left(\frac{1}{n^2}\right).
    \label{eq:var_esp_2}
\end{equation}
Ultimately, combining \eqref{eq:total_var_2}, \eqref{eq:esp_var_2} and \eqref{eq:var_esp_2}, we can conclude that:
\begin{equation*}
    \Esp{\Abs{\frac1q \Tr (\mA(\mQ-\Esp{\mQ}))}^2} = \mathcal{O}\left(\frac{1}{n^2}\right),
\end{equation*}
which proves, by Lemma \ref{lem:almost_sure_0}, the a.s.~convergence to $0$ of $\frac1q \Tr (\mA(\mQ-\Esp{\mQ}))$.

\subsection{Computation of $\Esp{\mQ \mid \mY}$}
\label{app:calcul_ED}

We recall that $\mQ = \left(\frac{1}{pq}\mY^\trans\mX\mX^\trans\mY-z\mI_q\right)^{-1}$. Since $\Esp{\mQ}=\Esp{\esp{\mQ\mid \mY}}$, we begin by considering $\mY$ to be fixed.
From \eqref{eq:K} and \eqref{eq:model}, we can write
\[
\mK = \frac{1}{pq}\mY^\trans\mX \mX^\trans\mY= \frac1p (\mZ +\mW) (\mZ+\mW)^\trans,
\]
with
\begin{equation*}
 \mZ = \frac1{\sqrt{q}} \mY^\trans\mE, \quad \text{ and } \quad \mW = \frac1{\sqrt{q}} \mY^\trans(\mT_X\mP^\trans+\mM).
\end{equation*}

Since the entries of $\mE$ are standard i.i.d. Gaussian variables with columns denoted as $\vb_i$, for $i\in[p]$, the columns of $\mZ$, denoted as $\vz_i= \tfrac1{\sqrt{q}}\mY^\trans \vb_i$, are i.i.d centered vectors with covariance
$$\mC \equiv \Cov(\vz_i\mid \mY) =\frac{1}{q}\mY^\trans \Esp{\vb_{i}^\trans\vb_i}\mY=\frac{1}{q}\mY^\trans\mY.$$
This leads us to first study the resolvent of a sample covariance matrix for a signal-plus-noise model with correlated noise $\mZ$ and signal matrix $\mW$.
In the subsequent computations of Section \ref{app:calcul_ED}, we assume that $\mY$ is fixed, and therefore the columns of $\mZ$ have covariance $\mC$.

Then, starting from $\mQ^{-1}\mQ=\mI_{q}$, we obtain:
\begin{equation}
\frac1p \left(\mZ\mZ^\trans\mQ + \mW\mW^\trans\mQ + \mW\mZ^\trans\mQ + \mZ\mW^\trans\mQ\right) = \mI_{q}+z\mQ.
\label{eq:QinvQ}
\end{equation}

\subsubsection{Term $\frac1p \mZ\mZ^\trans \mQ$}

From Stein's lemma \ref{lem:stein} for correlated Gaussian vectors,
\begin{equation*}
    \Esp{\mZ\mZ^\trans \mQ}_{ij} = \sum_{a=1}^p \Esp{\mZ_{ia}(\mZ^\trans\mQ)_{aj}} = \sum_{a=1}^p \sum_{b=1}^p \mC_{ib} \Esp{\deriv{\mZ^\trans\mQ}{\mZ_{ba}}}_{aj}.
\end{equation*}

Using the same idea as in Lemma \ref{lem:partial_derivative_X},
\begin{equation*}
    \deriv{\mZ^\trans\mQ}{\mZ_{ba}} = \ve_a\ve_b^\trans\mQ - \frac1p \mZ^\trans\mQ\left(\ve_b\ve_a^\trans\left(\mZ+\mW\right)^\trans + \left(\mZ+\mW\right)\ve_a\ve_b^\trans\right)\mQ,
\end{equation*}
and we thus obtain, after summation:
\begin{align*}
\frac{1}{p} \Esp{\mZ\mZ^\trans\mQ} &= 
\Esp{\mC \mQ} - \frac1p \Esp{\frac{\Tr(\mZ\mZ^\trans\mQ)}{p} \mC \mQ}\\
&- \frac1{p^2} \mC \Esp{\mQ \mZ \mZ^\trans\mQ} - \frac1p  \Esp{ \frac{\Tr(\mZ^\trans\mQ \mW)}{p} \mC \mQ}  
- \frac1{p^2} \mC \Esp{\mQ \mZ \mW^\trans\mQ}.
\end{align*}
The last three terms in the previous equation are $o(1)$ in norm. In particular, the one with the trace is $o(1)$ because $\mW$ has a fixed low rank. As $\mW$ is low-rank, we also have, from \eqref{eq:QinvQ}:
\begin{equation*}
    \frac1p \Tr\left(\mZ\mZ^\trans\mQ\right) \overset{a.s.}{=} \Tr\left(\mI_q+z\mQ\right) + o(1) = q+z\Tr\mQ + o(1).
\end{equation*}
As we have shown in Appendix~\ref{app:convergence_ED}, $\frac1q \Tr\mQ$ converges almost surely to $\frac1q \Tr\bar{\mQ}$. Thus, 
\begin{align}
\frac{1}{p} \Esp{\mZ\mZ^\trans\mQ} &= 
\Esp{\mC \mQ} - \left(\frac{q}{p} z m(z) + \frac{q}{p} \right) \Esp{\mC \mQ} = -z\tilde{m}(z)\Esp{\mC \mQ},
\label{eq:ZZTQ}
\end{align}
where $m(z)$ (resp. $\tilde{m}(z)$) is defined as the Stieltjes transform of the limiting spectral distribution of $\frac{1}{p} \mZ\mZ^\trans$ (resp. $\frac{1}{p} \mZ^\trans\mZ$), satisfying $m(z) = \frac1q \Tr \bar{\mQ}(z)$ (resp. $\tilde{m}(z) = \frac1p \Tr \bar{\tilde{\mQ}}(z)$).
By using standard properties of the resolvent, we can prove that $z\tilde{m}(z) = \frac{q}{p} z m(z) - \left(1 - \frac{q}{p}\right)$, hence the last equality.

\subsubsection{Term $\frac{1}{p} \mZ\mW^\trans\mQ$}
Using similar calculations, we obtain
\begin{align*}
&\frac{1}{p}\Esp{\mZ\mW^\trans\mQ}_{ij}=\frac{1}{p}\sum_{a=1}^p\Esp{\mZ_{ia}(\mW^\trans\mQ)_{aj}}=\frac{1}{p}\sum_{a=1}^p \sum_{b=1}^p \mC_{ib} \Esp{\deriv{\mW^\trans\mQ}{\mZ_{ba}}}_{aj} \\
&=-\frac{1}{p^2}\sum_{a=1}^p \sum_{b=1}^p \mC_{ib} \Esp{(\mW^\trans\mQ)_{ab}(\left(\mZ+\mW\right)^\trans\mQ)_{aj} + (\mW^\trans\mQ\left(\mZ+\mW\right))_{aa}\mQ_{bj}}
\end{align*}
Thus we obtain:
\begin{align*}
\frac{1}{p}\Esp{\mZ\mW^\trans\mQ}&=-\frac{1}{p^2}\Esp{\mC \mQ\mW\left(\mZ+\mW\right)^\trans\mQ + \Tr{(\mQ\left(\mZ+\mW\right)\mW^\trans)}\mC \mQ} = o(1).
\end{align*}

\subsubsection{Term $\frac{1}{p} \mW\mZ^\trans\mQ$}

\begin{align*}
&\frac{1}{p}\Esp{\mW\mZ^\trans\mQ}_{ij} = \frac{1}{p} \sum_{a=1}^q \sum_{k=1}^p \Esp{\mQ_{ja} \mZ_{ak}} \mW_{ik} = \frac{1}{p} \sum_{a=1}^q \sum_{k=1}^p \sum_{b=1}^p \mC_{ab} \Esp{\deriv{\mQ}{\mZ_{bk}}}_{ja} \mW_{ik} \\
&= -\frac{1}{p^2} \sum_{a=1}^q \sum_{k=1}^p \sum_{b=1}^p \mC_{ab} \Esp{\mQ\ve_b\ve_k^\trans\left(\mZ+\mW\right)^\trans\mQ + \mQ\left(\mZ+\mW\right)\ve_k\ve_b^\trans\mQ}_{ja} \mW_{ik} \\
&= -\frac{1}{p^2}\Esp{\mW\left(\mZ+\mW\right)^\trans\mQ\mC \mQ +  \mW\left(\mZ+\mW\right)^\trans\mQ\Tr(\mC \mQ)}_{ij} 
\end{align*}
\begin{align*}
\frac{1}{p}\Esp{\mW\mZ^\trans\mQ} = -\frac{1}{p^2}(\Esp{\mW\left(\mZ+\mW\right)^\trans\mQ\mC \mQ} +\Esp{\mW\left(\mZ+\mW\right)^\trans\mQ}\Tr{(\mC \mQ)}).
\end{align*}
The two first terms of the right hand side are negligeable, thus
\begin{align*}
\frac{1}{p}\left( \Esp{\mW\mZ^\trans\mQ} \left(1 + \Tr{(\mC \mQ)} \right) \right) = -\frac{1}{p^2}\Esp{\mW\mW^\trans\mQ}\Tr{(\mC \mQ)}.
\end{align*}
Taking $\frac1p \Tr{(\cdot)}$ from \eqref{eq:ZZTQ}, it comes that
\begin{equation*}
    \frac1p \Tr{(\mC \mQ)} =\frac{\frac{q}{p} zm(z) +\frac{q}{p} }{1 - \frac{q}{p} zm(z) - \frac{q}{p} } +o(1) = -1 - \frac{1}{z\tilde{m}(z)} +o(1).
\end{equation*}

Finally, 
\begin{align*}
\frac{1}{p}\Esp{\mW\mZ^\trans\mQ}= - \left( 1 + z \tilde{m}(z) \right) \frac1p
\Esp{\mW\mW^\trans\mQ} +o(1).
\end{align*}

\subsubsection{Term $\frac{1}{p} \mW\mW^\trans\mQ$}

Since $\mW = \frac{1}{\sqrt{q}} \mY^\trans (\mT_X\mP^\trans+\mM)$, it comes that \\
$\frac{1}{p} \mW\mW^\trans\mQ= \frac{1}{q} \mY^\trans \mD \mY \mQ$, where $\mD=\frac1p \left(\mT_X\mP^\trans+\mM\right)\left(\mT_X\mP^\trans+\mM\right)^\trans$ is a bounded low rank matrix.

\subsubsection{Synthesis}

Putting all the terms together, we obtain from \eqref{eq:QinvQ} that
\begin{align*}
\mI_q + z \Esp{\mQ} &= -z\tilde{m}(z) \mC\Esp{\mQ} - z\tilde{m}(z) \frac{1}{q} \mY^\trans \mD \mY \Esp{\mQ}\\
&= -z \tilde{m}(z) \left( \frac{1}{q} \mY^\trans \left(\mI_n +\mD \right) \mY \right) \Esp{\mQ}, 
\end{align*}
which yields (assuming now that $\mY$ is a random vector)
\begin{align}
\Esp{\mQ(z)} &= \frac{-1}{ z \tilde{m}(z)} \Esp{\left( \frac{1}{q} \mY^\trans \left( \mI_n + \mD \right) \mY  + \frac{1}{\tilde{m}(z)}  \mI_q \right)^{-1}}.
\label{eq:QY}
\end{align}

\subsection{Computation of the expectation w.r.t $\mY$}
We define the following $q\times q$ resolvent matrix
\begin{align*}
\mQ_\mY(z) &=
\left( \frac{1}{q} \mY^\trans \left( \mI_n + \mD \right) \mY  - z \mI_q \right)^{-1},
\end{align*}
where we recall that $\mY=\mH+\mF \in \R^{n\times q}$, $\mH=\mT_Y\mR^\trans + \mN \in \R^{n\times q}$ and  $\mD=\frac1p \left(\mT_X\mP^\trans+\mM\right)\left(\mT_X\mP^\trans+\mM\right)^\trans \in \R^{n\times n}$ are deterministic low rank matrices,  and $\mF \in \R^{n\times q}$ is the noise matrix with i.i.d. $\mathcal{N}(0,1)$ entries. This is a ``compound''  spike model with both a covariance spike and an information-plus-noise spike.

The objective is to compute its mean:
\begin{align}
\Esp{\mQ_\mY(z)} &= 
\Esp{\left( \frac{1}{q} \mY^\trans \left(\mI_n + \mD \right) \mY  - z \mI_q \right)^{-1}},
\label{eq:Q_Y}
\end{align}
Then we can directly derive from
\eqref{eq:QY} our deterministic equivalent for the original problem as 
$\bar{\mQ}(z) = \frac{-1}{z\tilde{m}(z)}\Esp{\mQ_\mY\left( \frac{-1}{\tilde{m}(z)} \right)}$.
In the following we will denote the Stieltjes transform of the limiting eigenvalue distribution of the matrix $\frac{1}{q} \mY^\trans \left( \mI_n + \mD \right) \mY $ as $m_Y(z)$. In fact we will check later in Section \ref{app:concl} that $m_Y(z)=\frac{1}{\beta_q} m_{\mathrm{MP}}\left(\frac{z}{\beta_q}\right)$, where $m_{\mathrm{MP}}(\tilde{z})$ is the Stieltjes transform of the Mar\v{c}enko--Pastur distribution with parameter $\frac{1}{\beta_q}$. Moreover the $\mY$ index for the resolvent in \eqref{eq:Q_Y} will be dropped when there is no risk of confusion. 

As usual, we start with the following development:
\begin{align}
    \frac{1}{q} \mF^\trans \mY \mQ + \frac{1}{q} \mF^\trans \mD \mY \mQ + \frac{1}{q} \mH^\trans(\mI_n + \mD) \mY\mQ = \mI_q +z\mQ.
    \label{eq:devY}
\end{align}

\subsubsection{First term $\frac{1}{q} \mF^\trans \mY \mQ$}

Applying once again Stein's lemma \ref{lem:stein},
\begin{align*}
&\frac{1}{q}\Esp{\mF^\trans\mY\mQ}_{ij}=\frac{1}{q}\sum_{k=1}^n\Esp{\mF_{ki}(\mY\mQ)_{kj}}=\frac{1}{q}\sum_{k=1}^n\Esp{\deriv{\mY\mQ}{\mF_{ki}}}_{kj}\\
&=\frac{1}{q}\sum_{k=1}^n\left(\Esp{\ve_k\ve_i^\trans\mQ}_{kj} -\frac{1}{q}\Esp{\mY\mQ\ve_i\ve_k^\trans(\mI_n+\mD)\mY\mQ+\mY\mQ\mY^\trans(\mI_n+\mD)\ve_k\ve_i^\trans\mQ}_{kj}\right),
\end{align*}
and thus
\begin{equation*}
\frac{1}{q}\Esp{\mF^\trans\mY\mQ} = \frac{n}{q}\Esp{\mQ}-\frac{1}{q^2}\left(\Esp{\mQ\mY^\trans(\mI_n+\mD)\mY\mQ}+\Tr\left(\mY\mQ\mY^\trans(\mI_n+\mD)\right)\Esp{\mQ}\right).
\end{equation*}

We have $\frac1q \mY^\trans(\mI_n+\mD)\mY\mQ = \mI_q + z \mQ$, thus $\frac{1}{q^2} \mQ\mY^\trans(\mI_n+\mD)\mY\mQ=\frac{1}{q}\mQ(\mI_q + z \mQ)$ and
$\frac{1}{q^2}\Norm{\mQ\mY^\trans(\mI_n+\mD)\mY\mQ} \le \frac{1}{q} \Norm{\mQ} + \frac{|z|}{q} \Norm{\mQ}^2 = \mathcal{O}\left(\frac{1}{q}\right)= o(1)$. 
This shows that the second term on the right hand side of the above equation vanishes.
In addition,  
$\frac{1}{q^2}\Tr{\left(\mY\mQ\mY^\trans(\mI_n+\mD) \right)}= \frac{1}{q^2} \Tr{\left(\mY^\trans(\mI_n+\mD)\mY\mQ \right)}= \frac{1}{q}\Tr{\left(\mI_q + z\mQ \right)}=1+z m_Y(z)$. Then
\begin{equation*}
\frac{1}{q}\Esp{\mF^\trans\mY\mQ} =\Esp{\mQ}\left(\frac{n}{q}-(1+zm_Y(z))\right)+o(1)
\end{equation*}.

\subsubsection{Second term $\frac{1}{q} \mF^\trans \mD\mY \mQ$}

With similar calculations, we have
\begin{align*}
&\frac{1}{q}\Esp{\mF^\trans\mD\mY\mQ}_{ij}=\frac{1}{q}\sum_{k=1}^n\Esp{\mF_{ki}(\mD\mY\mQ)_{kj}}=\frac{1}{q}\sum_{k=1}^n\Esp{\deriv{\mD\mY\mQ}{\mF_{ki}}}_{kj}\\
&=\frac{1}{q}\Esp{\mD^\trans\mQ}_{ij}-\frac{1}{q^2}\left(\Esp{\mQ\mY^\trans\mD^\trans(\mI_n+\mD)\mY\mQ}_{ij} + \Tr(\mD\mY\mQ\mY^\trans(\mI_n+\mD))\Esp{\mQ_{ij}}\right)
\end{align*}

As $\lVert \mD^\trans\mQ\rVert=\mathcal{O}(1)$, we see that $\frac{1}{q}\Esp{\mD^\trans\mQ}=o(1)$. Besides, we have the following inequality:
\begin{align*}
\frac{1}{q^2}\Norm{\mQ\mY^\trans\mD^\trans(\mI_n+\mD)\mY\mQ} &\leq \frac{1}{q^2}\lVert \mQ\rVert^2 \Norm{\mD^\trans(\mI_n+\mD)}\Norm{\mY\mY^\trans} = \mathcal{O}\left(\frac{1}{q}\right),
\end{align*}
with $\Norm{\mQ} = \mathcal{O}(1)$, $\Norm{\mD}= \mathcal{O}(1)$ and $\frac{1}{q}\Norm{\mY\mY^\trans}= \mathcal{O}(1)$.\newline

For the other term we have, with Cauchy-Schwarz inequality:

\begin{align*}
&\frac{1}{q^2}\Abs{\Tr\left(\mD\mY\mQ\mY^\trans(\mI_n+\mD)\right)} \leq \frac{1}{q^2}\Norm{(\mI_n+\mD)^\frac12 \mD^\frac12 \mY\mQ}_F\Norm{\mY^\trans\mD^\frac12(\mI_n+\mD)^\frac12}_F \\
&\leq \frac{r_M+r}{q^2}\Norm{\mI_n+\mD} \Norm{\mD} \Norm{\mQ} \Norm{\mY\mY^\trans} = \mathcal{O}\left(\frac{r_M+r}{q}\right) = o(1),
\end{align*}
where $r_M+r$ is the rank of $\mD$, which is finite. Finally, all the terms vanish:
\begin{equation*}
\frac{1}{q}\Esp{\mF^\trans\mD\mY\mQ} = o(1).
\end{equation*}

\subsubsection{Third term $\frac{1}{q}\mH^\trans(\mI_n+\mD)\mY\mQ$}

Applying again similar calculations, we get
\begin{equation*}
\frac{1}{q}\Esp{\mF\mQ}=  -\frac{1}{q^2}\left(\Esp{(\mI_n+\mD)\mY\mQ}\Tr(\mQ) + \Esp{(\mI_n+\mD)^\trans\mY\mQ^2}\right).
\end{equation*}

We know that that $\frac{1}{q^2}\Norm{(\mI_n+\mD)^\trans\mY\mQ^2}=\mathcal{O}\left(\frac{1}{q\sqrt{q}}\right)=o(1)$, thus
\begin{equation*}
    \Esp{\mF\mQ}= -\frac1q \Esp{(\mI_n + \mD) \mY\mQ}\Tr{(\mQ)} = -(\mI_n+\mD)\Esp{\mY\mQ}m_Y(z).
\end{equation*}
Then, from $\mF=\mY-\mH$, it comes that
\begin{equation*}
\Esp{\mY\mQ}=\left(\mI_n+m_Y(z)(\mI_n+\mD)\right)^{-1}\mH\Esp{\mQ}.
\end{equation*}

Returning to the original term, we get:
\begin{align*}
\frac{1}{q}\Esp{\mH^\trans(\mI_n+\mD)\mY\mQ}& = \frac{1}{q}\mH^\trans(\mI_n+\mD)^{\frac{1}{2}}\left(\mI_n+m_Y(z)(\mI_n+\mD)\right)^{-1}(\mI_n+\mD)^{\frac{1}{2}}\mH\Esp{\mQ}.
\end{align*}

\subsubsection{Conclusion}
\label{app:concl}


Putting all three terms together in \eqref{eq:devY}, we have:
\begin{align*}
&z\Esp{\mQ_{\mY}}+\mI_q = \Esp{\mQ_{\mY}}\left(\frac{n}{q}-(1+z m_Y(z))\right) \\
&+ \frac{1}{q}\mH^\trans(\mI_n+\mD)^{\frac{1}{2}}\left(\mI_n+m_Y(z)(\mI_n+\mD)\right)^{-1}(\mI_n+\mD)^{\frac{1}{2}}\mH\Esp{\mQ_{\mY}}
\end{align*}

Taking $\frac1q \Tr{(\cdot)}$ on both sides of the previous equation, and discarding the term on the second line, which vanishes, yields that $m_Y(z)$ satisfies:
\begin{equation}
   z m^2_{Y}(z) - \left(\frac{n}{q} - 1 -z\right) m_Y(z) + 1 = 0,
   \label{eq:YMP_proof}
\end{equation}
which is equivalent to say that $m_Y(z)=\frac{1}{\beta_q} m_{\mathrm{MP}}(\frac{z}{\beta_q})$, where $m_{\mathrm{MP}}(\tilde{z})$ satisfies:
\begin{align*}
   \tilde{z} \frac{1}{\beta_q} m^2_{MP}(\tilde{z}) - (1 -\frac{1}{\beta_q} -\tilde{z}) m_{\mathrm{MP}}(\tilde{z}) + 1 = 0,
\end{align*}
which is the Mar\v{c}enko--Pastur equation with parameter $\frac{1}{\beta_q}$. This means that $m_Y(z)$ can be interpreted as the Stieltjes transform of the limiting spectral distribution of $\frac1q \mY^\trans\mY$. From equation \eqref{eq:YMP_proof}, $\frac{n}{q} - 1 - z m_Y(z)= z+\frac1{m_Y(z)}$, and it comes that
\begin{align}
\Esp{\mQ_{\mY}(z)} & = \left( \frac{1}{m_Y(z)} \mI_q + \frac{1}{q}\mH^\trans(\mI_n+\mD)^{\frac{1}{2}}\left(\mI_n+m_Y(z)(\mI_n+\mD)\right)^{-1}(\mI_n+\mD)^{\frac{1}{2}}\mH\right)^{-1}
\end{align}

\subsubsection{Final simplifications}
\label{app:simp}

We recall that $\mH=\mT_Y\mR^\trans+\mN$ and $\mD=\frac1p (\mT_X\mP^\trans+\mM)(\mT_X\mP^\trans+\mM)^\trans$. Using (two times) Woodbury identity yields now:
\begin{align}
&\Esp{\mQ_{\mY}(z)} = \left( \frac1{m_Y(z)} \mI_q + \frac{1}{1+m_Y(z)} \frac1q \left(\mN^\trans \mN + \mR\mR^\trans \right)
+  \right. \nonumber\\
& \left. \frac{1}{(1+m_Y(z)) m_Y(z)} \frac1{pq} \mR\mSigma\mP^\trans \left( \frac{1+m_Y(z)}{m_Y(z)}\mI_p + \frac1p\left( \mM^\trans\mM + \mP\mP^\trans\right) \right)^{-1} \mP\mSigma\mR^\trans \right)^{-1}.
\label{eq:QY_P_R_S_T}
\end{align}

With the assumption $\Norm{\frac{1}{\sqrt{p}}\mM\mP}=\mathcal{O}(1)$, the deterministic equivalent can be further simplified. Writing $\mM^\trans\mM + \mP\mP^\trans= \mV \mV^\trans$ where  $\mV= \begin{pmatrix} \mP & \mM^\trans  \end{pmatrix} \in \R^{p\times(r+n)}$, $\mP= \mV \mDelta$ where $\mDelta= \begin{pmatrix} \mI_r\\ \m0_{n,r}  \end{pmatrix} \in \R^{(r+n)\times r}$  and using the {\em push-through} identity gives that

\begin{align*}
&\frac1p \mP^\trans \left( \frac{1+m_Y(z)}{m_Y(z)}\mI_p + \frac1p\left(\mM^\trans\mM + \mP\mP^\trans\right) \right)^{-1} \mP \\
&= \frac1p \mDelta^\trans\left( \frac{1+m_Y(z)}{m_Y(z)}\mI_{r+n} + \frac1p \mV^\trans \mV \right)^{-1} \mV^\trans \mV \mDelta\label{eq:QY_simp}
\end{align*}

Let us denote the matrix $\left( \frac{1+m_Y(z)}{m_Y(z)}\mI_{r+n} + \frac1p \mV^\trans \mV \right)$ as a block matrix $\begin{pmatrix} \tilde{\mA} & \tilde{\mB} \\ \tilde{\mC} &  \tilde{\mD} \end{pmatrix}$, where:

\[
\begin{aligned}
    \tilde{\mA} &= \frac{1+m_Y(z)}{m_Y(z)}\mI_r + \frac1p \mP^\trans \mP, \qquad & \qquad \tilde{\mB} &= \frac1p \mP^\trans \mM^\trans, \\
    \tilde{\mC} &= \frac1p \mM \mP, \qquad & \qquad \tilde{\mD} &= \frac{1+m_Y(z)}{m_Y(z)}\mI_n + \frac1p \mM \mM^\trans.
\end{aligned}
\]

We obtain, by applying the block matrix inversion formula:
\begin{align*}
&\frac1p \mP^\trans \left( \frac{1+m_Y(z)}{m_Y(z)}\mI_p + \frac1p\left(\mM^\trans\mM + \mP\mP^\trans\right) \right)^{-1} \mP \\
&= \frac1p \left(\tilde{\mA}-\tilde{\mB}\tilde{\mD}^{-1}\tilde{\mC}\right)^{-1}\left(\mP^\trans \mP - \tilde{\mB}\tilde{\mD}^{-1}\mM\mP\right).
\end{align*}

\begin{itemize}
    \item $\|\tilde{\mA}\|$ and $\|\tilde{\mD}\|$ are of order $\mathcal{O}(1)$, while $\|\tilde{\mB}\|$ and $\|\tilde{\mC}\|$ are of order $\mathcal{O}(\frac{1}{\sqrt{p}})$. Therefore, $\Norm{\tilde{\mB}\tilde{\mD}^{-1}\tilde{\mC}}=o(1)$, and is negligeable in comparison with $\|\tilde{\mA}\|$. The inverse function being continuous in the neighborhood of $\tilde{\mA}$, we can conclude that $\left(\tilde{\mA}-\tilde{\mB}\tilde{\mD}^{-1}\tilde{\mC}\right)^{-1} = \tilde{\mA}^{-1} + o(1)$.
    \item Using the previous results, we know that $\Norm{\frac1p \mP^\trans \mP}=\mathcal{O}(1)$ and \\
    $\Norm{\frac1p \tilde{\mB}\tilde{\mD}^{-1}\mM\mP}=\mathcal{O}(\frac1p)$. Therefore, $\frac1p \mP^\trans \mP - \frac1p\tilde{\mB}\tilde{\mD}^{-1}\mM\mP=\frac1p \mP^\trans \mP + o(1)$.
\end{itemize}

Finally:
\begin{align*}
&\frac1p \mP^\trans \left( \frac{1+m_Y(z)}{m_Y(z)}\mI_p + \frac1p\left( \mM^\trans\mM + \mP\mP^\trans\right) \right)^{-1} \mP = \frac1p \tilde{\mA}^{-1} \mP^\trans \mP + o(1)\\
&=\frac1p \mP^\trans \left( \frac{1+m_Y(z)}{m_Y(z)}\mI_p + \frac1p \mP\mP^\trans \right)^{-1} \mP + o(1).
\end{align*}

This shows that matrix $\mM$ is no longer involved in the expression of the mean, which is given by:
\begin{align}
&\Esp{\mQ_\mY(z)} = \left( \frac1{m_Y(z)} \mI_q + \frac{1}{1+m_Y(z)} \frac1q \left(  \mN^\trans \mN + \mR\mR^\trans \right)
+  \right. \nonumber\\
& \left. \frac{1}{(1+m_Y(z))m_Y(z)} \frac1{pq} \mR\mSigma\mP^\trans \left( \frac{1+m_Y(z)}{m_Y(z)}\mI_p + \frac1p \mP\mP^\trans \right)^{-1} \mP\mSigma\mR^\trans \right)^{-1}.
\end{align}

The expression of $\Esp{\mQ}$ is finally obtained as
\begin{align}
\Esp{\mQ(z)} &= -\frac{1}{z \tilde{m}(z)} \Esp{\mQ_\mY\left( \frac{-1}{\tilde{m}(z)} \right)}.
\label{eq:Qcompound}
\end{align}

\subsection{Polynomial expression of $m(z)$ and $\tilde{m}(z)$}

For now, we have expressed the deterministic equivalent $\bar{\mQ}(z)$ as a function of $m(z)$, which is itself expressed as $m(z) = \frac1q \Tr \bar{\mQ}(z)$. We want to find an explicit expression for $m(z)$, so that $\bar{\mQ}(z)$ can be computed explicitly.

We recall that we have
\begin{align*}
\bar{\mQ}(z) &= \frac{-1}{ z \tilde{m}(z)} \bar{\mQ}_\mY\left(\frac{-1}{\tilde{m}(z)}\right).
\end{align*}

Taking  $\frac1q \trace{(\cdot)}$ yields the following equation:
\begin{equation*}
    m(z) = -\frac{1}{z\tilde{m}(z)}m_Y\left(-\frac{1}{\tilde{m}(z)}\right) \Leftrightarrow m_Y\left(-\frac{1}{\tilde{m}(z)}\right) = -z m(z)\tilde{m}(z).
\end{equation*}

Using equation \eqref{eq:YMP}, we can get:
\begin{equation*}
    -z^2 m^2(z) \tilde{m}(z) + \left(\beta_q-1\right)z m(z)\tilde{m}(z) + z m(z) + 1 = 0.
\end{equation*}

By using $z\tilde{m}(z)=\frac{q}{p} z m(z)-(1-\frac{q}{p})$, we get the following polynomial equation for $m(z)$:
\begin{equation*}
    -m^3(z)\frac{\beta_p}{\beta_q} z^2 + m^2(z)\left(1+\beta_p-2\frac{\beta_p}{\beta_q}\right)z + m(z)\left[z-\left(1-\beta_q\right)\left(\frac{\beta_p}{\beta_q}-1\right)\right] + 1 = 0
\end{equation*}

\section{Limiting Spectral Distribution}
\label{app:lsd}
We display here an illustrative figure of limiting spectral Distribution characterized in Proposition~\ref{prop:lsd}, 
and a proof of this result directly derived from the deterministic equivalent established in Theorem~\ref{thm:eq_det}.

Figure~\ref{fig:LSD} illustrates the results of Proposition~\ref{prop:lsd} under various parameter settings. The theoretical limiting distribution computed in Proposition~\ref{prop:lsd} closely matches its empirical counterpart. From the first to the second row, the ratios $\beta_p$ and $\beta_q$ are held constant while $n$, $p$, and $q$ are each increased by a factor of $5$. This shows, on the one hand, the convergence of the empirical distribution toward the theoretical limit, and on the other hand, that the spectrum remains well confined when $n$, $p$, and $q$ are sufficiently large.

\begin{figure}[htbp!]
    \centering
    \includegraphics[width=.9\linewidth]{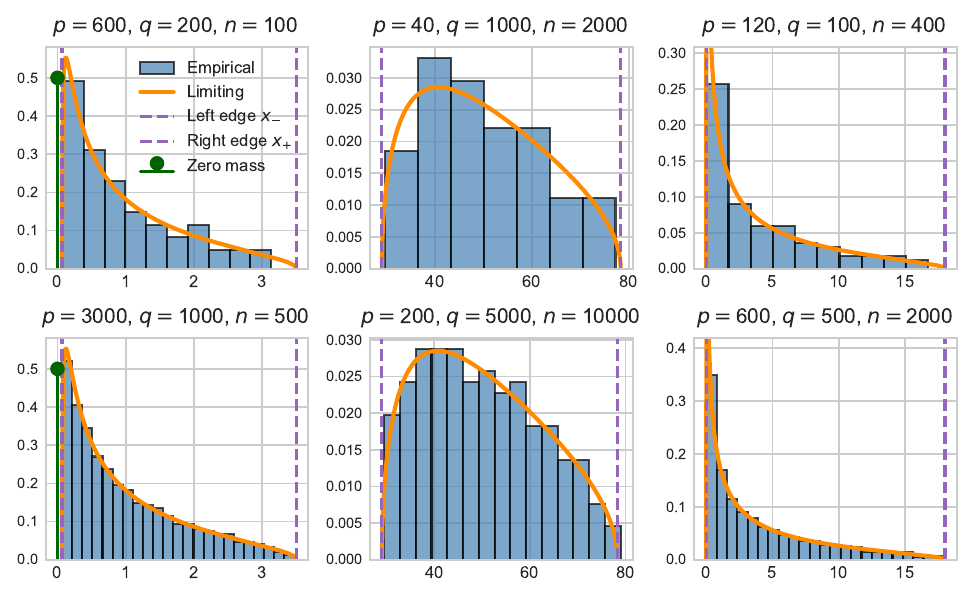}
    \caption{Empirical distribution of the squared singular values of $\mSXY$ together with the limiting spectral distribution predicted by Proposition~\ref{prop:lsd}, shown for different experimental settings. Left column: $\beta_p = 1/6$ and $\beta_q = 1/2$; middle column: $\beta_p = 50$ and $\beta_q = 2$; right column: $\beta_p=10/3$ and $\beta_q = 4$. From the first to the second row, the dimensions $p$, $q$, and $n$ are each multiplied by a factor of $5$.}
    \label{fig:LSD}
\end{figure}

\begin{remark}[On the confinement of the spectrum] Beyond establishing the limiting singular value distribution, we also provide a ``no singular values outside the bulk'' result: for noise-only matrices, the empirical singular values are asymptotically confined to the compact support of the limiting distribution, and in particular none escapes beyond the upper edge.  This confinement property will be crucial for identifying isolated singular values as spike-induced outliers, which in turn carry information about the underlying signal components in the data matrices.
\end{remark}

\subsection{Proof of Proposition \ref{prop:lsd}}
\label{app:proof_lsd}

\subsubsection{Limiting  distribution $\mu$ of the squared singular values}

Let us start with the mass at zero $\mu(\{0\})$. We need to compute the number of zero singular values of $\frac{1}{\sqrt{pq}} \mX^\trans\mY$:
\begin{itemize}
    \item If $n\geq d$, $\frac{1}{\sqrt{pq}} \mX^\trans\mY$ is full-rank, so there is no mass in zero.
    \item If $n<d$, then $\frac{1}{\sqrt{pq}} \mX^\trans\mY$ has rank $n$, so there are $d-n$ zero singular values, and the mass in zero is  $\mu(\{0\})=\frac{d-n}{d}=1-\frac{n}{d}$.
\end{itemize}

Let us now compute the density $f$ of $\mu$. We recall that we have, from Theorem 3.4 of the main paper, the following polynomial equation for $m(z)$:
\begin{equation}
    -m^3(z)\frac{\beta_p}{\beta_q} z^2 + m^2(z)\left(1+\beta_p-2\frac{\beta_p}{\beta_q}\right)z + m(z)\left[z-\left(1-\beta_q\right)\left(\frac{\beta_p}{\beta_q}-1\right)\right] + 1 = 0.
    \label{eq:m(z)_2}
\end{equation}
Then, substituting $m(z)= \bar{m}(z) \beta_q -\frac{1-\beta_q}{z}$, and dividing the resulting expression by $\beta_q$, we obtain
 \begin{equation}
 - \bar{m}^{3}(z) \beta_p \beta_q z^2 + 
 \bar{m}^{2}(z) \left(\beta_p + \beta_q -2 \beta_p\beta_q \right)z  
+ \bar{m}(z) \left[z-\changes{(1-\beta_p)} (1-\beta_q) \right] + 1 = 0.
    \label{eq:m_bar(z)}
\end{equation}

Recall that $m(z)$ denotes the Stieltjes transform of the limiting spectral distribution $\mu_K$ associated with
$
\mK = \frac{1}{pq}\,\mY^\trans \mX \mX^\trans \mY.
$
Our objective, however, is to characterize the limiting density of the squared singular values of
$
\frac{1}{\sqrt{pq}}\, \mX^\trans \mY,
$
which is slightly different. While the \emph{nonzero} squared singular values and eigenvalues coincide in both cases, the multiplicity of the zero values differs between $\mK$ and $\frac{1}{\sqrt{pq}}\, \mX^\trans \mY$. As a consequence, a rescaling is required to recover the density of the squared singular values from that of the eigenvalues of $\mK$. More precisely, if $g$ denotes the limiting spectral density of the eigenvalues of $\mK$ and $f$ the limiting spectral density of the squared singular values of $\frac{1}{\sqrt{pq}}\, \mX^\trans \mY$, then
\[
f(x) = \frac{q}{d}\, g(x),
\qquad d = \min(p,q).
\]

According to the Inverse Stieltjes transform formula, the density $g$ can be computed as
\begin{equation*}
    g(x)=\frac{1}{\pi} \lim_{y \downarrow 0} \Im(m(x+iy),
\end{equation*}
with $z = x + iy \in \C\backslash \Supp(\mu_\mK)$. For real $z$, the imaginary parts of the roots of the two polynomials \eqref{eq:m(z)_2} and \eqref{eq:m_bar(z)} differ only by a scaling factor of $\beta_q$. Thus, by continuity of the Stieltjes transform and of the coefficients of \eqref{eq:m_bar(z)}:
\begin{equation*}
    \lim_{y \downarrow 0} \Im(m(x+iy)) = \beta_q \Im(\bar{m}(x)),
\end{equation*}
where $\bar{m}(x)$ is the complex solution of \eqref{eq:m_bar(z)} with positive imaginary part, to ensure that the solution of \eqref{eq:m_bar(z)} is associated with the unique valid solution of \eqref{eq:m(z)_2}.

Finally, by combining the three previous equations:
\begin{equation}
    f(x)=\frac{n}{d}\frac{1}{\pi} \Im\left(\bar{m}(x)\right),
    \label{eq:density}
\end{equation}
where $\bar{m}(x)$ is the complex solution of \eqref{eq:m_bar(z)} with positive imaginary part.

\subsubsection{No eigenvalue outside the support}

We assume here that $\mX=\mE$ and $\mY=\mF$. We want to prove that, in this setting, no squared singular value of $\mSXY$ (or equivalently no eigenvalue of $\mK$ defined below) escapes outside of the support of the distribution $\mu$. 

In this setting, $\mK=\frac1p\mZ\mZ^\trans$, where $\mZ=\frac{1}{\sqrt{q}}\mY^\trans\mE$. As $\Cov(\vz_i\mid \mY)=\mSigma$, where $\vz_i$ denotes the $i$th column of $\mZ$ and  $\mSigma=\frac{1}{q}\mY^\trans\mY$, we can write $\mK=\frac1p \mZ\mZ^\trans = \frac1p \mSigma^{1/2}\tilde{\mZ}\tilde{\mZ}^\trans\mSigma^{1/2}$, where the columns of $\tilde{\mZ}$ are independent random vectors distributed as $\tilde{\vz}_i\sim\mathcal{N}(\m0_q,\mI_q)$. According to the Mar\v{c}enko-Pastur Theorem, the empirical distribution of $\mSigma$ converges to the Mar\v{c}enko-Pastur distribution, and, most importantly, for sufficiently large values of $n$, $\mSigma$ has no eigenvalue outside the support of its distribution (almost surely). Therefore, according to \cite{no_eig}, $\mK=\frac1p \mZ\mZ^\trans$ has no eigenvalue outside the support of its limiting distribution (almost surely).

\subsubsection{Support of $\mu$}

Let us now compute the edges of the support. We know, thanks to equation \eqref{eq:density}, that $x$ is in the support of $\mu$ if and only if the polynomial equation in $\bar{m}$ given by \eqref{eq:m_bar} has two complex conjugate roots. This can be characterized through the discriminant of this polynomial equation. However, the discriminant of the third-order polynomial in $\bar{m}$ given in \eqref{eq:m_bar} is a fifth order polynomial. To simplify it, we will instead consider the following equation, equivalent to \eqref{eq:m_bar}, obtained by multiplying by $x$ (valid as long as we consider strictly positive values of $x$) and considering the change of variable $y=x\bar{m}(x)$ :
\begin{equation}
    - \beta_p \beta_q y^3 + \left(\beta_p + \beta_q -2 \beta_p\beta_q\right) y^2 + \left[x-\changes{(1-\beta_p)}(1-\beta_q)\right]y + x = 0.
    \label{eq:pol_y}
\end{equation}

The discriminant $\Delta(x)$ of this third-order polynomial in $y$ is itself a third-order polynomial of $x$, with real coefficients. Its value gives some information about the nature of the roots of \eqref{eq:m_bar}:
\begin{itemize}
    \item If $\Delta(x)<0$, then $x\in\Supp(\mu)$, because \eqref{eq:m_bar} has one real and two complex conjugates roots.
    \item If $\Delta(x)>0$, then $x\notin\Supp(\mu)$, because \eqref{eq:m_bar} has three distinct real roots.
\end{itemize}

Thus, the edges must satisfy $\Delta(x)=0$. Therefore, the edges of $\Supp(\mu)$ can be characterized as the two real nonnegative roots of $\Delta(x)$.

\subsubsection{Discriminant of the polynomial equation of Proposition~\ref{prop:lsd}}
\label{app:discriminant}

The discriminant $\Delta(x)$ of the third-order polynomial in $y$ given in \eqref{eq:pol_y} is itself a third-order polynomial in $x$, given by
\begin{equation*}
    \Delta(x) = 18abcd-4b^{3}d+b^{2}c^{2}-4ac^{3}-27a^{2}d^{2},
\end{equation*}
where
\begin{align*}
    a &= -\beta_p\beta_q \\
    b &= \beta_p+\beta_q-2\beta_p\beta_q \\
    c &= x-(1-\beta_p)(1-\beta_q) \\
    d &= x.
\end{align*}

Expanding this expression, we get
\begin{equation}
    \Delta(x) = \tilde{a}x^3 + \tilde{b}x^2 + \tilde{c}x + \tilde{d},
    \label{eq:discriminant}
\end{equation}
where
\begin{align*}
    \tilde{a} &= 4 \beta_p \beta_q \\
    \tilde{b} &= \beta_p^2 \beta_q^2 - 10 \beta_p^2 \beta_q + \beta_p^2 - 10 \beta_p \beta_q^2 - 10 \beta_p \beta_q + \beta_q^2 \\
    \tilde{c} &= -2(\beta_p^3 (\beta_q^2 - 4 \beta_q + 1) + \beta_p^2 (\beta_q^3 + 2 \beta_q^2 + 2 \beta_q + 1) \\
    &+ 2 \beta_p \beta_q (-2 \beta_q^2 + \beta_q - 2) + \beta_q^2 (\beta_q + 1)) \\
    \tilde{d} &= (\beta_p - 1)^2 (\beta_q - 1)^2 (\beta_p - \beta_q)^2.
\end{align*}

\section{Proof of Lemma~\ref{lem:ortho}}
\label{app:proof_ortho}

As before, we only display the proof for the eigenspace generated by $\mN$, as the results for $\mM$ can be deduced by symmetry by doing the usual inversions. We recall the assumption $\Norm{\frac{1}{\sqrt{q}}\mN\mR}=\mathcal{O}(1)$. We will write $\bar{\mQ}_Y$ as
\begin{equation*}
    \bar{\mQ}_{Y}(z) = \left(\frac1{m_{Y}(z)} \mI_q + \frac{1}{1+m_{Y}(z)} \frac1q \left(\mN^\trans\mN + \tilde{\mR}\tilde{\mR}^\trans \right) \right)^{-1},
\end{equation*}
where $\tilde{\mR}$ is defined as
\begin{equation*}
     \tilde{\mR} \equiv \mR\left(\mI_r + \frac{1}{(1+m_{Y}(z)) m_{Y}(z)} \frac1p \mSigma\mP^\trans \left( \frac{1+m_{Y}(z)}{m_{Y}(z)}\mI_p + \frac1p\mP\mP^\trans \right)^{-1} \mP\mSigma\right)^\frac12.
\end{equation*}

We can apply the same idea that in Appendix \ref{app:simp} (\ie write $\mN^\trans\mN + \tilde{\mR}\tilde{\mR}^\trans = \mV \mV^\trans$ where  $\mV= \begin{pmatrix} \tilde{\mR} & \mN^\trans  \end{pmatrix} \in \R^{q\times(n+r)}$, use the {\em push-through} identity and the block matrix inversion formula, and discard the negligible terms), and prove that:

\begin{align*}
\frac1q \mN \bar{\mQ}_{Y}(z) \mN^\trans &= \frac1q \mN \left(\frac1{m_{Y}(z)} \mI_q + \frac{1}{1+m_{Y}(z)} \frac1q \mN^\trans\mN \right)^{-1} \mN^\trans + o(1),
\end{align*}
which concludes the proof.

\bigskip

\section{Proof of Lemma~\ref{lem:existence}}
\label{app:proofQEP}
We first show that the eigenvalues of $\mGamma(\lambda)$ and $\tilde{\mGamma}(\lambda)$ coincide. 
Applying the similarity transformation $\mPsi(\lambda) = \mSigma^{-1}\mGamma(\lambda)\mSigma$ preserves the spectrum of $\mGamma(\lambda)$. Moreover, the matrices $\mPsi(\lambda)$ and $\mPsi(\lambda)^\top$ share the same eigenvalues since $\det(\mPsi(\lambda)) = \det(\mPsi(\lambda)^\top)$. 
A direct calculation yields:
\begin{align*}
    \mPsi(\lambda)^\top &= \left(\lambda^2 \mI_r - \lambda \left( \mK_P + \mSigma^{-1} \mK_R \mSigma + \mK_P \mSigma \mK_R \mSigma\right) - \mK_P (\mSigma - \mSigma^{-1}) \mK_R \mSigma \right)^\top \\
    &= \lambda^2 \mI_r - \lambda \left( \mK_P + \mSigma \mK_R \mSigma^{-1} + \mSigma \mK_R \mSigma \mK_P \right) - \mSigma \mK_R (\mSigma - \mSigma^{-1}) \mK_P \\
    &= \tilde{\mGamma}(\lambda),
\end{align*}
where the second equality follows from the symmetry of $\mK_P$, $\mK_R$, and $\mSigma$. This shows that $\det(\mGamma(\lambda)) = \det(\tilde{\mGamma}(\lambda))$, thus both polynomials share the exact same spectrum.

To establish the properties of this spectrum, we rely on the companion linearization stated in the lemma. The condition $\mGamma(\lambda)\vx = \m0$ can be rewritten as $\lambda^2 \vx - \lambda \mA \vx - \mB \vx = \m0$. Introducing the extended vector $\vv = [\vx^\top, \lambda \vx^\top]^\top \in \mathbb{R}^{2r}$, the QEP is strictly equivalent to the linear eigenvalue problem $\mM \vv = \lambda \vv$. For notational convenience in the subsequent derivations, we define $\bar{\mK}_P \equiv \mSigma \mK_P \mSigma$, such that the blocks of $\mM$ can be factored as:
\begin{equation}
\mA = \mK_R + \bar{\mK}_P \mSigma^{-2} + \bar{\mK}_P \mK_R, \quad
\mB = \bar{\mK}_P (\mI_r - \mSigma^{-2}) \mK_R.
\label{eq:AetB}
\end{equation}
We now map this standard eigenvalue problem to a generalized symmetric one. Let us
introduce the similarity matrix $\mT$ and its inverse:
\begin{equation*}
\mT = 
\begin{bmatrix}
\mK_R & \m0 \\
-\mK_R & \mI_r
\end{bmatrix}, \qquad
\mT^{-1} = 
\begin{bmatrix}
\mK_R^{-1} & \m0 \\
\mI_r & \mI_r
\end{bmatrix}.
\end{equation*}
Applying the similarity transformation $\hat{\mM} = \mT \mM \mT^{-1}$ preserves the spectrum of $\mM$. The corresponding eigenvectors satisfy $\hat{\vv} = \mT \vv$. Partitioning the vectors as
\begin{equation*}
\hat{\vv} = 
\begin{bmatrix}
\hat{\vx}_1 \\
\hat{\vx}_2
\end{bmatrix}, \qquad
\vv = 
\begin{bmatrix}
\vx \\
\lambda \vx
\end{bmatrix},
\end{equation*}
the original eigenvector of the QEP is easily recovered via $\vx = \mK_R^{-1}\hat{\vx}_1$.
We first compute the product:
\begin{equation*}
    \mM \mT^{-1} = 
    \begin{bmatrix}
    \m0 & \mI_r \\
    \mB & \mA
    \end{bmatrix}
    \begin{bmatrix}
    \mK_R^{-1} & \m0 \\
    \mI_r & \mI_r
    \end{bmatrix}
    = 
    \begin{bmatrix}
    \mI_r & \mI_r \\
    \mB\mK_R^{-1} + \mA & \mA
    \end{bmatrix}.
\end{equation*}
Using the definitions in \eqref{eq:AetB}, we can simplify the lower-left block:
\begin{align*}
\mB\mK_R^{-1} + \mA &= \bar{\mK}_P(\mI_r - \mSigma^{-2}) + \mK_R + \bar{\mK}_P\mSigma^{-2} + \bar{\mK}_P\mK_R \\
&= \mK_R + \bar{\mK}_P(\mI_r + \mK_R).
\end{align*}
Therefore, the transformed matrix becomes:
\begin{align*}
\hat{\mM} &= \mT (\mM \mT^{-1}) = 
\begin{bmatrix}
\mK_R & \m0 \\
-\mK_R & \mI_r
\end{bmatrix}
\begin{bmatrix}
\mI_r & \mI_r \\
\mK_R + \bar{\mK}_P(\mI_r + \mK_R) & \mA
\end{bmatrix} \\
&= 
\begin{bmatrix}
\mK_R & \mK_R \\
\bar{\mK}_P(\mI_r + \mK_R) & \bar{\mK}_P(\mSigma^{-2} + \mK_R)
\end{bmatrix}.
\end{align*}
The linear eigenvalue problem $\hat{\mM}\hat{\vv} = \lambda\hat{\vv}$ admits the following factorization:
\begin{equation*}
\hat{\mM} = 
\begin{bmatrix}
(\mI_r + \mK_R^{-1})^{-1} & \m0 \\
\m0 & \bar{\mK}_P
\end{bmatrix}
\begin{bmatrix}
\mI_r + \mK_R & \mI_r + \mK_R \\
\mI_r + \mK_R & \mSigma^{-2} + \mK_R
\end{bmatrix},
\end{equation*}
where we exploited the identity $(\mI_r + \mK_R^{-1})^{-1}(\mI_r + \mK_R) = \mK_R$. This structure takes the form $\hat{\mM} = \hat{\mA}^{-1}\hat{\mB}$, with:
\begin{equation*}
\hat{\mA} = 
\begin{bmatrix}
\mI_r + \mK_R^{-1} & \m0 \\
\m0 & \bar{\mK}_P^{-1}
\end{bmatrix}, \qquad
\hat{\mB} = 
\begin{bmatrix}
\mK_R + \mI_r & \mK_R + \mI_r \\
\mK_R + \mI_r & \mK_R + \mSigma^{-2}
\end{bmatrix}.
\end{equation*}
Multiplying the system $\hat{\mA}^{-1}\hat{\mB}\hat{\vv} = \lambda\hat{\vv}$ by $\hat{\mA}$ yields the generalized symmetric eigenvalue problem $\hat{\mB}\hat{\vv} = \lambda\hat{\mA}\hat{\vv}$. Substituting back $\bar{\mK}_P = \mSigma \mK_P \mSigma$, we get:
\begin{equation*}
\hat{\mA} = 
\begin{bmatrix}
\mI_r + \mK_R^{-1} & \m0 \\
\m0 & \mSigma^{-1}\mK_P^{-1}\mSigma^{-1}
\end{bmatrix}.
\end{equation*}

Since both $\mK_P$ and $\mK_R$ are kernels (hence symmetric positive definite), the block-diagonal matrix $\hat{\mA}$ is symmetric positive definite (SPD) and $\hat{\mB}$ is symmetric. By the spectral theorem for symmetric generalized eigenvalue problems, this guarantees that all $2r$ eigenvalues are real and that the system is diagonalizable (i.e., it admits a complete set of $2r$ linearly independent eigenvectors).

Furthermore, because $\hat{\mA}$ is SPD, the signs of these $2r$ eigenvalues depend entirely on the inertia of $\hat{\mB}$. To characterize the inertia of $\hat{\mB}$, we apply the congruence transformation with $\mU$ which block-diagonalizes $\hat{\mB}$:
\begin{equation*}
\mU = 
\begin{bmatrix}
\mI_r & \m0 \\
-\mI_r & \mI_r
\end{bmatrix}, \quad \textrm{where }
\mU \hat{\mB} \mU^\top = 
\begin{bmatrix}
\mK_R + \mI_r & \m0 \\
\m0 & \mSigma^{-2} - \mI_r
\end{bmatrix}.
\end{equation*}
By Sylvester’s law of inertia, $\hat{\mB}$ and $\mU \hat{\mB} \mU^\top$ have the same number of positive, negative, and zero eigenvalues. We observe that $\mK_R + \mI_r$ is SPD, yielding $r$ strictly positive eigenvalues. The bottom-right block is diagonal with entries $\rho_i^{-2} - 1$. Since the coupling coefficients satisfy $0 < \rho_i \le 1$, it follows that $\rho_i^{-2} - 1 \ge 0$ (meaning that $\hat{\mB}$ is positive semi-definite), with equality if and only if $\rho_i = 1$.
Let $r_1$ denote the number of indices such that $\rho_i = 1$. This block therefore contributes $r_1$ zero eigenvalues and $r - r_1$ strictly positive eigenvalues. We conclude that the generalized eigenvalue problem admits $2r$ real non-negative eigenvalues, of which $r_1$ are zero and $2r - r_1$ are strictly positive.

\section{Proof of Proposition~\ref{prop:isolated_ST}}
\label{app:proof_ST}

As before, we only display the proof for the isolated eigenvalues due to $\mN$, as the results for $\mM$ can be deduced by symmetry by doing the usual inversions.

Let $r_N$ be the rank of $\mN$. We can write the following eigen-decomposition $\mK_N \equiv \frac1q \mN^\trans\mN = \mU_N \mLambda_N \mU_N^\trans$ where the columns of $\mU_N \in \R^{q \times r_N}$ form an orthogonal basis of the eigenvectors and $\mLambda_N$ is the diagonal matrix of the eigenvalues. A necessary condition to have an isolated eigenvalue of $\mK$ due to the component $\mN$, is that $\bar{\mQ}_Y(z)$ admits a singular point outside the limiting support of the bulk, and we know from Lemma~\ref{lem:ortho} that the eigenvectors due to $\mN$ live in the space generated by $\mU_N$. Thus, a necessary condition is that $\mU_N^\trans\bar{\mQ}_Y(z)\mU_N$ admits a singular point outside the limiting support of the bulk. From Lemma~\ref{lem:ortho}, we have:

\begin{align*}
\mU_N^\trans\bar{\mQ}_Y(z)\mU_N &= \mU_N^\trans\left(\frac{1}{m_Y(z)}\mI_q + \frac{1}{1+m_Y(z)}\mK_N\right)^{-1}\mU_N + o(1) \\ 
&= \left(\frac{1}{m_Y(z)}\mI_{r_N} + \frac{1}{1+m_Y(z)}\mLambda_N\right)^{-1} + o(1),
\end{align*}
where we used respectively $\mK_N\mU_N\mU_N^\trans = \mK_N$ and the {\em push-through} identity. Let $\lambda >0$ an eigenvalue of $\mK_N$, we can now see that the associated singular point $x$ satisfies
\begin{equation}
\frac1{m_{Y}(x)} + \frac{\lambda}{1+m_{Y}(x)} = 0 \Leftrightarrow m_{Y}(x) = \frac{-1}{\lambda + 1},
\label{eq:locMP}
\end{equation}
and hence, based on equation \eqref{eq:YMP}, 
\begin{align}
x &= \frac{(\lambda+1)(\lambda+\beta_q)}{\lambda}.
\label{eq:locMPxi}
\end{align}
Returning to the original problem by replacing $x$ by $\frac{-1}{\tilde{m}(\xi)}$, it comes that 
\begin{align}
    \tilde{m}(\xi) & = \frac{-\lambda}{(\lambda +1)(\lambda + \beta_q)}.
    \label{eq:loctilde}
\end{align}
A necessary condition for such a $\xi$ to exist is that $\tilde{m}(\xi)$ is a valid Stieltjes transform. If we set aside for now this condition, the corresponding value of $\xi$ can be derived 
by using directly the relation $m_{Y}\left(-\frac{1}{\tilde{m}(z)} \right) = -z \tilde{m}(z)m(z) =  -z\frac{p}{q} \tilde{m}(z)^2 - \left(\frac{p}{q}-1\right) \tilde{m}(z)$ since $zm(z)=z\frac{p}{q}\tilde{m}(z) + \left(\frac{p}{q} - 1\right)$. Thus we directly get 
\begin{align*}
     \xi &= \frac{ -\frac{q}{p}m_{Y}\left(-\frac1{\tilde{m}(\xi)}\right) + (\frac{q}{p}-1)\tilde{m}(\xi) }{ \tilde{m}(\xi)^2},
\end{align*}
which simplifies into 
\begin{align}
    \xi &= \frac{(\lambda+1)(\lambda + \beta_p)\left( \lambda + \beta_q\right) }{\lambda^2}.
    \label{eq:xiT}
\end{align}

The threshold at which this solution exists can now be deduced from the previous equation \eqref{eq:xiT}. Indeed, to be a valid Stieltjes transform, the function that maps $\tilde{m}(\xi)$ to   $\xi$ must be increasing\footnote{If $m$ is the Stieltjes transform of a real probability measure $\mu$ with support denoted $\Supp(\mu)$, then $x\mapsto m(x)$ is an increasing function on all connected components of $\mathbb{R}\backslash \Supp(\mu)$. The inverse of this function is then also increasing on all the images of the these connected components.}. 
From \eqref{eq:loctilde}, $\tilde{m}$ is an increasing function of $\lambda$. Hence we only need to look at the value of $\lambda$ from which the function $\xi(\lambda)$ given in \eqref{eq:xiT} becomes increasing. By studying the derivative, this threshold corresponds to the largest positive root of the following third-order equation:
\begin{align*}
\lambda^3 - \lambda (\beta_p \beta_q + \beta_p + \beta_q) - 2 \beta_p \beta_q =0 .
\end{align*}
The root of this depressed cubic polynomial can be expressed in closed-form using trigonometric functions as:
\begin{align*}
\tau & = 2 \sqrt{\frac{\beta_p \beta_q+\beta_p+\beta_q}{3}} \cos{\left(
\frac13\operatorname{acos}{ \left(\beta_p \beta_q  \left(\frac{\beta_p \beta_q+\beta_p+\beta_q}{3}\right)^{-3/2} \right) }  \right)} .
\end{align*}

\section{Proof of Proposition~\ref{prop:alignment_ST}}
\label{app:proof_align_ST}

As before, we only present the proof for the isolated eigenvectors associated with $\mN$, since the results for $\mM$ can be obtained by symmetry through the usual interchange. As there is no risk of confusion, we drop the subscript $N$ in the sequel. In particular, we denote by $\lambda_k$ the $k$th squared singular value of $\mN$, by $\vv_k$ (instead of $\vv_{N,k}$) the associated right singular vector of $\mN$, and by $\hat{\vv}_k$ (instead of $\hat{\vv}_{N,k}$) the corresponding right singular vector of the cross-covariance matrix.

Note that $\lambda_k$ and $\vv_k \in \mathbb{R}^q$ coincide with the $k$th eigenpair of
$\mK_N = \sum_{k=1}^{r_N} \lambda_k \vv_k\vv_k^\trans,$
and that $\hat{\vv}_k$ is the corresponding eigenvector of the kernel matrix $\mK$ defined in~\eqref{eq:K}. Our goal is to compute the quantity
$
\langle \hat{\vv}_k, \vv_k \rangle^2.
$

Using the properties of the resolvent (see \cite[section 2.1.3]{couillet_liao_2022})
\begin{equation*}
\langle \hat{\vv}_k, \vv_k \rangle^2 = -\frac{1}{2i\pi} \oint_{\Gamma_\xi} \vv_k^\trans \mQ(z) \vv_k \mathrm{d}z = -\frac{1}{2i\pi} \oint_{\Gamma_\xi} \vv_k^\trans \bar{\mQ}(z) \vv_k \mathrm{d}z + o(1),
\end{equation*}
where $\Gamma_\xi$ is a contour circling around $\xi$ only. Applying now the residue theorem with $f : z \mapsto \vv_k^\trans \bar{\mQ}(z) \vv_k$,
\begin{equation*}
-\frac{1}{2i\pi} \oint_{\Gamma_\xi} \vv_k^\trans \bar{\mQ}(z) \vv_k \mathrm{d}z = -\lim_{z \to \xi_k } (z-\xi_k) \vv_k^\trans \bar{\mQ}(z) \vv_k.
\end{equation*}

Therefore, we have
\begin{align*}
\langle \hat{\vv}_k, \vv_k \rangle^2 & = \lim_{z \to \xi_k } \frac{z-\xi_k}{z \tilde{m}(z)} \left[ \frac1{m_{Y}\left( \frac{-1}{\tilde{m}(z)}\right)} + \frac{ \lambda_k}{1+ m_{Y}\left( \frac{-1}{\tilde{m}(z)}\right) } \right]^{-1} \\
& = \lim_{z \to \xi_k } \frac{z-\xi_k}{z \tilde{m}(z)} \frac{h(z)(1+h(z))}{1+(\lambda_k+1)h(z)}  = \frac{h(\xi_k)(1+h(\xi_k))}{\xi_k \tilde{m}(\xi_k)} \lim_{z \to \xi_k } \frac{z-\xi_k}{1+(\lambda_k+1)h(z)},
\end{align*}
with $h(z)=m_{Y}\left(-\frac{1}{\tilde{m}(z)}\right)$. We can write this last limit as 
\begin{align}
\lim_{z \to \xi_k } \frac{z-\xi_k}{1+(\lambda_k+1)h(z)} & = \lim_{z \to \xi_k } \frac{f(z)}{g(z)} =  \lim_{z \to \xi_k } \frac{f'(z)}{g'(z)} \label{eq:hop},
\end{align}
where $f(z)= z-\xi_k$,  $g(z) = 1+(\lambda_k+1)h(z)$, and equation \eqref{eq:hop} is derived from l'Hôpital's rule. All that remains is to calculate the derivatives $f'(\xi_k)$ and $g'(\xi_k)$. The derivatives of the functions in \eqref{eq:hop} are expressed as 
\begin{align*}
f'(\xi_k) &= 1, \quad g'(\xi_k) = (\lambda_k+1)h'(\xi_k),
\end{align*}
with $h'(z)=\frac{\tilde{m}'(z)}{\tilde{m}(z)^{2}}m_{Y}'\left(-\frac{1}{\tilde{m}(z)}\right)$. 
Thus:
\begin{equation}
    \langle \hat{\vv}_k, \vv_k \rangle^2 = \frac{h(\xi_k)(1+h(\xi_k))}{\xi_k \tilde{m}(\xi_k)(\lambda_k+1)h'(\xi_k)}.
    \label{eq:f_align_T}
\end{equation}
We already know that
\begin{equation}
 h(\xi_k) = \frac{-1}{\lambda_k+1}, 
 \quad \xi_k=\frac{(\lambda_k+1)(\lambda_k+\beta_p)(\lambda_k+\beta_q)}{\lambda_k^2},
 \quad \tilde{m}(\xi_k) = \frac{-\lambda_k}{(\lambda_k+1)(
 \lambda_k+\beta_q)}.
 \label{eq:xi_k}
\end{equation}

We still need to compute the quantity $h'(\xi_k)=\frac{\tilde{m}'(\xi_k)}{\tilde{m}(\xi_k)^{2}}m_{Y}'\left(-\frac{1}{\tilde{m}(\xi_k)}\right)$. From the equation \eqref{eq:YMP}, we can deduce that 
\begin{align*}
m_{Y}'(z) &= - \frac{m^2_{Y}(z) + m_{Y}(z)}{2 z m_{Y}(z) - (\beta_q-1-z)},
\end{align*}
and after simplification
\begin{align*}
m_{Y}'\left( \frac{-1}{\tilde{m}(\xi_k)} \right) &=\frac{1}{\lambda_k^2 - \beta_q} \left( \frac{\lambda_k}{\lambda_k + 1} \right)^2.
\label{eq:mMPdiffxi}
\end{align*}
In addition, differentiating the identity $m_{Y}\left(-\frac{1}{\tilde{m}\left(z\right)}\right) =  - z\frac{p}{q}  \tilde{m}(z)^2 - \left(\frac{p}{q}-1\right) \tilde{m}(z)$ gives
\begin{equation*}
    \frac{\tilde{m}'(\xi_k)}{\tilde{m}(\xi_k)^{2}}m_{Y}'\left(-\frac{1}{\tilde{m}(\xi_k)}\right) = -\frac{p}{q}\tilde{m}(\xi_k)^2 - 2\xi_k\frac{p}{q}\tilde{m}(\xi_k)\tilde{m}'(\xi_k)-\left(\frac{p}{q}-1\right)\tilde{m}'(\xi_k),
\end{equation*}
and it comes that
\begin{align}
\tilde{m}'(\xi_k) &=  \frac{\lambda_k^3 (\lambda_k^2 - \beta_q)}
{(\lambda_k + 1)^2 (\lambda_k+\beta_q)^2 (\lambda_k^3 - (\beta_p \beta_q + \beta_p + \beta_q)\lambda_k - 2 \beta_p \beta_q)}, \nonumber \\
h'(\xi_k) &= \frac{\lambda_k^3}
{(\lambda_k + 1)^2(\lambda_k^3 - (\beta_p \beta_q + \beta_p + \beta_q)\lambda_k - 2 \beta_p \beta_q)}.
\label{eq:diff_mt}
\end{align}

Plugging \eqref{eq:xi_k} and \eqref{eq:diff_mt} in the previous equation \eqref{eq:f_align_T}, we obtain the formula \eqref{eq:align_T} given in Prop.~\ref{prop:alignment_ST}.

Let us prove the last point of our proposition, which is the fact that $\hat{\vv}_{M,k}$, the eigenvector of $\mK$ associated with $\lambda_{M,k}$ does not align with any deterministic vector. We know from Theorem~\ref{thm:eq_det} that the deterministic equivalent $\bar{\mQ}$ does not depend on $\mM$. Therefore, there is no singular point of $\bar{\mQ}$ associated with $\mM$, which means according to the residue theorem that the eigenvector of $\mK$ associated with $\lambda_{M,k}$ is asymptotically uncorrelated from $\mM$. Hence, it does not align on any deterministic component.

\section{Proof of Proposition~\ref{prop:isolated_PR}}
\label{app:proof_RP}

As stated in Lemma~\ref{lem:ortho}, the eigenvectors of $\mK$ due to $\mR$ and $\mP$ are asymptotically orthogonal to the eigenvectors due to $\mN$. 
Thus, the singular points of $\bar{\mQ}(z)$ due to the joint term are the values of $z$ that make the following matrix degenerate:

\begin{align}
& \frac1{h(z)} \mI_q + \frac{1}{1+h(z)} \frac1q \mR\left(\mI_r + \frac1{h(z)} \frac1{p} \mSigma\mP^\trans \left( \frac{1+h(z)}{h(z)}\mI_p + \frac1p \mP\mP^\trans \right)^{-1} \mP\mSigma\right)\mR^\trans \nonumber \\
&= \frac1{h(z)} \mI_q + \frac{1}{1+h(z)}\frac1q \mR\left(\mI_r + \mSigma\left(\mI_r + h(z)(\mI_r + \mK_P) \right)^{-1}\mK_P\mSigma\right) \mR^\trans,
\label{eq:M}
\end{align}
with $h(z)=m_{Y}\left(-\frac{1}{\tilde{m}(z)}\right)$. To reduce this to a $r\times r$ matrix, we can use the following lemma, with $\mA=\frac{1}{1+h(z)}\frac1q \mR$ and $\mB=\left(\mI_r + \mSigma\left(\mI_r + h(z)(\mI_r + \mK_P) \right)^{-1}\mK_P\mSigma\right)\mR^\top$.

\begin{lemma}[Sylvester's identity]
    For $\mA\in\R^{q\times r}$, $\mB\in\R^{r\times q}$ and $z\in\C\backslash \{0\}$:
    \begin{equation*}
        \det\left(\mA\mB-z\mI_q\right) = \det\left(\mB\mA-z\mI_r\right)(-z)^{p-n}.
    \end{equation*}
\end{lemma}

Thus, it is equivalent to have the following matrix singular:
\begin{align}
& \frac1{h(z)} \mI_r + \frac{1}{1+h(z)}\left(\mI_r + \mSigma\left(\mI_r + h(z)(\mI_r + \mK_P) \right)^{-1}\mK_P\mSigma\right) \mK_R.
\label{eq:M_tilde}
\end{align}
For convenience, we will make the change of variable $h(z)=\frac{-1}{1+\lambda}$, in order to use later on the same expressions of $\xi$ as in Appendix \ref{app:proof_ST}. By multiplying on the left by $\mSigma\left(\mI_r + h(z)(\mI_r + \mK_P) \right)\mSigma^{-1}$, we get:
\begin{align*}
    & -\lambda\mI_r + \mSigma\mK_P\mSigma^{-1} + \left(1+\frac1\lambda\right)\left(\mI_r+\mSigma\mK_P\mSigma\right)\mK_R - \frac1\lambda\mSigma\left(\mI_r+\mK_P\right)\mSigma^{-1}\mK_R \\
    &= -\lambda\mI_r + \mSigma\mK_P\mSigma^{-1}+\mK_R+\mSigma\mK_P\mSigma\mK_R + \frac{1}{\lambda}\mSigma\mK_P\left(\mSigma-\mSigma^{-1}\right)\mK_R =-\frac1\lambda\mGamma(\lambda).
\end{align*}
As soon as $\lambda>0$, this is equivalent to consider the quadratic eigenvalue problem $\det\mGamma(\lambda)=0$, with $\mGamma(\lambda)$ defined in Def.~\ref{def:joint_polynomials}. From Lemma \ref{lem:existence},  we know that this quadratic eigenvalue problem admits $2r$ real non-negative solutions, that we will denote as $(\lambda_{T,k})_k$. As we set $h(z)=\frac{-1}{1+\lambda}$, we obtain the same results as in Appendix \ref{app:proof_ST} (and in particular \eqref{eq:xiT}). Then, we have:
\begin{equation*}
    \xi_{T,k} = \frac{(\lambda_{T,k}+1) (\lambda_{T,k}+\beta_p) (\lambda_{T,k}+\beta_q)}{\lambda_{T,k}^2}.
\end{equation*}
Symmetrically, $\det\tilde{\mGamma}(\lambda)=0$ gives the same results.

\section{Proof of Proposition~\ref{prop:alignment_PR}}
\label{app:proof_align_RP}

As before, we only display the proof for the right singular vectors, the other side being deduced symmetrically. Let $(\lambda_{k},\vx_k)$ be the $k$th eigenpair of the QEP $\mGamma(\lambda)\vx=\m0$. We know from Appendix \ref{app:proof_RP} that $\vx_k$ is an eigenvector of the matrix \eqref{eq:M_tilde}. By using once again the {\em push-through} identity, one can prove that $\mR\vx_k$ is an eigenvector of \eqref{eq:M}. Thus, we define:
\begin{equation*}
    \vv_k \equiv \frac{\mR\vx_k}{\|\mR\vx_k\|},
\end{equation*}
which is the $k$th eigenvector of $\bar{\mQ}$ due to the joint component. We also denote as $\hat{\vv}_k \equiv \vv_k(\mSXY)$ the corresponding eigenvector of the kernel matrix $\mK$ defined in \eqref{eq:K}. Then, our goal is to compute the quantity $\langle \hat{\vv}_k, \vv_k \rangle^2$.
Using the same properties of the resolvent as in Appendix \ref{app:proof_align_ST}, we have
\begin{align*}
\langle \hat{\vv}_k, \vv_k \rangle^2 &= -\lim_{z \to \xi_k } (z-\xi_k) \vv_k^\trans \bar{\mQ}(z) \vv_k = -\lim_{z \to \xi_k } \frac{z-\xi_k}{\vx_k^\trans \mK_R \vx_k} \frac1q \vx_k^\trans \mR^\top\bar{\mQ}(z)\mR \vx_k \\
&= \lim_{z \to \xi_k } \frac{z-\xi_k}{z\tilde{m}(z)}\frac{1}{\vx_k^\trans \mK_R \vx_k} \frac1q \vx_k^\trans \mR^\top\bar{\mQ}_Y(z)\mR \vx_k.
\end{align*}
Performing the same operations as in Appendix \ref{app:proof_RP}, we get

\begin{align*}
    & \frac1q \mR^\top\bar{\mQ}_Y(z)\mR \\
    &= \frac1q \mR^\top \left(\frac1{h(z)} \mI_q + \frac{1}{1+h(z)}\frac1q \mR\left(\mI_r + \mSigma\left(\mI_r + h(z)(\mI_r + \mK_P) \right)^{-1}\mK_P\mSigma\right) \mR^\trans\right)^{-1}\mR \\
    &= \mK_R \left(\frac1{h(z)} \mI_r + \frac{1}{1+h(z)}\left(\mI_r + \mSigma\left(\mI_r + h(z)(\mI_r + \mK_P) \right)^{-1}\mK_P\mSigma\right) \mK_R\right)^{-1} \\
    &= -\lambda\mK_R\mGamma(\lambda)^{-1}\mSigma\left(\mI_r + h(z)(\mI_r + \mK_P) \right)\mSigma^{-1}.
    \end{align*}
In the limit $z \to \xi_k \Leftrightarrow \lambda\to\lambda_k$, the polynomial matrix $\mGamma(\lambda)^{-1}$ becomes singular, and we can analyze its limiting behavior. Considering the vector $\tilde{\vx}_k$ such that $\tilde{\mGamma}(\lambda_k)\tilde{\vx}_k$, one can see that $\tilde{\vx}_k^\top\mSigma^{-1}\mGamma(\lambda_k)=\m0$. Then, using \cite{Schumacher_2026}, one can prove that:
\begin{equation*}
    \mGamma(\lambda)^{-1} \underset{\lambda \to \lambda_k}{=} \frac{1}{\lambda-\lambda_k} \frac{\vx_k\tilde{\vx}_k^\top\mSigma^{-1}}{\tilde{\vx}_k^\top\mSigma^{-1}\mGamma'(\lambda_k)\vx_k} + \mathcal{O}(1).
\end{equation*}
Furthermore,
\begin{equation*}
    \mSigma\left(\mI_r + h(z)(\mI_r + \mK_P) \right)\mSigma^{-1} = \frac{1}{\lambda+1}\left(\lambda\mI_r-\mSigma\mK_P\mSigma^{-1}\right).
\end{equation*}
Then,
\begin{align*}
\langle \hat{\vv}_k, \vv_k \rangle^2 &= \lim_{z \to \xi_k } \frac{z-\xi_k}{z\tilde{m}(z)} \frac{1}{\vx_k^\trans \mK_R \vx_k}  \frac{-\lambda}{\lambda+1} \vx_k^\trans \mK_R\mGamma(\lambda)^{-1}\left(\lambda\mI_r-\mSigma\mK_P\mSigma^{-1}\right) \vx_k, \\
&= \lim_{z \to \xi_k } \frac{z-\xi_k}{z\tilde{m}(z)} \frac{1}{\lambda-\lambda_k} \frac{-\lambda}{\lambda+1} \frac{\tilde{\vx}_k^\top\mSigma^{-1}\left(\lambda\mI_r-\mSigma\mK_P\mSigma^{-1}\right) \vx_k}{\tilde{\vx}_k^\top\mSigma^{-1}\mGamma'(\lambda_k)\vx_k}, \\
&= \frac{1}{\xi_k\tilde{m}(\xi_k)} \frac{-\lambda_k}{\lambda_k+1} \frac{\tilde{\vx}_k^\top\left(\lambda_k\mI_r-\mK_P\right)\mSigma^{-1} \vx_k}{\tilde{\vx}_k^\top\mSigma^{-1}\mGamma'(\lambda_k)\vx_k} \lim_{z \to \xi_k } \frac{z-\xi_k}{\lambda-\lambda_k}.
\end{align*}
Noting that $\xi(\lambda) = \frac{(\lambda+1) (\lambda+\beta_p) (\lambda+\beta_q)}{\lambda^2}$, this last limit is equal to:
\begin{equation*}
    \xi'(\lambda_k) = \frac{\lambda_k^3 - \lambda_k (\beta_p \beta_q + \beta_p + \beta_q) - 2 \beta_p \beta_q}{\lambda_k^3}.
\end{equation*}
Finally, using the formulas of $\xi_k$ and $\tilde{m}(\xi_k)$ from Appendix \ref{app:proof_align_ST}, we can conclude that:
\begin{equation*}
    \langle \hat{\vv}_k, \vv_k \rangle^2 = \frac{\lambda_k^3 - \left(\beta_p \beta_q +\beta_p+\beta_q\right)\lambda_k -2\beta_p \beta_q}
{\lambda_k\left(\lambda_k + 1\right)\left(\lambda_k+\beta_p\right)}\eta_{R,k},
\end{equation*}
with
\begin{equation*}
    \eta_{R,k} \equiv \frac{\tilde{\vx}_k^\top\left(\lambda_k\mI_r-\mK_P\right)\mSigma^{-1} \vx_k}{\tilde{\vx}_k^\top\mSigma^{-1}\mGamma'(\lambda_k)\vx_k}.
\end{equation*}

\section*{Acknowledgments}
The authors thank Laurent Guyon for stimulating discussions on high-dimensional genomic data, which motivated this study on partial least squares, and Hugo Jeannin for his contributions to this project. They also express their gratitude to the reviewers for their careful reading and constructive feedback, which helped improve both the mathematical presentation and the overall clarity of the manuscript.

\bibliographystyle{siamplain}
\bibliography{pls}
\end{document}